\documentclass{article}

\usepackage[table]{xcolor}
\usepackage{subfig}
\usepackage[left=2.5cm,right=2.5cm,top=2.5cm,bottom=2.5cm]{geometry}

\PassOptionsToPackage{table}{xcolor}

\usepackage{graphicx} 

\usepackage{algorithm}
\usepackage{algorithmic}
\usepackage{amssymb}

\usepackage{hyperref}
\usepackage{nameref}
\usepackage{url}            
\usepackage{booktabs}       
\usepackage{amsfonts}       
\usepackage{nicefrac}       
\usepackage{microtype}      

\usepackage{cleveref}
\usepackage{xfrac}

\usepackage{enumitem}
\usepackage{pifont}

\usepackage{caption}

\usepackage{wrapfig}

\usepackage{multirow}

\usepackage{amsmath, amssymb, amsthm, amsfonts}
\usepackage{graphicx}
\usepackage{caption}
\usepackage{color}
\usepackage{wrapfig, framed}
\usepackage{mathrsfs}

\makeatletter
\newcommand{\mytag}[2]{%
  \text{#1}%
  \@bsphack
  \begingroup
    \@onelevel@sanitize\@currentlabelname
    \edef\@currentlabelname{%
      \expandafter\strip@period\@currentlabelname\relax.\relax\@@@%
    }%
    \protected@write\@auxout{}{%
      \string\newlabel{#2}{%
        {#1}%
        {\thepage}%
        {\@currentlabelname}%
        {\@currentHref}{}%
      }%
    }%
  \endgroup
  \@esphack
}
\makeatother

\newcommand{\polylog}{{\operatorname{polylog}}}

\def\tsum{{\mathop{\textstyle \sum }}}

\newtheorem{lemma}{Lemma}
\newtheorem{theorem}{Theorem}
\newtheorem{corollary}{Corollary}

\newcommand{\Exp}{{\operatorname{\mathbb{E}}}}

\newcommand{\Norm}{\mathcal{N}}
\newcommand{\Real}{\mathbb{R}}

\newcommand{\blue}[1]{{\color{blue}{#1}}}

\newcommand{\ltynorm}[2]{\| {#1} \|_{#2}}

\newcommand{\htheta}{\widehat{\theta}}

\newcommand{\sign}{\operatorname{sign}}

\usepackage{url}
\usepackage{xspace}
\usepackage{comment}
\usepackage{color}
\usepackage{afterpage}

\usepackage{bm}

\usepackage{cleveref}

\global\long\def\e{{\bm{e}}}

\global\long\def\g{\bm{g}}

\global\long\def\Pr{\mathbb{P}}

\DeclareMathOperator{\Shrink}{\mathbf{S}}

\newcommand{\diff}{{\mathrm{d}}}

\allowdisplaybreaks

\usepackage{setspace}

\begin{document}

\title{Approximate Newton-based statistical inference \\ using only stochastic gradients \\ \,}

%
\author{
Tianyang Li\textsuperscript{1} \\
\texttt{lty@cs.utexas.edu}
    \and
    Anastasios Kyrillidis\textsuperscript{2} \\
    \texttt{anastasios@rice.edu}
		\and
		Liu Liu\textsuperscript{1} \\
    \texttt{liuliu@utexas.edu}
    \and
    Constantine Caramanis\textsuperscript{1} \\
    \texttt{constantine@utexas.edu}\\
    \,\\
    {}\textsuperscript{1} The University of Texas at Austin \\
    {}\textsuperscript{2} Rice University
}

\date{}

\maketitle

\begin{abstract}
We present a novel statistical  inference framework for convex empirical risk minimization, using approximate stochastic Newton steps.
The proposed algorithm is based on the notion of finite differences and allows the approximation of a Hessian-vector product from first-order information.
In theory, our method efficiently computes the statistical error covariance in $M$-estimation, both for unregularized convex learning problems and high-dimensional LASSO regression, without using exact second order information, or resampling the entire data set.
We also present a stochastic gradient sampling scheme for statistical inference in non-i.i.d. time series analysis, where we sample  contiguous blocks of indices.
In practice, we demonstrate the effectiveness of our framework on large-scale machine learning problems, that go even beyond convexity:
as a highlight, our work can be used to detect certain adversarial attacks on neural networks.
\end{abstract}


\section{Introduction}


Statistical inference is an important tool for assessing uncertainties, both for estimation and prediction purposes \cite{friedman2001elements, efron2016computer}. 
\emph{E.g.}, in unregularized linear regression and high-dimensional LASSO settings \cite{van2014asymptotically, javanmard2015biasing, tibshirani2015statistical}, 
we are interested in computing coordinate-wise confidence intervals and p-values of a $p$-dimensional variable, in order to infer which coordinates are active or not \cite{wasserman2013all}. 
Traditionally, the inverse Fisher information matrix \cite{edgeworth1908probable} contains the answer to such inference questions;
however it requires storing and computing a $p \times p $ matrix structure, often prohibitive for large-scale applications \cite{tuerlinckx2006statistical}. 
Alternatively, the Bootstrap method is a popular statistical inference algorithm, where we solve an optimization problem per dataset replicate, 
but can be expensive for large data sets \cite{kleiner2014scalable}. 


While optimization is mostly used for point estimates, recently it is also used as a means for statistical inference in large scale machine learning  \cite{li2017statistical, chen2016statistical, su2018statistical, fang2017scalable}.
This manuscript follows this path: we propose an inference framework that uses stochastic gradients to approximate second-order, Newton steps. 
This is enabled by the fact that we only need to compute Hessian-vector products; in math, this can be approximated using 
$\nabla^2 f(\theta) v \approx \tfrac{\nabla f(\theta + \delta v) - \nabla f(\theta)}{\delta}$, where $f$ is the objective function, and $\nabla f$, $\nabla^2 f$ denote the gradient and Hessian of $f$.
Our method can be interpreted as a generalization of the SVRG approach in optimization \cite{johnson2013accelerating} (\Cref{append:sec:svrg-intuition}); further, it is related to other stochastic Newton methods (e.g. \cite{agarwal2017second}) when $\delta \to 0$. 
We defer the reader to \Cref{sec:realted-work} for more details. 
In this work, we apply our algorithm to unregularized $M$-estimation, and we use a similar approach, with proximal approximate Newton steps, in high-dimensional linear regression. 


Our contributions can be summarized as follows; a more detailed discussion is deferred to \Cref{sec:realted-work}:
\renewcommand\labelitemi{\ding{111}}
\begin{itemize}[leftmargin=0.5cm]
\item For the case of unregularized $M$-estimation, our method efficiently computes the statistical error covariance, useful for confidence intervals and p-values. 
Compared to state of the art, our scheme $\pmb{(i)}$ guarantees consistency of computing the statistical error covariance, $\pmb{(ii)}$ exploits better the available information (without wasting computational resources to compute quantities that are thereafter discarded), and $\pmb{(iii)}$ converges to the optimum (without 
swaying around it).
\item For high-dimensional linear regression, we propose a different estimator (see \eqref{eq:lasso-mod-cov}) than the current literature. It is the result of a different optimization problem that is strongly convex with high probability. 
This permits the use of linearly convergent proximal algorithms \cite{xiao2014proximal,lee2014proximal} towards the optimum; in contrast, state of the art only guarantees  convergence to a neighborhood of the LASSO solution within statistical error. 
Our model also does not assume that absolute values of the true parameter's non-zero entries are lower bounded. 
\item For statistical inference in non-i.i.d. time series analysis, we sample contiguous blocks of indices (instead of uniformly sampling) to compute stochastic gradients. 
This is similar to the Newey-West estimator \cite{newey1986simple} for HAC (heteroskedasticity and autocorrelation consistent) covariance estimation, 
  but does not waste computational resources to compute the entire matrix. 
\item The effectiveness of our framework goes even beyond convexity. As a highlight, we show that our work can be used to detect certain adversarial attacks on neural networks. 
\end{itemize}

%



\section{Unregularized $M$-estimation}\label{sec:unregularized-m-est}


In unregularized, low-dimensional $M$-estimation problems, we estimate a parameter of interest:
\begin{align*}
\theta^\star = \arg \min_{\theta \in \mathbb{R}^p} \Exp_{X \sim P} \left[\ell(X; \theta)\right],\quad \text{where $P(X)$ is the data distribution},   
\end{align*}
using {\em empirical risk minimization} (ERM)  on $n > p$ i.i.d. data points $\{X_i\}_{i=1}^n$:
\begin{align*}
\htheta = \arg \min_{\theta \in \mathbb{R}^p} \tfrac{1}{n} \sum_{i=1}^n \ell(X_i; \theta) . 
\end{align*}
Statistical inference, such as computing one-dimensional confidence intervals, gives us information beyond the point estimate $\htheta$, when $\htheta$ has an asymptotic limit distribution \cite{wasserman2013all}. 
\emph{E.g.}, under regularity conditions, the $M$-estimator satisfies asymptotic normality \cite[Theorem 5.21]{van1998asymptotic}. 
\emph{I.e.}, $\sqrt{n}(\htheta - \theta^\star)$ weakly converges to a normal distribution:
\begin{align*}
\sqrt{n}  \left( \htheta - \theta^\star \right) \to \Norm \left(0, {H^\star}^{-1} G^\star {H^\star}^{-1} \right) , 
\end{align*}
where $H^\star = \Exp_{X \sim P}[\nabla^2_\theta \ell(X;\theta^\star)]$ and $G^\star = \Exp_{X \sim P}[\nabla_\theta \ell(X;\theta^\star) \, \nabla_\theta \ell(X;\theta^\star)^\top]$. 
We can perform statistical inference when we have a good estimate of ${H^\star}^{-1} G^\star {H^\star}^{-1}$. 
In this work, we use the plug-in covariance estimator $\widehat{H}^{-1} \widehat{G} \widehat{H}^{-1}$ for ${H^\star}^{-1} G^\star {H^\star}^{-1}$, where: 
\begin{align*}
\widehat{H} = \tfrac{1}{n} \sum_{i=1}^n \nabla^2_\theta \ell(X_i;\htheta) , \quad \text{ and } \quad
\widehat{G} = \tfrac{1}{n} \sum_{i=1}^n \nabla_\theta \ell(X_i;\htheta) \, \nabla_\theta \ell(X_i;\htheta)^\top . 
\end{align*}
Observe that, in the naive case of directly computing $\widehat{G}$ and $\widehat{H}^{-1}$, we require both high computational- and  space-complexity. 
Here, instead, we utilize approximate stochastic Newton motions from first order information to compute the quantity  $\widehat{H}^{-1} \widehat{G} \widehat{H}^{-1}$. 

\subsection{Statistical inference with approximate Newton steps using only stochastic gradients}
Based on the above, we are interested in solving the following $p$-dimensional optimization problem:
\begin{align*}
\htheta = \arg \min_{\theta \in \Real^p} f(\theta) := \tfrac{1}{n} \sum_{i=1}^n f_i(\theta), \quad \text{where $f_i(\theta) = \ell(X_i;\theta)$.} 
\end{align*}
Notice that $\widehat{H}^{-1} \widehat{G} \widehat{H}^{-1}$ can be written as $\frac{1}{n}$ $\sum_{i=1}^n$ $ \left( \widehat{H}^{-1} \nabla_\theta \ell(X_i;\htheta) \right)$ $ \left( \widehat{H}^{-1} \nabla_\theta \ell(X_i;\htheta) \right)^\top$, 
which can be interpreted as the covariance of stochastic --inverse-Hessian conditioned-- gradients  at $\htheta$. 
Thus, the covariance of stochastic Newton steps can be used for statistical inference. 

\begin{algorithm}
\begin{algorithmic}[1]
%
%
\STATE \textbf{Parameters:} $S_o, S_i \in \mathbb{Z}_+$; $\rho_0, \tau_0 \in \mathbb{R}_+$; $d_o, d_i \in \left(\tfrac{1}{2}, 1\right)$  \quad   \textbf{Initial state:} $\theta_0 \in \mathbb{R}^p$ \\
{\kern1.5pt \hrule \kern2pt} 
\FOR[approximate stochastic Newton descent]{$t = 0 \text{ to } T-1$}  \label{alg:stat-inf:unregularized:outer:start} 
	\STATE $\rho_t \gets \rho_0(t + 1)^{-d_o}$
	\STATE $I_o \gets$ uniformly sample $S_o$ indices with replacement from $[n]$ \label{alg:stat-inf-spnd:sgd1}
	\STATE $g^0_t \gets -\rho_t \left( \frac{1}{S_o} \sum_{i  \in I_o} \nabla f_i(\theta_t) \label{alg:stat-inf-spnd:sgd2} \right) $ 
	\FOR[solving \eqref{eq:sec:unregularized-m-est:inner-loops:intuition:newton:obj} approximately using SGD]{$j = 0 \text{ to } L-1$} \label{alg:stat-inf-spnd:spnd-start}
		\STATE $\tau_j \gets \tau_0 (j+1)^{-d_i}$ and $\delta_t^j \gets O(\rho_t^4 \tau_j^4)$
		\STATE $I_i \gets$ uniformly sample $S_i$ indices without replacement from $[n]$
		\STATE $g^{j+1}_t \gets g_t^j -\tau_j  \left( \frac{1}{S_i}  \sum_{k \in I_i}  \tfrac{\nabla f_k(\theta_t +  {\delta_t^j} g_t^j)  - \nabla f_k(\theta_t)}{\delta_t^j} \right) + \tau_j g_t^0$  
	\ENDFOR 
	\STATE Use $\sqrt{S_o} \cdot \frac{\bar{g}_t}{\rho_t}$ for statistical inference, where $\bar{g}_t = \frac{1}{L+1} \sum_{j=0}^L g^j_t$
	\STATE $\theta_{t+1} \gets \theta_t + g_t^L$  \label{alg:stat-inf-spnd:spnd-end} 
\ENDFOR \label{alg:stat-inf:unregularized:outer:end}
\end{algorithmic}
\caption{Unregularized M-estimation statistical inference}
\label{alg:stat-inf-spnd}
\end{algorithm}

\Cref{alg:stat-inf-spnd} approximates each stochastic Newton $\widehat{H}^{-1} \nabla_\theta \ell(X_i;\htheta)$ step using only first order information. 
We start from $\theta_0$ which is sufficiently close to $\htheta$, which can be effectively achieved using SVRG \cite{johnson2013accelerating}; a description of the SVRG algorithm can be found in \Cref{append:sec:svrg-intuition}. 
Lines \ref{alg:stat-inf-spnd:sgd1}, \ref{alg:stat-inf-spnd:sgd2}  compute a stochastic gradient whose covariance is used as part of  statistical inference. 
Lines \ref{alg:stat-inf-spnd:spnd-start} to \ref{alg:stat-inf-spnd:spnd-end} use SGD to solve the Newton step, 
\begin{align}
\min_{g \in \mathbb{R}^p}  \left\langle  \tfrac{1}{S_o} \sum_{i  \in I_o} \nabla f_i(\theta_t)   , g \right\rangle + \tfrac{1}{2 \rho_t} \left \langle g, \nabla^2 f(\theta_t) g \right \rangle, \label{eq:sec:unregularized-m-est:inner-loops:intuition:newton:obj}
\end{align}
which can be seen as a generalization of SVRG; this relationship is described in more detail in \Cref{append:sec:svrg-intuition}.
In particular, these lines correspond to solving \eqref{eq:sec:unregularized-m-est:inner-loops:intuition:newton:obj} using SGD by uniformly sampling a random $f_i$, and approximating:
\begin{align}
\nabla^2 f(\theta) g \approx \tfrac{\nabla f(\theta + {\delta_t^j} g) - \nabla f(\theta)}{\delta_t^j} = \Exp \left[ \tfrac{\nabla f_i(\theta + {\delta_t^j} g) - \nabla f_i(\theta)}{\delta_t^j} \mid \theta\right].
	\label{eq:sec:unregularized-m-est:inner-loops:intuition:hessian-vec-product}
\end{align} 
Finally, the outer loop (lines \ref{alg:stat-inf:unregularized:outer:start} to \ref{alg:stat-inf:unregularized:outer:end}) can be viewed as solving inverse Hessian conditioned stochastic gradient descent, similar to stochastic natural gradient descent \cite{amari1998natural}.

In terms of parameters, similar to \cite{polyak1992acceleration, ruppert1988efficient}, we use a decaying step size in Line 8 to control the error of approximating $H^{-1} g$. 
We set $\delta_t^j =  O(\rho_t^4 \tau_j^4)$ to control the error of approximating Hessian vector product using a finite difference of gradients, so that it is smaller than the error of approximating  $H^{-1} g$ using stochastic approximation. 
For similar reasons, we use a decaying step size in the outer loop to control the optimization error. 

The following theorem characterizes the behavior of \Cref{alg:stat-inf-spnd}.
\begin{theorem}
\label{thm:spnd-bound}
For a twice continuously differentiable and convex function $f(\theta)=\frac{1}{n} \sum_{i=1}^n f_i(\theta)$
where each $f_i$ is also convex and twice continuously differentiable, 
assume  $f$ satisfies
\begin{itemize}[leftmargin=0.5cm]
\item strong convexity: $\forall \theta_1, \theta_2$,  $f(\theta_2) \geq f(\theta_1) + \langle \nabla f(\theta_1) , \theta_2 - \theta_1  \rangle + \frac{1}{2} \alpha \| \theta_2 - \theta_1 \|_2^2$; \vspace{-0.1cm}
\item  $\forall \theta$, each $\| \nabla^2 f_i(\theta) \|_2 \leq \beta_i$, which implies that $f_i$ has Lipschitz gradient: $\forall \theta_1, \theta_2$,  $\| \nabla f_i(\theta_1) - \nabla f_i(\theta_2) \|_2 \leq \beta_i \| \theta_1 - \theta_2 \|_2 $; \vspace{-0.1cm}
\item each $\nabla^2 f_i$ is Lipschitz continuous: $\forall \theta_1, \theta_2$, $\| \nabla^2 f_i(\theta_2) - \nabla^2 f_i(\theta_1) \|_2 \leq h_i \|\theta_2 - \theta_1 \|_2 $.  
\end{itemize}

In \Cref{alg:stat-inf-spnd}, we assume that batch sizes $S_o$---in the outer loop---and $S_i$---in the inner loops---are $O(1)$. 
The outer loop step size is
\begin{align}
\rho_t = \rho_0 \cdot (t+1)^{-d_o}, \quad \text{where $d_o \in \left(\tfrac{1}{2}, 1\right)$ is the decaying rate. } \label{eq:sec:unregularized-m-est:outer-loops:step-size}
\end{align}

In each outer loop, 
the inner loop step size is 
\begin{align}
\tau_j = \tau_0 \cdot  (j+1)^{-d_i}, \quad \text{where $d_i \in \left(\tfrac{1}{2},1 \right)$ is the decaying rate. } \label{eq:sec:unregularized-m-est:inner-loops:step-size}
\end{align}
The scaling constant for Hessian vector product approximation is 
\begin{align}
{\delta_t^j} = \delta_0 \cdot \rho_t^4 \cdot \tau_j^4  = o\left(\tfrac{1}{(t+1)^2(j+1)^2} \right). 
	\label{eq:sec:unregularized-m-est:gradient:finiite-diference:delta}
\end{align}

Then, for the outer iterate $\theta_t$ we have \\
\noindent\begin{minipage}{0.4\textwidth}
\begin{equation}
\Exp \left[ \|\theta_t - \htheta\|_2^2 \right] \lesssim t^{-d_o}, \label{eq:thm:spnd-bound:outer-iter}
\end{equation} 
    \end{minipage}%
    \begin{minipage}{0.1\textwidth}\centering
    and 
    \end{minipage}%
    \begin{minipage}{0.4\textwidth}
\begin{equation}
\Exp \left[\|\theta_t - \htheta\|_2^4 \right] \lesssim t^{-2d_o}.  \label{eq:thm:spnd-bound:outer-iter:4th-moment}  
\end{equation}
    \end{minipage}

In each outer loop, after $L$ steps of the inner loop, we have:
\begin{align}
	\Exp \left[ \left\| \tfrac{\bar{g}_t}{\rho_t} - [\nabla^2 f(\theta_t)]^{-1} g_t^0 \right\|_2^2 \mid \theta_t \right] \lesssim \tfrac{1}{L} \left\|g_t^0 \right\|_2^2, \label{eq:thm:spnd-bound:newton}  
\end{align}
and at each step of the inner loop, we have:
\begin{align}
	\Exp \left[ \left\| g_t^{j+1}- [\nabla^2 f(\theta_t)]^{-1} g_t^0 \right\|_2^4 \mid \theta_t \right] \lesssim  (j+1)^{-2d_i} \left\|g_t^0 \right\|_2^4. \label{eq:thm:spnd-bound:newton:4th-moment} 
\end{align}

After $T$ steps of the outer loop, we have a non-asymptotic bound on the ``covariance'':
\begin{align}
\label{eq:thm:spnd-bound:covariance-bound}
	  \Exp \left[ \left\| H^{-1} G H^{-1}  - \tfrac{S_o}{T} \sum_{t=1}^T \tfrac{\bar{g}_t \bar{g}_t^\top}{\rho_t^2} \right\|_2 \right] \lesssim T^{-\frac{d_o}{2}} + L^{-\frac{1}{2}}, 
\end{align}
where $H = \nabla^2 f(\htheta)$ and $G = \frac{1}{n} \sum_{i=1}^n \nabla f_i(\htheta)\,  \nabla f_i(\htheta)^\top $.

\end{theorem}

Some comments on the results in \Cref{thm:spnd-bound}. 
The main outcome is that \eqref{eq:thm:spnd-bound:covariance-bound} provides a non-asymptotic bound and consistency guarantee for computing the estimator covariance using \Cref{alg:stat-inf-spnd}. 
This is based on the bound for approximating the inverse-Hessian conditioned stochastic gradient in \eqref{eq:thm:spnd-bound:newton}, and the optimization bound in \eqref{eq:thm:spnd-bound:outer-iter}. 
As a side note, the rates in \Cref{thm:spnd-bound} are very  similar to classic results in stochastic approximation \cite{polyak1992acceleration, ruppert1988efficient}; however the nested structure of outer and inner loops is different from standard stochastic approximation algorithms. 
Heuristically, calibration methods for parameter tuning in subsampling methods (\cite{efron1994introduction}, Ch.18; \cite{politis2012subsampling}, Ch. 9) can be used for hyper-parameter tuning in our algorithm.


In \Cref{alg:stat-inf-spnd}, $\left\{\sfrac{\bar{g}_t}{\rho_t} \right\}_{i=1}^n$ does not have  asymptotic normality.
\emph{I.e.}, $\tfrac{1}{\sqrt{T}} \sum_{t=1}^T \frac{\bar{g}_t}{\rho_t}$ does not weakly converge to $\Norm \left(0, \frac{1}{S_o} H^{-1} G H^{-1} \right)$; we give an example using mean estimation in \Cref{subsec:foasnd:no-asymptotic-normality:example:mean-est}. 
For a similar algorithm based on SVRG (\Cref{alg:stat-inf-svrg} in \Cref{sec:svrg-fosnd:inference:unregularized}), we show that we have asymptotic normality and improved bounds for the ``covariance''; however, this requires a full gradient evaluation in each outer loop. 
In \Cref{sec:foasnd-stat-inf-increasing-segment}, we present corollaries for the case where the iterations in the inner loop increase, as the counter in the outer loop increases (\emph{i.e.}, $(L)_t $ is an increasing series). 
This guarantees consistency (convergence of the covariance estimate to $H^{-1} G H^{-1}$), although it is less efficient than using a constant number of inner loop iterations.  
Our procedure also serves as a general and flexible framework for using different stochastic gradient optimization algorithms \cite{toulis2017asymptotic,harikandeh2015stopwasting,loshchilov2015online,daneshmand2016starting} in the inner and outer loop parts. 

%
Finally, we present the following corollary that states that the average of consecutive iterates, in the outer loop, has asymptotic normality, similar to \cite{polyak1992acceleration, ruppert1988efficient}. 
\begin{corollary}
\label{cor:foasnd:asymptotic-normality:outer-avg}
In \Cref{alg:stat-inf-spnd}'s outer loop,  the average of consecutive iterates satisfies\\
\noindent\begin{minipage}{0.45\textwidth}
\begin{equation}
\Exp\left[\left\|\tfrac{\sum_{t=1}^T \theta_t}{T} - \htheta\right\|_2^2\right] \lesssim  \tfrac{1}{T} , \label{eq:thm:spnd-bound:outer-iter:avg_bound}  
\end{equation} 
    \end{minipage}%
    \begin{minipage}{0.1\textwidth}\centering
    and 
    \end{minipage}%
    \begin{minipage}{0.45\textwidth}
\begin{equation}
\tfrac{1}{\sqrt{T}} \left( \tfrac{\sum_{t=1}^T \theta_t}{T} - \htheta \right) = W + \Delta , \label{eq:thm:spnd-bound:outer-iter:avg_normality}   
\end{equation}
    \end{minipage}\\
%
where $W$ weakly converges to $\Norm(0, \frac{1}{S_o} H^{-1} G H^{-1})$, and $\Delta = o_P(1)$ when $T \to \infty$ and $L \to \infty$ ( $ \Exp[\|\Delta\|_2^2] $ $ \lesssim $ $ T^{1-2d_o} + T^{d_o-1} + \frac{1}{L} )$. 
\end{corollary}

\Cref{cor:foasnd:asymptotic-normality:outer-avg} uses 2\textsuperscript{nd} , 4\textsuperscript{th}  moment bounds on individual iterates (eqs. \eqref{eq:thm:spnd-bound:outer-iter}, \eqref{eq:thm:spnd-bound:outer-iter:4th-moment} in the above theorem), and the approximation of inverse Hessian conditioned stochastic gradient in \eqref{eq:thm:spnd-bound:newton:4th-moment}.

\section{High dimensional LASSO linear regression} \label{sec:high-dim:lasso:linear-regression}
In this section, we focus on the case of high-dimensional linear regression.
Statistical inference in such settings, where $p \gg n$, is arguably a more difficult task: the bias introduced by the regularizer is of the same order with the estimator's variance. 
Recent works \cite{zhang2014confidence, van2014asymptotically, javanmard2015biasing} propose statistical inference via de-biased LASSO estimators. 
Here, we present a new $\ell_1$-norm regularized objective and propose an approximate stochastic \emph{proximal} Newton algorithm, using only first order information.
%
%


We consider the linear model $y_i = \langle \theta^{\star} , x_i \rangle + \epsilon_i$, 
for some sparse $\theta^{\star} \in \Real^p$. For each sample, $\epsilon_i \sim \Norm(0, \sigma^2)$ is i.i.d. noise. 
And each data point $x_i \sim \Norm(0, \Sigma) \in \Real^p$.  
\begin{itemize}[leftmargin=0.5cm]
\item \emph{Assumptions on $\theta$:} $\pmb{(i)}$ $\theta^{\star}$ is $s$-sparse; $\pmb{(ii)}$ $\|\theta^{\star}\|_2=O(1)$, which implies that $\|\theta^{\star}\|_1 \lesssim \sqrt{s}$. 
\item {\em Assumptions on $\Sigma$:} 
$\pmb{(i)}$ $\Sigma$ is sparse, where each column (and row) has at most $b$ non-zero entries;\footnote{This is satisfied when $\Sigma$ is block diagonal or banded. Covariance estimation under this sparsity assumption has been extensively studied \cite{bickel2008covariance,bickel2009simultaneous,cai2012optimal}, and  soft thresholding   is an effective yet simple estimation method \cite{rothman2009generalized}. } 
$\pmb{(ii)}$ $\Sigma$ is well conditioned: all of $\Sigma$'s eigenvalues are $\Theta(1)$;
$\pmb{(iii)}$ $\Sigma$ is diagonally dominant ($\Sigma_{ii}- \tsum_{j\neq i}|\Sigma_{ij}| \geq D_\Sigma > 0$ for all $1 \leq i \leq p$), and this will be used to bound the $\ell_\infty$ norm of $\widehat{S}^{-1}$ \cite{varah1975lower}. A commonly used design covariance that satisfies all of our assumptions is  $I$. 
\end{itemize}

We estimate $\theta^{\star}$ using:
\begin{align}
\htheta = \arg \min_{\theta \in \mathbb{R}^p} \tfrac{1}{2} \left \langle \theta, \left(\widehat{S}- \tfrac{1}{n} \sum_{i=1}^n x_i x_i^\top\right) \theta \right \rangle
	+ \tfrac{1}{n} \sum_{i=1}^n \tfrac{1}{2} \left(x_i^\top \theta -y_i\right)^2 + \lambda \ltynorm{\theta}{1} , \label{eq:lasso-mod-cov}
\end{align}
where  $ \widehat{S}_{jk} = \sign \big( \left( \tfrac{1}{n} \tsum_{i=1}^n x_i x_i^\top \right)_{jk} \big) \big(  \left| \left( \tfrac{1}{n} \tsum_{i=1}^n x_i x_i^\top \right)_{jk} \right| - \omega \big)_+$ is an estimate of $\Sigma$ by soft-thresholding
each element of $\tfrac{1}{n} \tsum_{i=1}^n x_i x_i^\top$ with $\omega = \Theta \big( \sqrt{\tfrac{\log p}{n}} \big)$   \cite{rothman2009generalized}.  
Under our assumptions, $\widehat{S}$ is positive definite with high probability when $n \gg b^2 \log p $ (\Cref{lem:proof:lasso:linear:cov:soft-thresh:guarantees}), 
and this guarantees that the optimization problem \eqref{eq:lasso-mod-cov} is well defined. 
\emph{I.e.}, we replace the degenerate Hessian in regular LASSO regression with an estimate, which is positive definite with high probability under our assumptions. 

We set the regularization parameter 
\begin{align*}
\lambda  = \Theta\left( (\sigma + \|\theta^{\star}\|_1) \sqrt{\tfrac{\log p}{n}} \right), 
\end{align*}
which is similar to  LASSO regression \cite{buhlmann2011statistics, negahban2012unified} and related estimators using thresholded covariance \cite{yang2014elementary, jeng2011sparse}. 

\paragraph{Point estimate.}
\Cref{thm:lasso-mod-cov-bounds} provides guarantees for our proposed point estimate \eqref{eq:lasso-mod-cov}. 
\begin{theorem} \label{thm:lasso-mod-cov-bounds}
When $n \gg b^2 \log p  $,  
the solution $\htheta$ in \eqref{eq:lasso-mod-cov} satisfies
\begin{align}
\left\| \htheta - \theta^{\star} \right\|_1 &\lesssim s \left(\sigma + \|\theta^{\star}\|_1\right) \sqrt{\tfrac{\log p}{n}} \lesssim  s \left( \sigma + \sqrt{s} \right)\sqrt{\tfrac{\log p}{n}} , \label{eq:thm:lasso-mod-cov-bounds:l1} \\
\left\| \htheta - \theta^{\star} \right\|_2 &\lesssim \sqrt{s} \left(\sigma + \|\theta^{\star}\|_1\right) \sqrt{\tfrac{\log p}{n}} \lesssim  \sqrt{s} \left( \sigma + \sqrt{s} \right) \sqrt{\tfrac{\log p}{n}} \label{eq:thm:lasso-mod-cov-bounds:l2} ,  
\end{align}
with probability at least $1-p^{-\Theta(1)}$. 
\end{theorem}

\paragraph{Confidence intervals.}
We next present a de-biased estimator $\htheta^\diff$ \eqref{eq:lasso-mod-cov:de-bias:theta}, based on our proposed estimator.
$\htheta^\diff$ can be used to compute confidence intervals and p-values for each coordinate of $\htheta^\diff$, which can be used for false discovery rate control \cite{javanmard2018false}. 
The estimator satisfies:
\begin{align}
\label{eq:lasso-mod-cov:de-bias:theta}
\htheta^\diff =   \htheta + \widehat{S}^{-1} \left[  \tfrac{1}{n} \sum_{i=1}^n  \left( y_i  -  x_i^\top  \htheta \right)  x_i    \right] .
\end{align}

%
%
%
%

Our de-biased estimator is similar to \cite{zhang2014confidence, van2014asymptotically,javanmard2014confidence, javanmard2015biasing}.
however, we have different  terms, since we need to de-bias covariance estimation. 
Our estimator assumes $n \gg b^2 \log p$, since then $\widehat{S}$ is positive definite with high probability (\Cref{lem:proof:lasso:linear:cov:soft-thresh:guarantees}). 
The assumption that $\Sigma$ is diagonally dominant guarantees that the $\ell_\infty$ norm $\|\widehat{S}^{-1}\|_\infty$ is bounded by $O \left( \tfrac{1}{D_\Sigma} \right) $ with high probability when $n \gg  \tfrac{1}{{D_\Sigma}^2}  \log p $.

\Cref{thm:lasso-mod:de-bias:stat-inf} shows that we can compute valid confidence intervals for each coordinate when $n \gg (\tfrac{1}{D_\Sigma} s \left(\sigma + \|\theta^{\star}\|_1\right)\log p)^2.$
This is satisfied when $n \gg (\tfrac{1}{D_\Sigma} s \left(\sigma + \sqrt{s} \right)\log p)^2 $. 
And the covariance is similar to the sandwich estimator  \cite{huber1967behavior,white1980heteroskedasticity}. 
\begin{theorem}\label{thm:lasso-mod:de-bias:stat-inf}
Under our assumptions, when $n \gg \max\{b^2,  \tfrac{1}{{D_\Sigma}^2}\} \log p$, we have:
\begin{align}
\sqrt{n} (\htheta^\diff - \theta^{\star}) = Z + R,
\end{align} 
where the conditional distribution satisfies $Z$ $\mid$ $\{x_i\}_{i=1}^n$ $ \sim$  $\Norm\left(0,   \sigma^2 \cdot \left[ \widehat{S}^{-1} \left( \frac{1}{n} \sum_{i=1}^n x_i x_i^\top \right) \widehat{S}^{-1} \right] \right)$, 
and $\|R\|_\infty \lesssim \frac{1}{D_\Sigma} s \left(\sigma + \|\theta^{\star}\|_1\right) \frac{\log p}{\sqrt{n}} \lesssim \frac{1}{D_\Sigma} s \left(\sigma + \sqrt{s}\right) \frac{\log p}{\sqrt{n}} $ with probability at least $1-p^{-\Theta(1)}$. 
\end{theorem}

Our estimate in \eqref{eq:lasso-mod-cov} has similar error rates to the estimator in \cite{yang2014elementary}; however, no confidence interval guarantees are provided, and the estimator is based on inverting a large covariance matrix. 
Further, although it does not match minimax rates achieved by regular LASSO regression \cite{raskutti2011minimax},  
and the sample complexity in  \Cref{thm:lasso-mod:de-bias:stat-inf} is slightly higher than other methods \cite{van2014asymptotically, javanmard2014confidence, javanmard2015biasing}, 
our criterion is strongly convex with high probability:
this allows us to use linearly convergent proximal algorithms \cite{xiao2014proximal,lee2014proximal}, 
whereas provable linearly convergent optimization bounds for LASSO only guarantees convergence to a neighborhood of the LASSO solution within statistical error \cite{agarwal2010fast}. 
This is crucial for computing the de-biased estimator, as we need the optimization error to be much less than the statistical error. 

In \Cref{sec:appendix:lasso:linear-regression:stat-inf}, we present our algorithm for statistical inference in high dimensional linear regression using stochastic gradients. 
It estimates the statistical error covariance using the plug-in estimator:
$$\widehat{S}^{-1} \left( \tfrac{1}{n} \sum_{i=1}^n (x_i^\top \htheta - y_i)^2 x_i x_i^\top \right) \widehat{S}^{-1} , $$ 
which is related to the empirical sandwich estimator \cite{huber1967behavior,white1980heteroskedasticity}. 
\Cref{alg:stat-inf:high-dim:linear:proximal:svrg}  computes the statistical error covariance. 
Similar to \Cref{alg:stat-inf-spnd}, \Cref{alg:stat-inf:high-dim:linear:proximal:svrg} has an outer loop part and an inner loop part, 
where the outer loops correspond to  approximate proximal Newton steps, 
and the inner loops solve each proximal Newton step using proximal SVRG \cite{xiao2014proximal}.  
To control the variance, we use SVRG and proximal SVRG to solve the Newton steps. This is because in the high dimensional setting, the variance is too large when we use SGD \cite{moulines2011non} and proximal SGD \cite{atchade2017perturbed} for solving Newton steps. 
However, since  we have $p \gg n$ , instead of sampling {\em by sample}, we sample {\em by feature}.  
When we set $L_o^t = \Theta(\log (p) \cdot \log (t))$,  
we can estimate the statistical error covariance with element-wise error less than $ O\left( \frac{\max\{1,\sigma\}\polylog(n,p)}{\sqrt{T}} \right)$ with high probability, 
using $O\left(T \cdot n \cdot p^2 \cdot \log(p) \cdot \log(T) \right)$ numerical operations. 
And \Cref{alg:stat-inf:high-dim:linear:calc:de-biased:svrg} calculates the de-biased estimator $\htheta^\diff$ \eqref{eq:lasso-mod-cov:de-bias:theta} via SVRG. 
For more details, we defer the reader to \Cref{sec:appendix:lasso:linear-regression:stat-inf}.

\section{Time series analysis}
\label{sec:time-series}
In this section, we present a sampling scheme for statistical inference in time series analysis using $M$-estimation, 
where we sample contiguous blocks of indices, instead of uniformly.

We consider a linear model $y_i = \langle x_i , \theta^\star \rangle + \epsilon_i $, where $\Exp[\epsilon_i x_i] = 0$, but $\{x_i, y_i\}_{i=1}^n$ may not be i.i.d. as this is a time series. 
 And we use ordinary least squares (OLS) 
 $\htheta = \arg \min_\theta \sum_{i=1}^n \frac{1}{2} \left( \langle x_i , \theta \rangle -  y_i \right)^2 $  
 to estimate $\theta^\star$. 
Applications include  multifactor financial models for explaining returns \cite{bender2013foundations, rosenberg1973prediction}. 
  For non-i.i.d. time series data, OLS may not be the optimal estimator, as opposed to the maximum likelihood estimator \cite{shumway2011time}, but OLS is simple yet often robust, 
compared to more sophisticated models that take into account time series dynamics.   
 And it is widely used in econometrics for time series analysis \cite{berndt1991practice}. 
 To perform statistical inference, we use the asymptotic normality 
\begin{align}
\sqrt{n} \left( \htheta - \theta^\star \right) \to \Norm \left( 0 , {H^\star}^{-1} G^\star {H^\star}^{-1} \right) , 
\end{align}
where $H^\star  = \lim_{n\to \infty} \frac{1}{n} \left( \sum_{i=1}^n \nabla^2 f_i(\theta^\star) \right) $ and 
$G^\star =  \lim_{n\to \infty} \frac{1}{n} \left( \sum_{i=1}^n \sum_{j=1}^n \nabla f_i(\theta^\star) \, \nabla f_j(\theta^\star)^\top \right) $, with $f_i(\theta) =  \frac{1}{2} \left( \langle x_i , \theta \rangle -  y_i \right)^2  $. 
The difference compared with the i.i.d. case (\Cref{sec:unregularized-m-est}) is that $G^\star$ now includes autocovariance terms. 
 We use the plug-in estimate $\widehat{H} = \frac{1}{n} \sum_{i=1}^n \nabla^2 f_i(\htheta)$ as before, and we estimate $G^\star$ using the Newey-West covariance estimator
\cite{newey1986simple} for HAC (heteroskedasticity and autocorrelation consistent) covariance estimation 
\begin{align}
\textstyle \widehat{G} = \tfrac{1}{n} \sum\limits_{i=1}^n \nabla f_i(\htheta) \, f_i(\htheta)^\top 
	+ \sum\limits_{j=1}^{\mathbf{l}} w(j, \mathbf{l}) \sum\limits_{i=j+1}^n  \left( 
		\nabla f_i(\htheta) \,  \nabla f_{i - j}(\htheta)^\top 
			 + \nabla f_{i-j}(\htheta) \,  \nabla f_i(\htheta)^\top
		 \right)  , \label{eq:newey-west}
\end{align}
where $w(j, \mathbf{l})$ is sample autocovariance weight, such as Bartlett weight $w(j, \mathbf{l}) = 1 - \sfrac{j}{(\mathbf{l} + 1)}$, 
and $\mathbf{l}$ is the {\em lag} parameter, 
 which captures data dependence across time. 
 Note that this is an essential building block in time series statistical inference procedures, such as Driscoll-Kraay standard errors \cite{driscoll1998consistent, kraay1999spatial}, moving block bootstrap \cite{kunsch1989jackknife}, and circular bootstrap \cite{politis1992circular, politis1994stationary}.

In our framework, we solve OLS using our approximate Newton procedure with a slight modification to \Cref{alg:stat-inf-spnd}. 
Instead of uniformly sampling indices as in line \ref{alg:stat-inf-spnd:sgd1} of \Cref{alg:stat-inf-spnd}, 
we uniformly select some $i_o \in [n]$, 
and set the outer mini-batch indexes $I_o$ to the random contiguous block $\{i_o, i_o + 1, \dots, i_o + \mathbf{l} - 1 \} \mod n$, where we let the indexes circularly wrap around, as in line \ref{alg:stat-inf-spnd:sgd1:time-series} of \Cref{alg:stat-inf-spnd:time-series},  
and this sampling scheme is similar to {\em circular bootstrap}. 
Here $\mathbf{l}$ is the lag parameter, similar to the Newey-West estimator.
And the stochastic gradient's expectation is still the full gradient.  
The complete algorithm is in \Cref{alg:stat-inf-spnd:time-series}, and its guarantees are given in \Cref{cor:time-series:spnd-bound}.  
Our approximate Newton statistical inference procedure 
 is equivalent to using weight $w(j, \mathbf{l}) = 1 -\sfrac{j}{\mathbf{l}}$ in the Newey-West covariance estimator \eqref{eq:newey-west}, with negligible terms for blocks that wrap around, and this is the same as circular bootstrap.
 Note that the connection between sampling scheme and Newey-West estimator was also observed in \cite{kunsch1989jackknife}.  
 Following \cite{politis1992circular}, 
 we can set the lag parameter such that  $\mathbf{l} \cdot n^{-\sfrac{1}{3}} \to 0 $, 
 and run at least $n$ outer loops. 
 In practice, other methods for tuning the lag parameter can be used, such as  \cite{newey1994automatic}. 
 For more details, we refer the reader to \Cref{appendix:sec:time-series}.

\section{Experiments}
\label{sec:experiments}

\subsection{Synthetic data}
\label{subsec:experiments:synthetic}

The coverage probability is defined as
$
\frac{1}{p} \sum_{i=1}^p \Pr[\theta^\star_i \in \hat{C}_i],
$
where $\hat{C}_i$ is the estimated confidence interval for the $i^{\text{th}}$ coordinate.
The average confidence interval length is defined as
$
\frac{1}{p} \sum_{i=1}^p (\hat{C}_i^u - \hat{C}_i^l),
$
where $[\hat{C}_i^l , \hat{C}_i^u]$ is the estimated confidence interval for the $i^{\text{th}}$ coordinate.
In our experiments, coverage probability and average confidence interval length are estimated through simulation.
  Result  given as a   $(\alpha, \beta)$  indicates (coverage probability, confidence interval length).

\begin{table}[!t]
\centering
\rowcolors{2}{white}{black!05!white}
\resizebox{\textwidth}{!}{
 \begin{tabular}{c c c c c c c  c c}
 \toprule
 & & Approximate Newton & & Bootstrap & & Inverse Fisher information & & Averaged SGD \\
 \cmidrule{3-3} \cmidrule{5-5} \cmidrule{7-7} \cmidrule{9-9}
Lin1 & & (0.906, 0.289) & & (0.933, 0.294) & & (0.918, 0.274) & & (0.458, 0.094) \\
Lin2 & & (0.915, 0.321) & & (0.942, 0.332) & & (0.921,0.308) & & (0.455 0.103) \\
 \end{tabular}
} 
\caption{Linear regression (low dimensional): synthetic data confidence interval (coverage, length)}\label{tab:exp:coverage:linear}
\end{table}

\begin{table}[!t]
\rowcolors{2}{white}{black!05!white}
\centering
\resizebox{\textwidth}{!}{
 \begin{tabular}{c c c c c c c  c c}
 \toprule
 & & Approximate Newton & & Jackknife & & Inverse Fisher information & & Averaged SGD \\
 \cmidrule{3-3} \cmidrule{5-5} \cmidrule{7-7} \cmidrule{9-9}
Log1 & & (0.902, 0.840) & & (0.966  1.018) & & (0.938, 0.892) & & (0.075 0.044) \\
Log2 & & (0.925, 1.006) & & (0.979, 1.167) & & (0.948, 1.025) & & (0.065 0.045) \\
 \bottomrule
 \end{tabular}
} 
\caption{Logistic regression (low dimensional): synthetic data confidence interval (coverage, length)}\label{tab:exp:coverage:logistic}
\end{table}

\paragraph{Low dimensional problems.}
\Cref{tab:exp:coverage:linear} and \Cref{tab:exp:coverage:logistic} show 95\% confidence interval's coverage and length of 200 simulations for linear and logistic regression. 
The exact configurations for linear/logistic regression examples are provided in \Cref{append:subsubsec:exp:sim:low-dim}. 
Compared with Bootstrap and Jackknife \cite{efron1994introduction}, \Cref{alg:stat-inf-spnd} uses less numerical operations, while achieving similar results. 
Compared with the averaged SGD method \cite{li2017statistical, chen2016statistical}, our algorithm performs much better, while using the same amount of computation, 
and is much less sensitive to the choice hyper-parameters. 
And we observe that  calibrated approximate Newton confidence intervals \cite{efron1994introduction, politis2012subsampling} are better than bootstrap and inverse Fisher information (\Cref{tab:exp:coverage:calibrated}).

\begin{figure}[t!]
    \centering
        \includegraphics[width=0.4\textwidth]{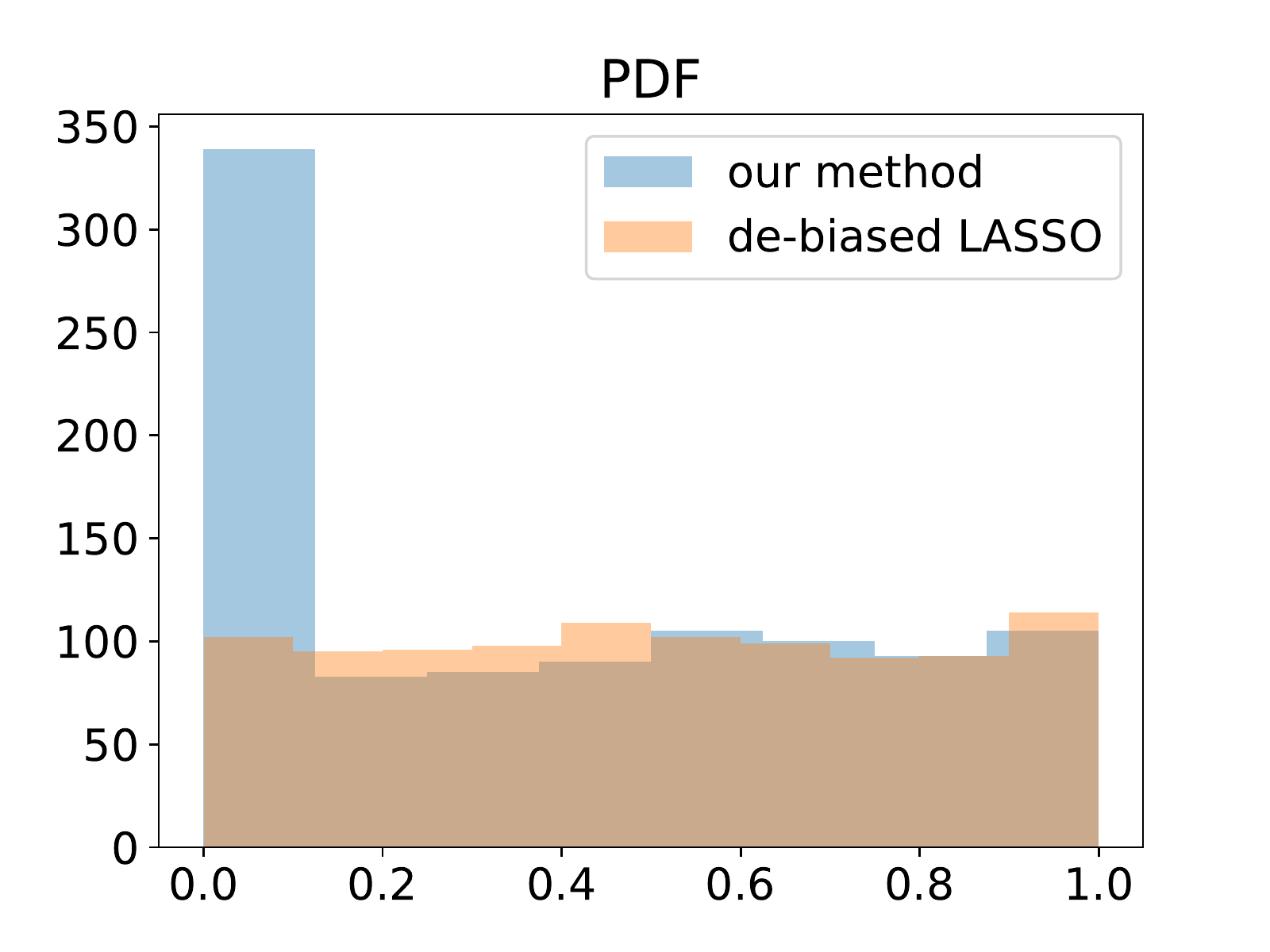}
        \includegraphics[width=0.4\textwidth]{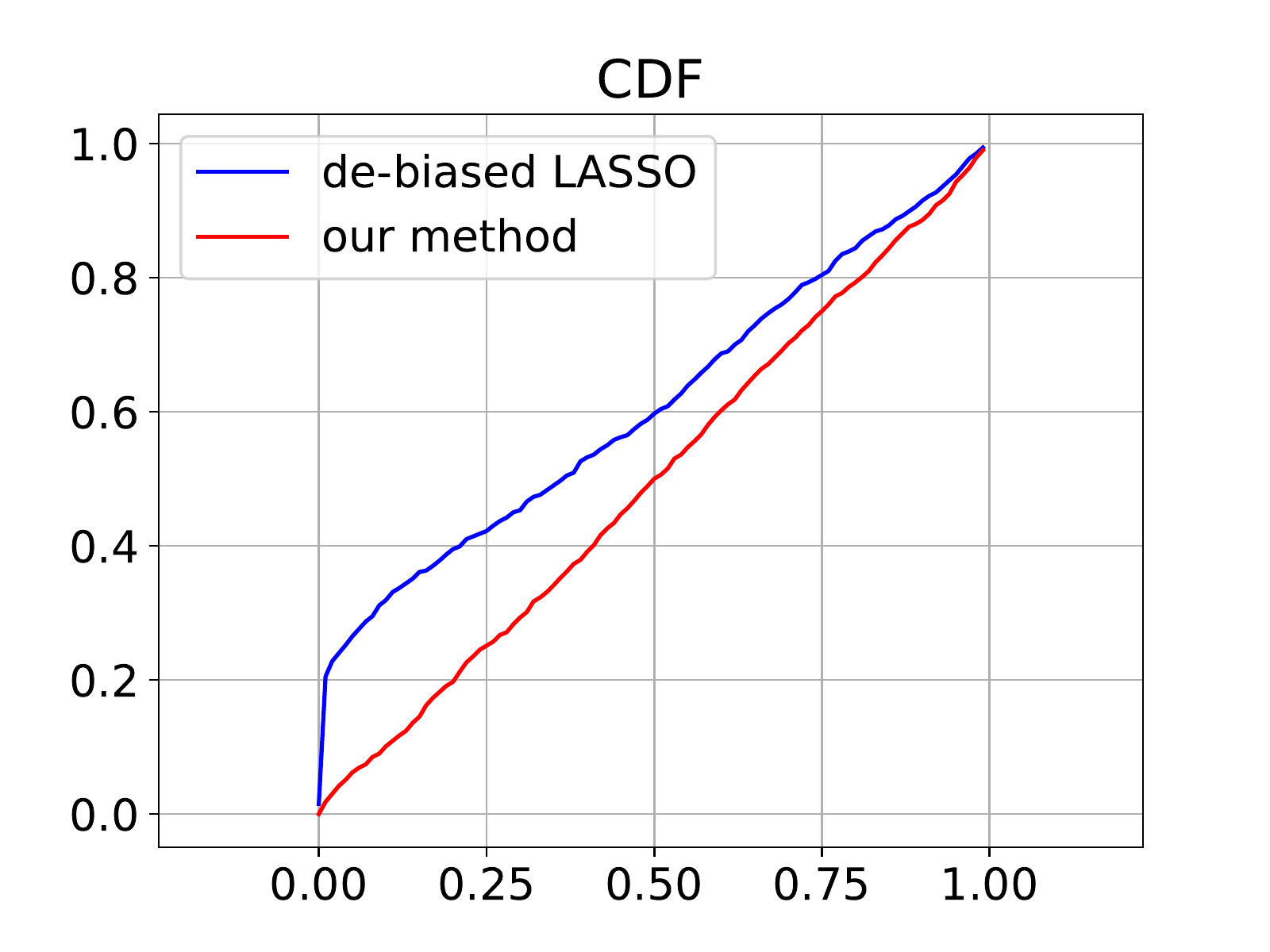}
    \caption{Distribution of two-sided Z-test p-values  under the null hypothesis (high dimensional) }
    \label{fig:exp:high-dim-linear:0:z-test:2-sided:p-value}
\end{figure}

\paragraph{High dimensional linear regression.}
\Cref{fig:exp:high-dim-linear:0:z-test:2-sided:p-value} shows p-value distribution  under the null hypothesis for our method and the de-biased LASSO estimator with known covariance, using  600 i.i.d. samples generated from a model with $\Sigma = I$, $\sigma = 0.7$, 
and we can see that it is close to a uniform distribution, similar results are observed for other high dimensional statistical inference procedures such as \cite{candes2018panning}.
%
And visualization of confidence intervals computed by our algorithm is shown in \Cref{fig:exp:sim:high-dim-linear:CI:plot}. 
 
\paragraph{Time series analysis.} 
In our linear regression simulation, we generate i.i.d. random explanatory variables, 
and the observation noise is a 0-mean moving average (MA) process independent of the explanatory variables. 
Results on average 95\% confidence interval coverage and length are given in \Cref{append:subsubsec:exp:sim:time-series}, and they validate our theory.

\subsection{Real data}
\label{subsec:experiments:real}

\paragraph{Neural network adversarial attack detection.} 
Here we use ideas from statistical inference to detect certain adversarial attacks on neural networks. 
A key  observation is that neural networks are effective at representing low dimensional manifolds such as natural images \cite{basri2016efficient,chui2016deep}, 
and this causes the risk function's Hessian to be degenerate \cite{sagun2017empirical}.  
From a statistical inference perspective, 
we interpret this as meaning that the confidence intervals in the null space of $H^+ G H^+$ is infinity, where $H^+$ is the pseudo-inverse of the Hessian (see \Cref{sec:unregularized-m-est}). 
When we make a prediction $\Psi(x; \htheta)$ using a fixed data point $x$ as input 
(i.e., conditioned on $x$), 
using the delta method \cite{van1998asymptotic}, 
the confidence interval of the prediction can be derived from the asymptotic normality of $\Psi(x; \htheta)$  
$$
\sqrt{n} \left( \Psi(x; \htheta) - \Psi(x; \theta^\star) \right)  
	 \to  \Norm\left( 0 , \nabla_\theta \Psi(x; \htheta)^\top \left[ \widehat{H}^{-1} \widehat{G}  \widehat{H}^{-1}  \right]  \nabla_\theta \Psi(x; \htheta) \right) .
$$
To detect adversarial attacks, 
we use the score 
$$
\tfrac{\left\|\left( I - P_{H^+ G H^+} \right) \nabla_\theta \Psi(x; \htheta) \right\|_2}{\left\|\nabla_\theta \Psi(x; \htheta)\right\|_2} , 
$$  
to measure how much $\nabla_\theta \Psi(x; \htheta)$ lies in null space of $H^+ G H^+$,
where $P_{H^+ G H^+}$ is the projection matrix onto the range of $H^+ G H^+$. 
Conceptually, for the same image, the randomly perturbed image's score should be larger than the original image's score, and the adversarial image's score should be larger than the randomly perturbed image's score.

We train a binary classification neural network with 1 hidden layer and softplus activation function, to distinguish between ``Shirt'' and  ``T-shirt/top'' in the Fashion MNIST data set \cite{xiao2017fashion}.  
\Cref{fig:exp:real-data:nn:adversarial-scores}  shows distributions of scores of original images,   adversarial images generated using the fast gradient sign method \cite{goodfellow2014explaining}, and randomly perturbed images. 
Adversarial and random perturbations have the same $\ell_\infty$ norm. 
The adversarial perturbations and example images are shown in \Cref{append:subsubsec:exp:nn:adversarial}. 
Although  the scores' values are small, 
they are still significantly larger than 64-bit floating point  precision ($2^{-53} \approx 1.11 \times 10^{-16}$). 
We observe that scores of randomly perturbed images is an order of magnitude larger than scores of original images, 
and scores of adversarial images is an order of magnitude larger than scores of randomly perturbed images.

\begin{figure}[!t]
\centering
\includegraphics[width=.4\textwidth]{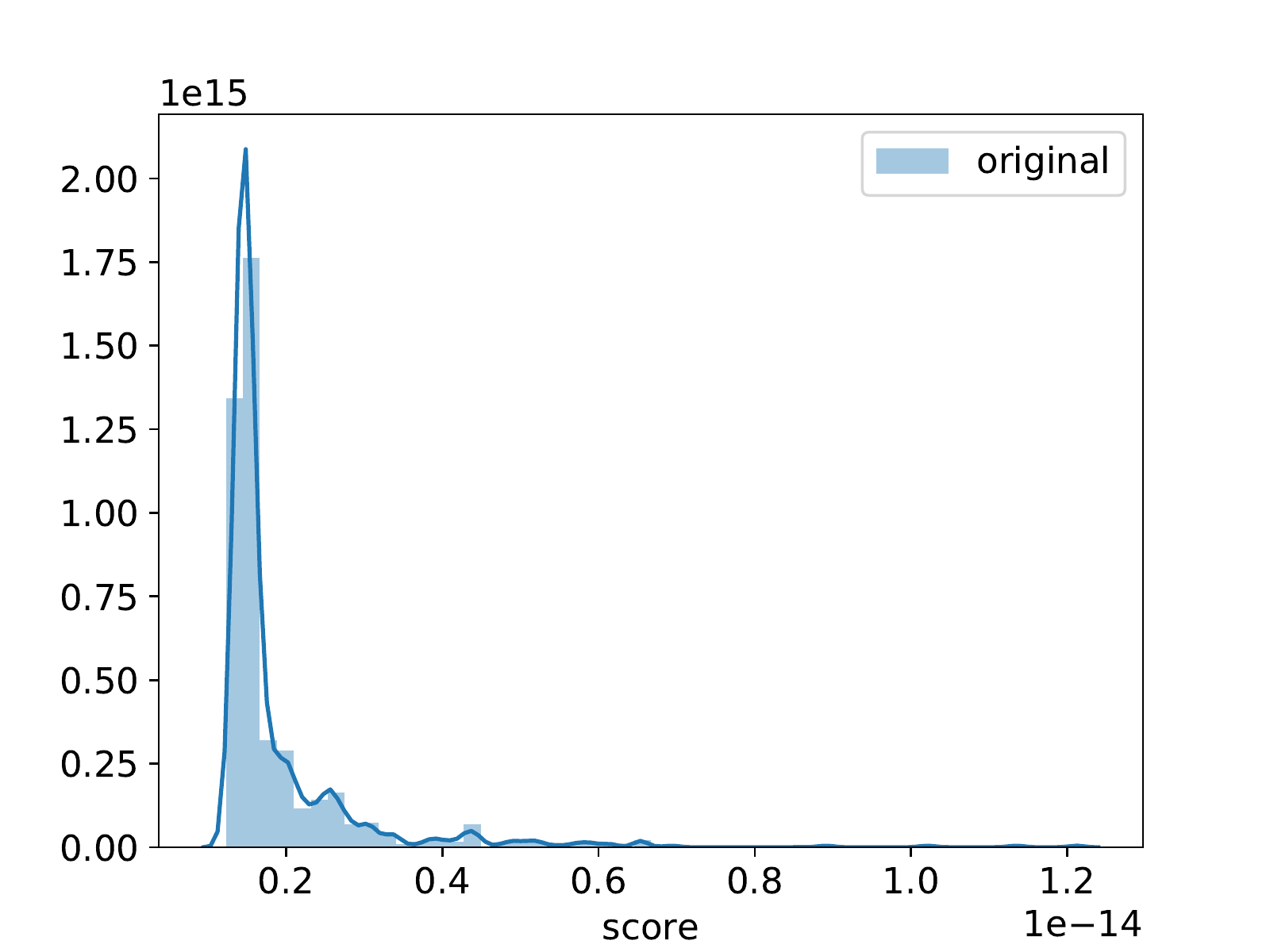}
\includegraphics[width=.4\textwidth]{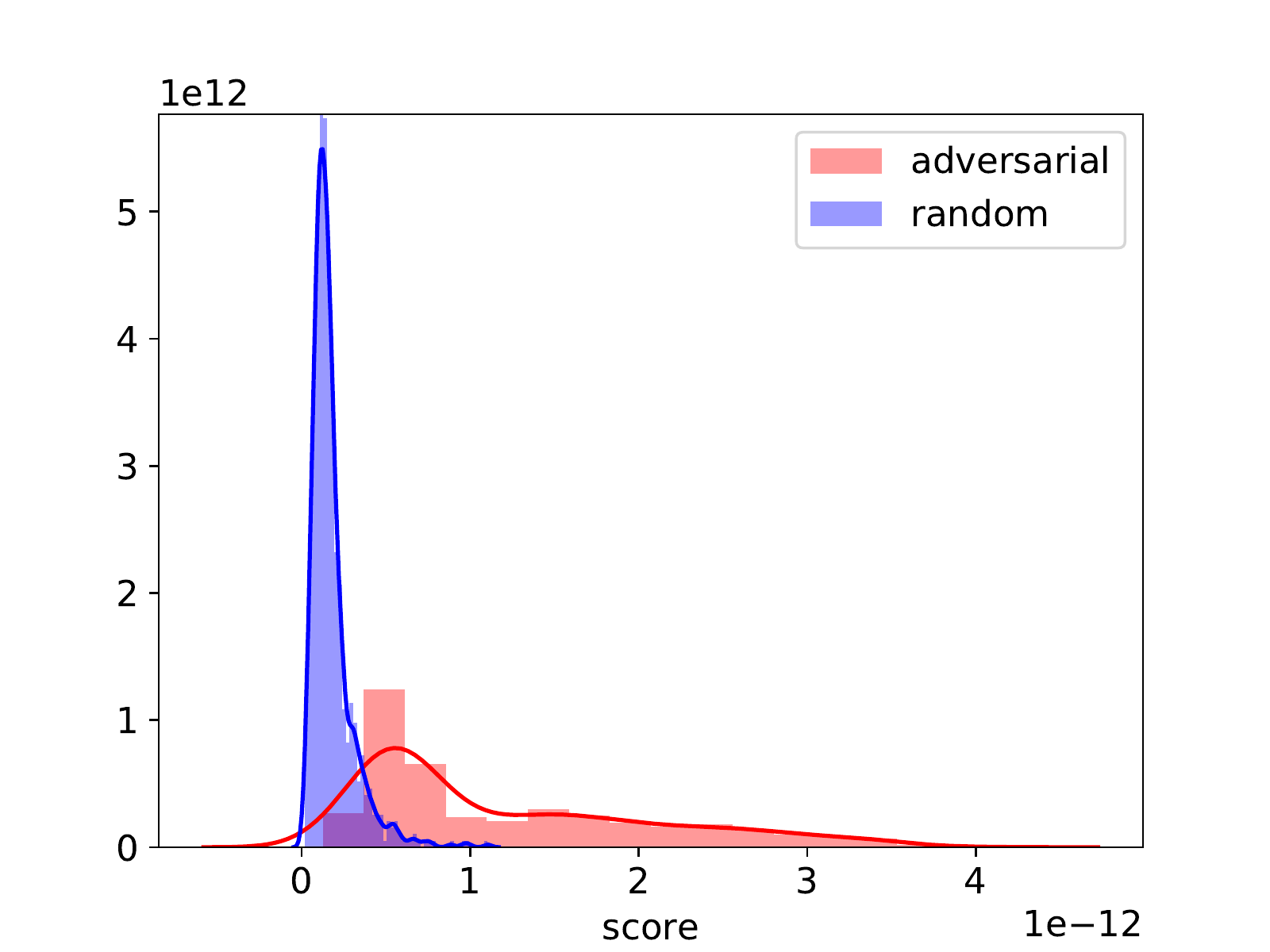}
\caption{Distribution of scores for original, randomly perturbed, and adversarially perturbed images}
\label{fig:exp:real-data:nn:adversarial-scores}
\end{figure}

\paragraph{High dimensional linear regression.}
We apply our high dimensional inference procedure to the dataset in \cite{rhee2006genotypic}  to detect mutations related to HIV drug resistance, where we randomly sub-sample the dataset so that the number of features is larger than the number of samples. 
When we control the family-wise error rate (FWER) at $0.05$ using the Bonferroni correction \cite{bonferroni1936teoria},  
our procedure is able to detect verified mutations in an expert dataset \cite{johnson2005update} (\Cref{tab:exp:real:high-dim:hiv}), 
and the details are given in \Cref{append:subsubsec:exp:real:high-dim:riboflavin}.
Another experiment with a 
 genomic data set concerning riboflavin (vitamin B2) production rate \cite{buhlmann2014high} is given in the appendix. 

\paragraph{Time series analysis.} 
Using monthly equities returns data from \cite{frazzini2014betting}, 
we use our approximate Newton statistical inference procedure to show that the correlation between US equities market returns and non-US global equities market returns is statistically significant, 
which validates the capital asset pricing model (CAPM) \cite{sharpe1964capital, lintner1965valuation, fama2004capital}. 
The details 
 are given in \Cref{append:subsubsec:exp:real:time-series}.

\section{Related work}\label{sec:realted-work}

\paragraph{Unregularized M-estimation.}
This work
provides a general, flexible framework for {\em simultaneous} point estimation and statistical inference, 
and improves upon previous methods, based on averaged  stochastic gradient descent \cite{li2017statistical, chen2016statistical}. 
%

Compared to \cite{chen2016statistical} (and similar works \cite{su2018statistical,fang2017scalable} using SGD with decreasing step size), our method does not need to increase the lengths of  ``segments'' (inner loops) to reduce correlations between different ``replicates''. 
Even in that case, if we use $T$ replicates and increasing ``segment'' length (number of inner loops is $t^{\frac{d_o}{1-d_o}} \cdot L$) with a total of $O(T^{\frac{1}{1-d_o}}\cdot L)$ stochastic gradient steps,  
\cite{chen2016statistical} guarantees $O(L^{-\frac{1-d_o}{2}}+T^{-\frac{1}{2}} + T^{\max\{\frac{1}{2} -\frac{d_o}{4(1-d_o)},0\} - \frac{1}{2}} \cdot L^{-\frac{d_o}{4}} + T^{\max\{\frac{1-2d_o}{2(1-d_o)},0\} - \frac{1}{2}} \cdot L^{\frac{1-2d_o}{2}})$ , whereas our method guarantees $O(T^{-\frac{d_o}{2}})$. 
Further, \cite{chen2016statistical} is inconsistent, whereas our scheme  guarantees consistency of computing the statistical error covariance. 
%

\cite{li2017statistical} uses fixed step size SGD for statistical inference, and discards iterates between different ``segments'' to reduce correlation, whereas we do not discard any iterates in our computations. 
Although  \cite{li2017statistical} states empirically constant step SGD performs well in statistical inference, it has been empirically shown \cite{dieuleveut2017bridging} that averaging consecutive iterates in constant step SGD does not guarantee convergence to the optimal -- the average will be ``wobbling'' around the optimal, whereas decreasing step size stochastic approximation methods (\cite{polyak1992acceleration, ruppert1988efficient} and our work) will converge to the optimal, and averaging consecutive iterates guarantees ``fast'' rates. 
%

Finally, from an optimization perspective, our method is similar to stochastic Newton methods (e.g. \cite{agarwal2017second}); however, our method only uses first-order information to approximate a Hessian vector product ($\nabla^2 f(\theta) v \approx \tfrac{\nabla f(\theta + \delta v) - \nabla f(\theta)}{\delta}$). 
\Cref{alg:stat-inf-spnd}'s outer loops are similar to stochastic natural gradient descent \cite{amari1998natural}. 
Also, we demonstrate an intuitive view of SVRG \cite{johnson2013accelerating} as a special case of  approximate stochastic Newton steps using first order information (\Cref{append:sec:svrg-intuition}). 


\paragraph{High dimensional linear regression.}
%
\cite{chen2016statistical}'s high dimensional inference algorithm is based on \cite{agarwal2012stochastic}, 
and only guarantees that optimization error is at the same scale as the statistical error. 
However, proper de-biasing of the LASSO estimator requires the optimization error to be much less than the statistical error, 
otherwise the optimization error introduces additional bias that de-biasing cannot handle. 
Our optimization objective  is strongly convex with high probability: this permits the use of linearly convergent proximal algorithms \cite{xiao2014proximal,lee2014proximal} towards the optimum, which guarantees the optimization error to be much smaller than the statistical error. 
%

Our method of de-biasing the LASSO in \Cref{sec:high-dim:lasso:linear-regression} is similar to \cite{zhang2014confidence, van2014asymptotically,javanmard2014confidence, javanmard2015biasing}. 
Our method uses a new $\ell_1$ regularized objective \eqref{eq:lasso-mod-cov} for high dimensional linear regression, 
and we have different de-biasing terms, because we also need to de-bias the covariance estimation. 
In \Cref{alg:stat-inf:high-dim:linear:proximal:svrg}, our covariance estimate is similar to the classic {\em sandwich estimator} \cite{huber1967behavior,white1980heteroskedasticity}. 
Previous methods require $O(p^2)$ space which unsuitable for large scale problems, 
whereas our method only requires $O(p)$ space. 
%

Similar to our $\ell_1$-norm regularized objective, \cite{yang2014elementary, jeng2011sparse} shows similar point estimate statistical guarantees for related estimators; however there are no confidence interval results.  
Further, although \cite{yang2014elementary} is an elementary estimator in closed form, it still requires computing the inverse of the thresholded covariance, which is challenging in high dimensions, and may not computationally outperform optimization approaches.  
%
%

Finally, for feature selection, we do not assume that absolute values of the true parameter's non-zero entries are lower bounded (``beta-min'' condition). \cite{fan2018statistical,loh2017statistical,loh2017support,buhlmann2011statistics,wainwright2009sharp}. 

\paragraph{Time series analysis.} 
Our approach of sampling contiguous blocks of indices to compute stochastic gradients for statistical inference in time series analysis is similar to resampling procedures in {\em moving block} or {\em circular} bootstrap  \cite{carlstein1986use, kunsch1989jackknife, buhlmann2002bootstraps, davison1997bootstrap, efron1994introduction, lahiri2013resampling, politis1992circular, politis1994stationary, kreiss2012bootstrap}, and {\em conformal prediction} \cite{balasubramanian2014conformal,shafer2008tutorial,vovk2005conformal}. 
Also, our procedure is similar to Driscoll-Kraay standard errors \cite{driscoll1998consistent, kraay1999spatial, hoechle2007robust}, but does not waste computational resources to explicitly store entire matrices, 
and is suited for large scale time series analysis.

{
\newpage
\small
\bibliography{ref.bib}
\bibliographystyle{alpha}
}


\appendix


%
%
%


\newpage

\section{High dimensional linear regression statistical inference using stochastic gradients (\Cref{sec:high-dim:lasso:linear-regression})}
\label{sec:appendix:lasso:linear-regression:stat-inf}

\subsection{Statistical inference using approximate proximal Newton steps with stochastic gradients}
\label{sec:high-dim:lasso:linear-regression:alg-fo-a-s-p-nd}
Here, we present a statistical inference procedure for high dimensional linear regression via  approximate proximal Newton steps using stochastic gradients. 
It uses the plug-in estimator:
$$\widehat{S}^{-1} \left( \tfrac{1}{n} \sum_{i=1}^n (x_i^\top \htheta - y_i)^2 x_i x_i^\top \right) \widehat{S}^{-1} , $$
which is related to the empirical sandwich estimator \cite{huber1967behavior,white1980heteroskedasticity}. \Cref{lem:sec:high-dim:lasso:linear-regression:plug-in:sandwich} shows this is a good estimate of the covariance when $n \gg \tfrac{1}{{D_\Sigma}^4} \max\{1, \sigma^2\} s^2 (\sigma+\|\theta^\star\|_1)^2$.

\Cref{alg:stat-inf:high-dim:linear:proximal:svrg} performs statistical inference in high dimensional linear regression \eqref{eq:lasso-mod-cov}, by computing the statistical error covariance in \Cref{thm:lasso-mod:de-bias:stat-inf}, based on the plug-in estimate in \Cref{lem:sec:high-dim:lasso:linear-regression:plug-in:sandwich}. 
We denote the soft thresholding of $A$ by $\omega$  as an element-wise  procedure $\left(\Shrink_\omega(A)\right)_e = \sign(A_e)(|A_e| - \omega)_+$. 
For a vector $v$, we write $v$'s $i$\textsuperscript{th} coordinate as $v(i)$. 
The optimization objective  \eqref{eq:lasso-mod-cov}  is denoted as:
$$\tfrac{1}{2} \theta^\top \left( \widehat{S} - \tfrac{1}{n} \tsum_{i=1}^n x_i x_i^\top \right) \theta + \tfrac{1}{n} \tsum_{i=1}^n f_i,$$ 
where $f_i = \tfrac{1}{2} \left( x_i^\top - y_i \right)^2.$ 
Further, 
$$\g_{\widehat{S}} (v) =   \nabla_v \left[ \tfrac{1}{2} v^\top \widehat{S}  v \right] = \widehat{S} v =  \sum_{j=1}^p v(j) \cdot \Shrink_\omega   \left(    \tfrac{1}{n}  \sum_{i=1}^n  \left[ \nabla f_i(\theta + \e_j) - \nabla f_i(\theta)  \right]  \right),$$
where $\e_i \in \Real^p$ is the basis vector where the $i$\textsuperscript{th} coordinate is $1$ and others  are $0$, and $\widehat{S} v$ is computed in a column-wise manner. 

For point estimate optimization, the proximal Newton step \cite{lee2014proximal} at $\theta$ solves the  optimization problem 
$$\min_\Delta  \tfrac{1}{2 \rho} \Delta^\top \widehat{S} \Delta + \left\langle ( \widehat{S} - \tfrac{1}{n} \tsum_{i=1}^n x_i x_i^\top ) \theta + \tfrac{1}{n} \sum_{i=1}^n \nabla f_i(\theta) , \Delta \right\rangle  + \lambda \| \theta + \Delta \|_1,$$ 
to determine a descent direction.
For statistical inference, we solve a Newton step:
$$\min_\Delta \tfrac{1}{2 \rho} \Delta^\top \widehat{S} \Delta + \left\langle \tfrac{1}{S_o} \sum_{k \in I_o} \nabla f_k(\theta_t) , \Delta \right\rangle $$ 
to compute $-\widehat{S}^{-1} \frac{1}{S_o} \sum_{i \in I_o} \nabla f_i(\theta) $, whose covariance is the statistical error covariance. 

To control  variance, we solve Newton steps using SVRG and proximal SVRG \cite{xiao2014proximal}, because in the high dimensional setting, the variance using SGD \cite{moulines2011non} and proximal SGD \cite{atchade2017perturbed} for solving Newton steps is too large. 
However because $p \gg n$, instead of sampling {\em by sample}, we sample {\em by feature}. 
We start from $\theta_0$ sufficiently close to $\htheta$ (see \Cref{thm:alg:stat-inf:high-dim:linear:proximal:svrg} for details), which can be effectively achieved using proximal SVRG (\Cref{subsec:high-dimensional-linear:lasso-:mod-cov:opt:svrg}). 
Line \ref{line-alg:stat-inf:high-dim:linear:proximal:svrg:newton:stat-inf:1} corresponds to SVRG's outer loop part that computes the full gradient,  
and line \ref{line-alg:stat-inf:high-dim:linear:proximal:svrg:newton:stat-inf:2} corresponds to SVRG's inner loop update. 
Line \ref{line-alg:stat-inf:high-dim:linear:proximal:svrg:newton:opt:1} corresponds to proximal SVRG's outer loop part that computes the full gradient,  
and line \ref{line-alg:stat-inf:high-dim:linear:proximal:svrg:newton:opt:2} corresponds to proximal SVRG's inner loop update. 

The covariance estimate bound, asymptotic normality result, and choice of hyper-parameters are described in \Cref{subsec:high-dimensional:linear:bound-covariance:asymptotic-normality}. 
When $L_o^t = \Theta(\log (p) \cdot \log (t))$, 
we can estimate the covariance with element-wise error less than $ O\left( \frac{\max\{1,\sigma\}\polylog(n,p)}{\sqrt{T}} \right)$ with high probability, 
using $O\left(T \cdot n \cdot p^2 \cdot \log(p) \cdot \log(T) \right)$ numerical operations.
Calculation of  the de-biased estimator $\htheta^\diff$ \eqref{eq:lasso-mod-cov:de-bias:theta} via SVRG is described in \Cref{subsec:high-dimensional:linear:de-biased:calculation}. 

{
\begin{algorithm}
\begin{algorithmic}[1]
\STATE \textbf{Parameters:} $S_o, S_i \in \mathbb{Z}_+$; $\eta, \tau \in \mathbb{R}_+$;  \quad   \textbf{Initial state:} $\theta_0 \in \mathbb{R}^p$ \\
{\kern1.5pt \hrule \kern2pt} 
\FOR{$t=0$ \TO $T-1$} 
	\STATE $I_o \gets$ uniformly sample $S_o$ indices with replacement from $[n]$ 
	\STATE $g_t^0 \gets - \frac{1}{S_o} \sum_{k \in I_o }  \nabla f_k(\theta_t ) $ 
	\STATE $d_t^0 \gets  - \left( \g_{\widehat{S}}(\theta_t  ) - \frac{1}{n} \sum_{i=1}^n  \left[ \nabla f_i(\theta_t + \theta_t) - \nabla f_i(\theta_t) \right] + \tfrac{1}{n} \tsum_{i=1}^n \nabla f_i(\theta_t  )\right) $ 
	\FOR[solving  Newton steps using SVRG]{$j=1$ \TO $L_o^t -$}  \label{line-alg:stat-inf:high-dim:linear:proximal:svrg:newton:start}
		\STATE $u_t^j \gets  \g_{\widehat{S}} (g_t^{j-1}) - g_t^0$ \label{line-alg:stat-inf:high-dim:linear:proximal:svrg:newton:stat-inf:1}
		\STATE $v_t^j \gets \g_{\widehat{S}}(d_t^{j-1}) - d_t^0 $ \label{line-alg:stat-inf:high-dim:linear:proximal:svrg:newton:opt:1}
		\STATE $g_t^j \gets g_t^{j-1}$, $d_t^j \gets d_t^{j-1}$
		\FOR{$l=1$ \TO $L_i $}
			\STATE $I_i \gets$ uniformly sample $S_i$ indices without replacement from $[p]$
			\STATE \small{$g_t^j \gets g_t^j   - \tau \left[ u_t^j  +  \frac{p}{S_i} \sum_{k \in S_i} \left[ g_t^j(k) - g_t^{j-1}(k) \right] \cdot \Shrink_\omega \left( \nabla f_k(\theta_t + \e_k) - \nabla f_k(\theta_t) \right) \right]  $} \label{line-alg:stat-inf:high-dim:linear:proximal:svrg:newton:stat-inf:2}
			\STATE  \small{$d_t^j \gets \Shrink_{\eta \lambda} \left( d_t^j   - \eta\left[ v_t^j  +  \frac{p}{S_i} \sum_{k \in S_i} \left[ d_t^j(k) - d_t^{j-1}(k)  \right]\cdot \Shrink_\omega \left( \nabla f_k(\theta_t + \e_k) - \nabla f_k(\theta_t) \right) \right]  \right) $} \label{line-alg:stat-inf:high-dim:linear:proximal:svrg:newton:opt:2}  
		\ENDFOR  
	\ENDFOR \label{line-alg:stat-inf:high-dim:linear:proximal:svrg:newton:end}
	\STATE Use $\sqrt{S_o} \cdot \frac{\bar{g}_t}{\rho_t}$ for statistical inference, where $\bar{g}_t = \frac{1}{L_o^t + 1} \sum_{j=0}^{L_o^t} g_t^j$  
	\STATE $\theta_{t+1} = \theta_t + \bar{d}_t$, where $\bar{d}_t = \frac{1}{L_o + 1} \sum_{j=0}^{L_o^t} d_t^j$  \COMMENT{point estimation (optimization)}
\ENDFOR
\end{algorithmic}
\caption{High dimensional linear regression statistical inference}
\label{alg:stat-inf:high-dim:linear:proximal:svrg}
\end{algorithm}
}


%

\subsection{Computing the de-biased estimator \eqref{eq:lasso-mod-cov:de-bias:theta} via SVRG}
\label{subsec:high-dimensional:linear:de-biased:calculation}

To control  variance, we solve each proximal Newton step using SVRG, in stead of SGD as in \Cref{alg:stat-inf-spnd}. 
Because 
However because the number of features is much larger than the number of samples, instead of sampling {\em by sample}, we sample {\em by feature}. 

The de-biased estimator  is 
\begin{align*}
\htheta^\diff =&   \htheta + \widehat{S}^{-1} \left[  \frac{1}{n} \sum_{i=1}^n y_i x_i - \left( \frac{1}{n}\sum_{i=1}^n x_i x_i^\top \right) \htheta   \right]  \nonumber \\
=& \htheta + \widehat{S}^{-1} \left( -\frac{1}{n} \sum_{i=1}^n \nabla f_i(\htheta) \right) . 
\end{align*}
And we compute $\widehat{S}^{-1} \frac{1}{n} \sum_{i=1}^n \nabla f_i(\htheta)$ using SVRG \cite{johnson2013accelerating} by solving the following optimization problem using SVRG and sampling by feature
\begin{align*}
\min_u \frac{1}{2} u^\top \widehat{S} u + \left\langle  \frac{1}{n} \sum_{i=1}^n \nabla f_i(\htheta) , u \right\rangle . 
\end{align*}

\begin{algorithm}
\begin{algorithmic}[1]
\FOR{$i=0$ \TO $L_o-1$} 
	\STATE $d_i^0 \gets -\eta [ \g_{\widehat{S}}(u_i) + \frac{1}{n} \sum_{k=1}^n \nabla f_k(\htheta) ] $
	\FOR{$j=0$ \TO $L_i - 1$}
		\STATE $I \gets$ sample $S $ indices  uniformly from $[p]$ without replacement
		\STATE $d_i^{j+1} \gets d_i^j + d_t^0 - \eta \left( \frac{1}{S} \sum_{k \in I} d_i^j(k) \cdot \Shrink_\omega(\nabla f_k(\htheta + \e_k) - f_k(\htheta)) \right)$
	\ENDFOR
	\STATE $u_{i+1}\gets u_i + \bar{d}_i$, where $\bar{d}_i = \frac{1}{L_i+1}\sum_{j=0}^{L_i} d_i^j$
\ENDFOR
\end{algorithmic}
\caption{Computing the de-biased estimator \eqref{eq:lasso-mod-cov:de-bias:theta} via SVRG}
\label{alg:stat-inf:high-dim:linear:calc:de-biased:svrg}
\end{algorithm}

Similar to \Cref{alg:stat-inf:high-dim:linear:proximal:svrg}, we choose $\eta = \Theta\left(\frac{1}{p}\right)$ and $L_i = \Theta(p)$.

\subsection{Solving the high dimensional linear regression optimization objective \eqref{eq:lasso-mod-cov} using proximal SVRG}
\label{subsec:high-dimensional-linear:lasso-:mod-cov:opt:svrg}

We solve our high dimensional linear regression optimization problem  using proximal SVRG  \cite{xiao2014proximal}
\begin{align}
\htheta = \arg \min_\theta \frac{1}{2} \theta^\top \left(\widehat{S}- \frac{1}{n} \sum_{i=1}^n x_i x_i^\top\right) \theta 
	+ \frac{1}{n} \sum_{i=1}^n \frac{1}{2} \left(x_i^\top \theta -y_i\right)^2 + \lambda \ltynorm{\theta}{1} .
\end{align}

\begin{algorithm}
\begin{algorithmic}[1]
\FOR{$i=0$ \TO $L_o-1$} 
	\STATE $u_i^0 \gets \theta_i $
	\STATE $d_t \gets \g_{\widehat{S}}(\theta_i)  - \frac{1}{n} \sum_{k=1}^n [\nabla f_k(\theta_i + \theta_i) - \nabla f_k(\theta_i)] + \frac{1}{n} \sum_{k=1}^n \nabla f_k(\theta_i)  $
	\FOR{$j=0$ \TO $L_i -1 $}
		\STATE $u_i^{j+1}$ $\gets$ $\Shrink_{\eta \lambda} ( u_i^j - \eta [d_t + \frac{1}{S} \sum_{k \in I} \left( u_i^j(k)  - \theta_i(k) \right)\cdot    \Shrink_\omega \left(\nabla f_k(\theta_t + \e_k) - \nabla f_k(\theta_t) \right) ] )$
	\ENDFOR
	\STATE $\theta_{t+1}\gets \frac{1}{L_i + 1} \sum_{j=0}^{L_i} u_i^j$
\ENDFOR
\end{algorithmic}
\caption{Solving the high dimensional linear regression optimization objective \eqref{eq:lasso-mod-cov} using proximal SVRG}
\label{alg:stat-inf:high-dim:linear:lasso-mod-cov:opt:svrg}
\end{algorithm}

Similar to \Cref{alg:stat-inf:high-dim:linear:proximal:svrg}, we choose $\eta = \Theta\left(\frac{1}{p}\right)$ and $L_i = \Theta(p)$.

\subsection{Non-asymptotic covariance estimate bound and asymptotic normality in \Cref{alg:stat-inf:high-dim:linear:proximal:svrg}}
\label{subsec:high-dimensional:linear:bound-covariance:asymptotic-normality}

We have a non-asymptotic covariance estimate bound and an asymptotic normality result. 
\begin{theorem}
\label{thm:alg:stat-inf:high-dim:linear:proximal:svrg}
Under our assumptions, when $n \gg \max \{b^2, \frac{1}{{D_\Sigma}^2}\} \log p$, $S_o = O(1)$, $S_i = O(1)$, 
and conditioned on $\{x_i\}_{i=1}^n$ and  following events which simultaneously with probability at least $1 - p^{-\Theta(1)} - n^{-\Theta(1)}$
\begin{itemize}
\item \mytag{[A]}{mytag:thm:alg:stat-inf:high-dim:linear:proximal:svrg:event-A}: $\max_{1\leq i \leq n} |\epsilon_i| \lesssim \sigma \sqrt{\log n} $,
\item \mytag{[B]}{mytag:thm:alg:stat-inf:high-dim:linear:proximal:svrg:event-B}: $\max_{1 \leq i \leq n} \| x_i \|_\infty \lesssim \sqrt{\log p + \log n}$, 
\item \mytag{[C]}{mytag:thm:alg:stat-inf:high-dim:linear:proximal:svrg:event-C}: $ \| {\widehat{S}}^{-1} \|_\infty \lesssim \frac{1}{D_\Sigma}$, 
\end{itemize}
we choose $L_i = \Theta(p)$,  $\tau=\Theta(\frac{1}{p})$, $\eta = \Theta(\frac{1}{p})$ in \Cref{alg:stat-inf:high-dim:linear:proximal:svrg}.

Here, we denote the objective function as 
\begin{align*}
P(\theta) = \frac{1}{2} \theta^\top \left(\widehat{S}- \frac{1}{n} \sum_{i=1}^n x_i x_i^\top\right) \theta 
	+ \frac{1}{n} \sum_{i=1}^n \frac{1}{2} \left(x_i^\top \theta -y_i\right)^2 + \lambda \ltynorm{\theta}{1} . 
\end{align*}

Then, we have a non-asymptotic covariance estimate bound 
\begin{align*} 
& \left\| \tfrac{S_o}{T} \tsum_{t=1}^T \bar{g}_t \bar{g}_t^\top - {\widehat{S}}^{-1} \left( \tfrac{1}{n} \tsum_{i=1}^n (x_i^\top \htheta - y_i)^2 x_i x_i^\top \right) {\widehat{S}}^{-1}  \right\|_{\max } \nonumber \\
\lesssim & \sqrt{  \left( (\log p + \log n) \|\htheta - \theta^\star\|_1 + \sigma \sqrt{(\log p + \log n) \log n} \right) \tfrac{\log p}{T} } \nonumber \\
	& + \tfrac{1}{u}\left[ \tfrac{1}{\sqrt{T}}\tsum_{t=1}^T 0.95^{L_o^t} (1+ \sqrt{P(\theta_0)-P(\htheta)} 0.95^{\sum_{i=0}^{t-1} L_o^t}) + \sqrt{p}(\log p + \log n)  \sqrt{P(\theta_0)-P(\htheta)} 0.95^{\sum_{i=0}^{t-1} L_o^t} \right], 
\end{align*}
where $\|A\|_{\max} = \max\{1 \leq j, k \leq p\} |A_{jk}|$ is the matrix max norm, with probability at least $1-p^{-\Theta(-1)} - u $. 

And we have asymptotic normality 
\begin{align*}
\tfrac{1}{\sqrt{t}}\left( \tsum_{t=1}^T \sqrt{S_o} \bar{g_t} + \tfrac{1}{n}\tsum_{i=1}^n x_i (x_i^\top \htheta - y_i)  \right) = W + R,  
\end{align*}
where $W$ weakly converges to $\scriptstyle \Norm\left(0, \widehat{S}^{-1} \left[ \frac{1}{n}\sum_{i=1}^n (x_i^\top \htheta - y_i)^2 x_i x_i^\top  - \left( \frac{1}{n} \sum_{i=1}^n x_i (x_i^\top \htheta - y_i) \right) \left( \frac{1}{n} \sum_{i=1}^n x_i (x_i^\top \htheta - y_i) \right)^\top \right] \widehat{S}^{-1}  \right)$, 
and $\Exp[\|R\|_\infty \mid \{x_i\}_{i=1}^n, \ref{mytag:thm:alg:stat-inf:high-dim:linear:proximal:svrg:event-A} , \ref{mytag:thm:alg:stat-inf:high-dim:linear:proximal:svrg:event-B}, \ref{mytag:thm:alg:stat-inf:high-dim:linear:proximal:svrg:event-C} ] \lesssim  \frac{1}{\sqrt{T}}\sum_{t=1}^T 0.95^{L_o^t} (1+ \sqrt{P(\theta_0)-P(\htheta)} 0.95^{\sum_{i=0}^{t-1} L_o^t}) + \sqrt{p}(\log p + \log n)  \sqrt{P(\theta_0)-P(\htheta)} 0.95^{\sum_{i=0}^{t-1} L_o^t} $. 

\end{theorem}

Note that when we choose $L_o^t = \Theta(\log (p) \cdot \log (t))$, 
and start from $\theta_0$ satisfying $P(\theta_0) - P(\htheta) \lesssim \frac{1}{p (\log p + \log n)^2}$ which can be effectively achieved using proximal SVRG (\Cref{subsec:high-dimensional-linear:lasso-:mod-cov:opt:svrg}), 
we can estimate the statistical error covariance with element-wise error less than $ O\left( \frac{\max\{1,\sigma\}\polylog(n,p)}{\sqrt{T}} \right)$ with high probability, 
using $O\left(T \cdot n \cdot p^2 \cdot \log(p) \cdot \log(T) \right)$ numerical operations.

\subsection{Plug-in statistical error covariance estimate}

\Cref{alg:stat-inf:high-dim:linear:proximal:svrg} is similar to using plug-in estimator $\frac{1}{n} \sum_{i=1}^n (x_i^\top \htheta - y_i)^2 x_i x_i^\top$ for $\sigma^2 \left( \frac{1}{n} \sum_{i=1}^n x_i x_i^\top \right) $ in \Cref{thm:lasso-mod:de-bias:stat-inf}, similar to the sandwich estimator \cite{huber1967behavior,white1980heteroskedasticity}. \Cref{lem:sec:high-dim:lasso:linear-regression:plug-in:sandwich} gives a bound on using this plug-in estimator in the statistical error covariance (\Cref{thm:lasso-mod:de-bias:stat-inf}) for coordinate-wise confidence intervals.

\begin{lemma}
\label{lem:sec:high-dim:lasso:linear-regression:plug-in:sandwich}
Under our assumptions, when $n \gg \max\{b^2,  \tfrac{1}{{D_\Sigma}^2}\} \log p$, we have
\begin{align*}
&  \left\|  \widehat{S}^{-1} \left( \tfrac{1}{n} \tsum_{i=1}^n (x_i^\top \htheta - y_i)^2 x_i x_i^\top \right) \widehat{S}^{-1} -  \sigma^2 \widehat{S}^{-1} \left( \tfrac{1}{n} \tsum_{i=1}^n x_i x_i^\top \right) \widehat{S}^{-1}  \right\|_{\max} \nonumber \\ 
\lesssim  &  \tfrac{1}{{D_\Sigma}^2} \left( \sigma \sqrt{ \log n}  + s \left( \sigma +\|\theta^{\star}\|_1 \right)  \sqrt{\log p + \log n} \sqrt{\tfrac{\log p}{n}} \right) s \left(\sigma + \|\theta^{\star}\|_1\right)  (\log p + \log n)^{\frac{3}{2}} \sqrt{\tfrac{\log p}{n}}   , 
\end{align*}
where $\|A\|_{\max} = \max_{1 \leq j, k \leq p} |A_{jk}|$ is the matrix max norm, with probability at least $1 - p^{-\Theta(1)} - n^{-\Theta(1)}$. 
\end{lemma}

\newpage

\section{Time series statistical inference with approximate Newton steps using only stochastic gradients (\Cref{sec:time-series})}
\label{appendix:sec:time-series}


Here, we give the complete approximate Newton-based time series statistical inference algorithm using only stochastic gradients.

\begin{algorithm}
\begin{algorithmic}[1]
%
%
\STATE \textbf{Parameters:} $\mathbf{l}, S_i \in \mathbb{Z}_+$; $\rho_0, \tau_0 \in \mathbb{R}_+$; $d_o, d_i \in \left(\tfrac{1}{2}, 1\right)$  \quad   \textbf{Initial state:} $\theta_0 \in \mathbb{R}^p$ \\
{\kern1.5pt \hrule \kern2pt} 
\FOR[approximate stochastic Newton descent]{$t = 0 \text{ to } T-1$}  \label{alg:stat-inf:unregularized:outer:start:time-series} 
	\STATE $\rho_t \gets \rho_0(t + 1)^{-d_o}$
	\STATE Uniformly select some $i_o \in [n]$, then set $I_o$ to the random contiguous block $\{i_o, i_o + 1, \dots, i_o + \mathbf{l} - 1 \} \mod n $, which circularly wraps around \label{alg:stat-inf-spnd:sgd1:time-series}
	\STATE $g^0_t \gets -\rho_t \left( \frac{1}{\mathbf{l}} \sum_{i  \in I_o} \nabla f_i(\theta_t) \label{alg:stat-inf-spnd:sgd2:time-series} \right) $ 
	\FOR[solving \eqref{eq:sec:unregularized-m-est:inner-loops:intuition:newton:obj} approximately using SGD]{$j = 0 \text{ to } L-1$} \label{alg:stat-inf-spnd:spnd-start:time-series}
		\STATE $\tau_j \gets \tau_0 (j+1)^{-d_i}$ and $\delta_t^j \gets O(\rho_t^4 \tau_j^4)$
		\STATE $I_i \gets$ uniformly sample $S_i$ indices without replacement from $[n]$
		\STATE $g^{j+1}_t \gets g_t^j -\tau_j  \left( \frac{1}{S_i}  \sum_{k \in I_i}  \tfrac{\nabla f_k(\theta_t +  {\delta_t^j} g_t^j)  - \nabla f_k(\theta_t)}{\delta_t^j} \right) + \tau_j g_t^0$  
	\ENDFOR 
	\STATE Use $\sqrt{\mathbf{l}} \cdot \frac{\bar{g}_t}{\rho_t}$ for statistical inference, where $\bar{g}_t = \frac{1}{L+1} \sum_{j=0}^L g^j_t$
	\STATE $\theta_{t+1} \gets \theta_t + g_t^L$  \label{alg:stat-inf-spnd:spnd-end:time-series} 
\ENDFOR \label{alg:stat-inf:unregularized:outer:end:time-series}
\end{algorithmic}
\caption{Unregularized M-estimation statistical inference}
\label{alg:stat-inf-spnd:time-series}
\end{algorithm}

\Cref{cor:time-series:spnd-bound} gives guarantees for \Cref{alg:stat-inf-spnd:time-series}, and is similar to the i.i.d. case  (\Cref{thm:spnd-bound}).

\begin{corollary}
\label{cor:time-series:spnd-bound}

Under the same assumptions as \Cref{thm:spnd-bound}, in  \Cref{alg:stat-inf-spnd:time-series}, 
for the outer iterate $\theta_t$ we have 
\begin{align}
\Exp \left[ \|\theta_t - \htheta\|_2^2 \right] &\lesssim t^{-d_o}, \label{eq:thm:spnd-bound:outer-iter:time-series} \\
\Exp \left[\|\theta_t - \htheta\|_2^4 \right] &\lesssim t^{-2d_o}.  \label{eq:thm:spnd-bound:outer-iter:4th-moment:time-series}  
\end{align}

In each outer loop, after $L$ steps of the inner loop, we have:
\begin{align}
	\Exp \left[ \left\| \tfrac{\bar{g}_t}{\rho_t} - [\nabla^2 f(\theta_t)]^{-1} g_t^0 \right\|_2^2 \mid \theta_t \right] \lesssim \tfrac{1}{L} \left\|g_t^0 \right\|_2^2, \label{eq:thm:spnd-bound:newton:time-series}  
\end{align}
and at each step of the inner loop, we have:
\begin{align}
	\Exp \left[ \left\| g_t^{j+1}- [\nabla^2 f(\theta_t)]^{-1} g_t^0 \right\|_2^4 \mid \theta_t \right] \lesssim  (j+1)^{-2d_i} \left\|g_t^0 \right\|_2^4. \label{eq:thm:spnd-bound:newton:4th-moment:time-series} 
\end{align}

After $T$ steps of the outer loop, we have a non-asymptotic bound on the ``covariance'':
\begin{align}
\label{eq:thm:spnd-bound:covariance-bound:time-series}
	  \Exp \left[ \left\| H^{-1} G H^{-1}  - \tfrac{S_o}{T} \sum_{t=1}^T \tfrac{\bar{g}_t \bar{g}_t^\top}{\rho_t^2} \right\|_2 \right] \lesssim T^{-\frac{d_o}{2}} + L^{-\frac{1}{2}}, 
\end{align}
where $H = \nabla^2 f(\htheta)$,  and 
\begin{align}
G = \tfrac{1}{n} \sum\limits_{i=1}^n \nabla f_i(\htheta) \, f_i(\htheta)^\top 
	+ \sum\limits_{j=1}^{\mathbf{l}} w(j, \mathbf{l}) \sum\limits_{i=j+1}^n  \left( 
		\nabla f_i(\htheta) \,  \nabla f_{i - j}(\htheta)^\top 
			 + \nabla f_{i-j}(\htheta) \,  \nabla f_i(\htheta)^\top
		 \right) , 
\end{align}
with $w(j, \mathbf{l}) = 1 -\tfrac{j}{\mathbf{l}}$.

Also, in \Cref{alg:stat-inf-spnd:time-series}'s outer loop,  the average of consecutive iterates satisfies
\begin{align}
\Exp\left[\left\|\tfrac{\sum_{t=1}^T \theta_t}{T} - \htheta\right\|_2^2\right] &\lesssim  \tfrac{1}{T} , \label{eq:thm:spnd-bound:outer-iter:avg_bound:time-series}  \\
\tfrac{1}{\sqrt{T}} \left( \tfrac{\sum_{t=1}^T \theta_t}{T} - \htheta \right) &= W + \Delta , \label{eq:thm:spnd-bound:outer-iter:avg_normality:time-series}   
\end{align}
where $W$ weakly converges to $\Norm(0, \frac{1}{S_o} H^{-1} G H^{-1})$, and $\Delta = o_P(1)$ when $T \to \infty$ and $L \to \infty$ ($\Exp[\|\Delta\|_2^2]$ $ \lesssim$ $ T^{1-2d_o} + T^{d_o-1} + \frac{1}{L} $).

\end{corollary}

Our approximate Newton time series statistical inference procedure estimates $H^{-1} G H^{-1}$, 
where $G$ is the Newey-West covariance estimator \eqref{eq:newey-west} with weight 
\begin{align}
w(j, \mathbf{l}) = 1 -\tfrac{j}{\mathbf{l}} ,  
\end{align}
which is because when we estimate the variance in \Cref{alg:stat-inf-spnd:time-series}, 
for $j>0$, terms $ \nabla f_{i} \, \nabla f_{i+j}^\top $ and $  \nabla f_{i+j} \, \nabla f_{i}^\top $ appear $\mathbf{l} - j$ times, 
and the term $\nabla f_i \, \nabla f_i^\top$ appears $\mathbf{l}$ times. 
Note that the connection between sampling scheme and Newey-West estimator was also observed in \cite{kunsch1989jackknife}.  
Thus, our stochastic approximate Newton statistical inference procedure for time series analysis has similar statistical properties compared circular bootstrap \cite{politis1992circular, politis1994stationary}.

Because expectation of the stochastic gradient in line \ref{alg:stat-inf-spnd:sgd2:time-series} of \Cref{alg:stat-inf-spnd:time-series} is the full gradient $ \frac{1}{n} \sum_{i=1}^n f_i(\htheta)$, we have the same optimization guarantees as the i.i.d. case (\Cref{cor:foasnd:asymptotic-normality:outer-avg}).

\newpage

\section{Statistical inference via  approximate stochastic Newton steps using first order information with increasing inner loop counts}
\label{sec:foasnd-stat-inf-increasing-segment}

Here, we present corollaries when the number of inner loops increases in the outer loops (i.e., $(L)_t $ is an increasing series). 
This guarantees convergence of the covariance estimate to $H^{-1} G H^{-1}$, although it is less efficient than using a constant number of inner loops. 

\subsection{Unregularized M-estimation}

Similar to \Cref{thm:spnd-bound}'s proof, we have the following result when the number of inner loop increases in the outer loops. 

\begin{corollary}
\label{cor:sec:foasnd-stat-inf-increasing-segment:unregularized}

In \Cref{alg:stat-inf-spnd}, if the number of inner loop in each outer loop $(L)_t$  increases in the outer loops, 
then we have 
\begin{align*}
	  \Exp\left[ \left\| H^{-1} G H^{-1}  - \frac{S_o}{T} \sum_{t=1}^T \frac{\bar{g}_t \bar{g}_t^\top}{\rho_t^2}  
		 \right\|_2 \right] \lesssim T^{-\frac{d_o}{2}} + \sqrt{\frac{1}{T}\sum_{i=1}^T \frac{1}{(L)_t}}.  
\end{align*}

For example, when we choose choose $(L)_t = L (t+1)^{d_L}$ for some $d_L > 0  $, then $\sqrt{\frac{1}{T}\sum_{i=1}^T \frac{1}{(L)_t}} = O(\frac{1}{\sqrt{L}} T^{-\frac{d_L}{2}})$. 

\end{corollary}

\newpage
\section{SVRG based statistical inference algorithm in unregularized M-estimation}
\label{sec:svrg-fosnd:inference:unregularized}

Here we present a SVRG based statistical inference algorithm in unregularized M-estimation, which has  asymptotic normality and improved bounds for the ``covariance''. 
Although \Cref{alg:stat-inf-svrg} has stronger guarantees than \Cref{alg:stat-inf-spnd},  \Cref{alg:stat-inf-svrg}  requires a full gradient evaluation in each outer loop. 

\begin{algorithm}
\begin{algorithmic}[1]
\FOR{$t \gets 0; t < T; ++t$}
	\STATE $d^0_t \gets -\eta \nabla f(\theta_t) = -\eta \left( \frac{1}{n} \sum_{i=1}^n \nabla f_i(\theta) \right) $ \COMMENT{point estimation via SVRG}
	\STATE $I_o \gets$ uniformly sample $S_o$ indices with replacement from $[n]$ 
	\STATE $g^0_t \gets -\rho_t \left( \frac{1}{S_o} \sum_{i  \in I_o} \nabla f_i(\theta_t)  \right) $ \COMMENT{statistical inference}
	\FOR[solving \eqref{eq:sec:unregularized-m-est:inner-loops:intuition:newton:obj} approximately using SGD]{$j \gets 0; j < L; ++ j$} 
		\STATE $I_i \gets$ uniformly sample $S_i$ indices without replacement from $[n]$
		\STATE $d^{j+1}_t \gets d_t^j - \eta \left( \frac{1}{S_i} \sum_{k \in I_i} (\nabla f_k(\theta_t +   d_t^j)  - \nabla f_k(\theta_t) \right) + d_t^0$  \COMMENT{point estimation via SVRG}
		\STATE $g^{j+1}_t \gets g_t^j -\tau_j  \left( \frac{1}{S_i}  \sum_{k \in I_i}    \frac{1}{{\delta_t^j}} [\nabla f_k(\theta_t +  {\delta_t^j} g_t^j)  - \nabla f_k(\theta_t) ] \right) + \tau_j g_t^0$  \COMMENT{statistical inference}
	\ENDFOR 
	\STATE Use $\sqrt{S_o} \cdot \frac{\bar{g}_t}{\rho_t}$ for statistical inference  \COMMENT{$\bar{g}_t = \frac{1}{L+1} \sum_{j=0}^L g^j_t$}
	\STATE $\theta_{t+1} \gets \theta_t + \bar{d}_t$ \COMMENT{$\bar{d}_t = \frac{1}{L+1} \sum_{j=0}^L d^j_t$}  
\ENDFOR
\end{algorithmic}
\caption{SVRG based statistical inference algorithm in unregularized M-estimation}
\label{alg:stat-inf-svrg}
\end{algorithm}

\begin{corollary}
\label{cor:svrg-foasnd:bounds}

In \Cref{alg:stat-inf-svrg}, when $L \geq 20 \frac{\max_{1\leq i \leq n}\beta_i}{\alpha}$ and  $\eta = \frac{1}{10 \max_{1\leq i \leq n}\beta_i}$, 
after $T$ steps of the outer loop, we have a non-asymptotic bound on the ``covariance'' 
\begin{align}
\label{eq:thm:svrg-bound:covariance-bound}
	  \Exp\left[ \left\| H^{-1} G H^{-1}  - \frac{S_o}{T} \sum_{t=1}^T \frac{\bar{g}_t \bar{g}_t^\top}{\rho_t^2}  
		 \right\|_2 \right] \lesssim L^{-\frac{1}{2}} ,  
\end{align}
and asymptotic normality
\begin{align*}
\frac{1}{\sqrt{T}}(\sum_{t=1}^T \frac{\bar{g}_t}{\rho_t})= W+\Delta,
\end{align*}
where $W$ weakly converges to $ \Norm(0, \frac{1}{S_o} H^{-1} G H^{-1})$ and $\Delta = o_P(1)$ when $T \to \infty$ and $L \to \infty$ ($\Exp[\|\Delta\|_2] \lesssim \frac{1}{\sqrt{T}}  + \frac{1}{L}$). 

\end{corollary}

When the number of inner loops increases in the outer loops (i.e., $(L)_t $ is an increasing series), we have a result similar to \Cref{cor:sec:foasnd-stat-inf-increasing-segment:unregularized}.

A better understanding of concentration, 
and Edgeworth expansion of the average consecutive iterates averaged  
(beyond \cite{dippon2008asymptotic,dippon2008edgeworth}) 
in stochastic approximation, would give stronger guarantees for our algorithms, and better compare and understand different algorithms.

\subsection{Lack of asymptotic normality in \Cref{alg:stat-inf-spnd} for mean estimation}
\label{subsec:foasnd:no-asymptotic-normality:example:mean-est}

In mean estimation,  we solve the following optimization problem 
\begin{align*}
\htheta = \arg \min_\theta \frac{1}{n} \sum_{i=1}^n  \frac{1}{2} \| \theta - X^{(i)}\|_2^2 , 
\end{align*}
where we assume that $\{X^{(i)}\}_{i=1}^n$ are constants. 

For ease of explanation we use $S_o = 1$, $\rho_t = \rho$, and $\theta_0 = 0$,
and we have 
\begin{align*}
\frac{\bar{g_t}}{\rho_t}  = -\theta_t + X_t , 
\end{align*}
where $X_t$ is uniformly sampled from $\{X^{(i)}\}_{i=1}^n$. 

And for $t\geq 1$ we have
\begin{align*}
\theta_t = \sum_{i=0}^{t-1} \rho (1-\rho)^{t-1-i} X_i . 
\end{align*} 

Then, we have
\begin{align*}
&\frac{1}{\sqrt{T}} (\sum_{i=1}^T \frac{\bar{g_t}}{\rho_t}) \nonumber \\
=& \frac{1}{\sqrt{T}} (\sum_{t=1}^T X_t - \sum_{t=1}^T\sum_{i=0}^{t-1} \rho (1-\rho)^{t-1-i} X_i) \nonumber \\
=& \frac{1}{\sqrt{T}}(\sum_{t=1}^T X_t - \sum_{i=0}^{T-1} (\sum_{t=i+1}^T\rho (1-\rho)^{t-1-i}) X_i) \nonumber \\
=& \frac{1}{\sqrt{T}}(\sum_{t=1}^T X_t - \sum_{i=0}^{T-1} (1-(1-\rho)^{T-i}) X_i) \nonumber \\
=&  \frac{1}{\sqrt{T}}( X_T - X_0 + \sum_{i=1}^{T-1} (1-\rho)^{T-i} X_i ) ,  
\end{align*}
whose $\ell_2$ norm's expectation converges to $0$ when $T \to \infty$, which implies that it converges to $0$ with probability 1. Thus, in this setting $\frac{1}{\sqrt{T}}\left(\sum_{t=1}^T \frac{\bar{g}_t}{\rho_t}\right)$ does not weakly converge to  $\Norm\left(0, \frac{1}{S_o} H^{-1} G H^{-1}\right)$.

\newpage

\section{An intuitive view of SVRG as approximate stochastic  Newton descent}
\label{append:sec:svrg-intuition}

Here we present an intuitive view of SVRG as approximate stochastic Newton descent, 
which is the inspiration behind our work. 

Gradient descent solves the optimization problem $\htheta = \arg \min_\theta f(\theta)$,
where the function is a sum of $n$ functions $f(\theta) = \frac{1}{n} \sum_{i=1}^n f_i(\theta)$, using 
\begin{align*}
\theta_{t+1} = \theta_t - \eta \nabla f(\theta_t), 
\end{align*}
and stochastic gradient descent uniformly samples a random index at each step
\begin{align*}
\theta_{t+1} = \theta_t - \eta_t \nabla f_i(\theta_t). 
\end{align*}

\begin{center}%
\scalebox{1.0}{\framebox{
\begin{minipage}{.5\textwidth}
\begin{itemize}
\item \textbf{Outer loop:}
\item $g \gets \nabla f(\theta_t) = \sum_{i=1}^n \nabla f_i(\theta_t)$
\item Let $d$ be the descent direction
\item 	\begin{itemize}
					\item \textbf{Inner loop:}
					\item Choose a random index $k$
					\item $d \gets d - \eta ( \nabla f_k(\theta_t + d) - \nabla f_k(\theta_t) + g)$
				\end{itemize}
\item $\theta_{t+1} = \theta_t + d$
\end{itemize}
\end{minipage}
}}
\end{center}

SVRG \cite{johnson2013accelerating} improves gradient descent and SGD by having an outer loop and an inner loop. 

Here, we give an intuitive explanation of SVRG as stochastic proximal Newton descent, 
by arguing that 
\begin{itemize}
\item each outer loop approximately computes the Newton direction $-(\nabla^2 f)^{-1} \nabla f$
\item the inner loops can be viewed as SGD steps solving a proximal Newton step $
\min_d \langle \nabla f, d \rangle + \frac{1}{2} d^\top (\nabla^2 f) d$
\end{itemize}

First, it is well known \cite{bubeck2015convex} that the Newton direction is exactly the solution of 
\begin{align}
\min_d \langle \nabla f(\theta), d \rangle + \frac{1}{2} d^\top [\nabla^2 f(\theta)] d . 
\label{eq:newton-dir}
\end{align}

Next, let's consider solving \eqref{eq:newton-dir} using gradient descent on a function of $d$, 
and notice that its gradient with respect to $d$ is 
\begin{align*}
\nabla f(\theta) + [\nabla^2 f(\theta)] d ,
\end{align*}
which can be approximated through $f$'s Taylor expansion ($[\nabla^2 f(\theta)] d \approx \nabla f(\theta + d) - \nabla f(\theta)$) as 
\begin{align*}
\nabla f(\theta) + [\nabla f(\theta + d) - \nabla f(\theta)] . 
\end{align*}

Thus, SVRG's inner loops can be viewed as using SGD to solve proximal Newton steps in outer loops. 
And it can be viewed as the power series identity for matrix inverse $H^{-1} = \sum_{i=0}^\infty (I - \eta H)$, which corresponds to unrolling the gradient descent recursion for the optimization problem $H^{-1} = \arg \min_\Omega \operatorname{Tr}\left( \frac{1}{2} \Omega^\top H \Omega - \Omega \right) $.

\newpage
\section{Proofs}

\subsection{Proof of \Cref{thm:spnd-bound}}
Given assumptions about strong convexity, Lipschitz gradient continuity and Hessian Lipschitz continuity in \Cref{thm:spnd-bound}, we denote:
\begin{align*}
\bar{\beta} &= \tfrac{\beta_i}{n}, \quad \bar{h} = \tfrac{h_i}{n}.
\end{align*}
Then, $\forall \theta_1, \theta_2$ we have:
\begin{align*}
\| \nabla f(\theta_2) - \nabla f(\theta_1) \|_2 &\leq \bar{\beta} \|\theta_2 - \theta_1\|_2, \quad \text{and} \quad 
\| \nabla^2 f(\theta_2) - \nabla^2 f(\theta_1) \|_2 \leq \bar{h} \|\theta_2 - \theta_1\|_2 .  
\end{align*}
and $\forall \theta$:
\begin{align*}
\| \nabla^2 f(\theta) \|_2 \leq \bar{\beta}. 
\end{align*}

In our proof, we also use the following:
\begin{align*}
\bar{h}_2 = \tfrac{1}{n} \sum_{i=1}^n h_i^2,~~
\bar{\beta}_2  = \tfrac{1}{n} \sum_{i=1}^n \beta_i^2, ~~\text{and}~~
\beta = \sup_\theta \|\nabla^2 f(\theta)\|_2.
\end{align*}
Observe that:
\begin{align*}
\bar{h} \leq \sqrt{\bar{h}_2}, ~~\text{and}~~
\alpha \leq \beta \leq \bar{\beta} \leq \sqrt{\bar{\beta}_2} . 
\end{align*}

\subsubsection{Proof of \eqref{eq:thm:spnd-bound:newton}}\label{subsec:proof:thm:spnd-bound:eq:thm:spnd-bound:newton}

We first prove \eqref{eq:thm:spnd-bound:newton}; the proof is similar to standard SGD convergence proofs (e.g. \cite{li2017statistical, chen2016statistical, polyak1992acceleration}). 
For the rest of our discussion, we assume that 
\begin{align*}
\delta_t^j \cdot \bar{h} \leq \delta_t^j \cdot \sqrt{\bar{h}_2} \ll 1, \quad \forall t, j.
\end{align*}

Using $\nabla f(\theta)$'s Taylor series expansion with a Lagrange remainder, 
we have the following lemma, which bounds the Hessian vector product  approximation error.
\begin{lemma}\label{append:proof:thm:spnd-bound:lem:hessian-taylor-approx}
$\forall, \theta, g, \delta \in \Real^p$, we have:
\begin{align*}
\left \| \tfrac{\nabla f_i(\theta + \delta g) - \nabla f_i(\theta)}{\delta} - \nabla^2 f_i(\theta) g \right\|_2 &\leq h_i \cdot | \delta | \cdot \| g \|_2, \\
\left \| \tfrac{\nabla f(\theta + \delta g) - \nabla f(\theta)}{\delta} - \nabla^2 f(\theta) g \right \|_2 &\leq \bar{h} \cdot | \delta | \cdot \| g \|_2.  
\end{align*}
\end{lemma}

Denote $H_t = \nabla^2 f( \theta_t) $ and  
$$
e_t^j = \left(\tfrac{1}{S_i} \cdot \sum_{k \in I_i}   \tfrac{\nabla f_k(\theta_t +  {\delta_t^j} g_t^j)  - \nabla f_k(\theta_t)}{\delta_t^j}\right) - \tfrac{\nabla f(\theta_t + {\delta_t^j} g_t^j) - \nabla f(\theta_t)}{\delta_t^j},
$$ 
then we have
\begin{align}
g_t^{j+1} - H_t^{-1} g_t^0 = g_t^j - H_t^{-1} g_t^0 - \tau_j \cdot \tfrac{\nabla f(\theta_t + {\delta_t^j} g_t^j) - \nabla f(\theta_t) }{\delta_t^j} +\tau_j g_t^0 - \tau_j e_t^j \label{eq:append:proof:thm:spnd-bound:inner:eq-expand} . 
\end{align}

Because $\Exp[e_j^t \mid g_j^t , \theta_t] = 0$, we have 
\begin{align}
\Exp \left[ \left\|g_t^{j+1} - H_t^{-1} g_t^0 \right\|_2^2 \mid \theta_t\right] &= \Exp\Bigg[ \left\|g_t^j - H_t^{-1} g_t^0 \right\|_2^2  - \tau_j \underbrace{  \left\langle g_t^j - H_t^{-1} g_t^0 , \tfrac{\nabla f(\theta_t + {\delta_t^j} g_t^j) - \nabla f(\theta_t) }{\delta_t^j} - g_t^0 \right \rangle}_{\mytag{[1]}{append:proof:thm:spnd-bound:lem:hessian-taylor-approx:tmp1}} \nonumber \\ 
	& \quad \quad + \tau_j^2 \underbrace{ \left\|\tfrac{\nabla f(\theta_t + {\delta_t^j} g_t^j) - \nabla f(\theta_t) }{\delta_t^j} - g_t^0 \right\|_2^2 }_{\mytag{[2]}{append:proof:thm:spnd-bound:lem:hessian-taylor-approx:tmp2}}
	+ \tau_j^2 \underbrace{ \left\| e_t^j \right\|_2^2 }_{\mytag{[3]}{append:proof:thm:spnd-bound:lem:hessian-taylor-approx:tmp3})}
	\mid \theta_t \Big] . 
\end{align}

For term \ref{append:proof:thm:spnd-bound:lem:hessian-taylor-approx:tmp1}, we have 
\begin{align}
\Big\langle g_t^j - H_t^{-1} g_t^0 , &\tfrac{\nabla f(\theta_t + {\delta_t^j} g_t^j) - \nabla f(\theta_t) }{\delta_t^j} - g_t^0 \Big \rangle \nonumber \\
						    &=  \left(g_t^j - H_t^{-1} g_t^0 \right)^\top H_t \left(g_t^j - H_t^{-1} g_t^0 \right) + \left\langle g_t^j - H_t^{-1} g_t^0 , \tfrac{\nabla f(\theta_t + {\delta_t^j} g_t^j) - \nabla f(\theta_t) }{\delta_t^j} - H_t \right \rangle  \nonumber \\ 
						    &\geq \left(g_t^j - H_t^{-1} g_t^0 \right)^\top H_t \left(g_t^j - H_t^{-1} g_t^0 \right) - 
	\left| \langle g_t^j - H_t^{-1} g_t^0 , \tfrac{\nabla f(\theta_t + {\delta_t^j} g_t^j) - \nabla f(\theta_t) }{\delta_t^j} - H_t \rangle \right| \nonumber \\
						    & \text{\blue{by Hessian approximation}} \nonumber \\ 
						    & \geq  \left(g_t^j - H_t^{-1} g_t^0 \right)^\top H_t \left(g_t^j - H_t^{-1} g_t^0 \right) - {\delta_t^j} \cdot \bar{h} \cdot \left\|g_t^j - H_t^{-1} g_t^0\right\|_2 \cdot \left\|g_t^j \right\|_2 \nonumber \\
						    &\text{\blue{by AM-GM inequality}} \nonumber \\
						    & \geq  \left(g_t^j - H_t^{-1} g_t^0 \right)^\top H_t \left(g_t^j - H_t^{-1} g_t^0 \right) - \tfrac{\delta_t^j \cdot \bar{h}}{2}  \cdot \left\| g_t^j - H_t^{-1} g_t^0 \right\|_2^2 -  \tfrac{\delta_t^j \cdot \bar{h}}{2} \cdot \left\|g_t^j \right\|_2^2 \nonumber \\
						    & =  \left(g_t^j - H_t^{-1} g_t^0 \right)^\top H_t \left(g_t^j - H_t^{-1} g_t^0 \right) - \tfrac{\delta_t^j \cdot \bar{h}}{2} \cdot \left\| g_t^j - H_t^{-1} g_t^0 \right\|_2^2 -  \tfrac{\delta_t^j \cdot \bar{h}}{2} \cdot \left\|g_t^j -  H_t^{-1} g_t^0 +  H_t^{-1} g_t^0 \right\|_2^2 \nonumber \\
						    &\text{\blue{$\|x + u\|_2^2 \leq 2\|x\|_2^2 + 2\|y\|_2^2$}} \nonumber \\
						    & \geq \left(g_t^j - H_t^{-1} g_t^0 \right)^\top H_t \left(g_t^j - H_t^{-1} g_t^0 \right) - \tfrac{3\delta_t^j \cdot \bar{h}}{2} \cdot \left\| g_t^j - H_t^{-1} g_t^0 \right\|_2^2   - {\delta_t^j} \bar{h} \cdot \left\|H_t^{-1} g_t^0 \right\|_2^2 \nonumber \\
						    &\text{\blue{by strong convexity}} \nonumber \\
						    & \geq \left(g_t^j - H_t^{-1} g_t^0 \right)^\top H_t \left(g_t^j - H_t^{-1} g_t^0 \right) - \tfrac{3\delta_t^j \cdot \bar{h}}{2} \cdot \left\| g_t^j - H_t^{-1} g_t^0 \right\|_2^2    - \tfrac{ {\delta_t^j} \bar{h} }{\alpha^2} \cdot \left\| g_t^0 \right\|_2^2 . 
\end{align}

For term \ref{append:proof:thm:spnd-bound:lem:hessian-taylor-approx:tmp2}, by repeatedly applying AM-GM inequality, using $f$'s smoothness and strong convexity,  and assuming $\delta_t^j \bar{h} \ll 1$, we have:
\begin{align*}
\left \|\tfrac{\nabla f(\theta_t + {\delta_t^j} g_t^j) - \nabla f(\theta_t) }{\delta_t^j} - g_t^0 \right\|_2^2 &= \left\|\tfrac{\nabla f(\theta_t + {\delta_t^j} g_t^j) - \nabla f(\theta_t) }{\delta_t^j} - H_t g_t^j + H_t g_t^j - g_t^0 \right\|_2^2 \nonumber \\
						&\leq \left \|\tfrac{\nabla f(\theta_t + {\delta_t^j} g_t^j) - \nabla f(\theta_t) }{\delta_t^j} - H_t g_t^j \right\|_2^2 \\
						&\quad \quad + 2 \left\|\tfrac{\nabla f(\theta_t + {\delta_t^j} g_t^j) - \nabla f(\theta_t) }{\delta_t^j} - H_t g_t^j \right\|_2 \cdot \left\| H_t g_t^j - g_t^0 \right\|_2 + \left\| H_t g_t^j - g_t^0 \right\|_2^2 \nonumber \\
						&\leq \left( {\delta_t^j} \bar{h} \right)^2 \|g_t^j\|_2^2 +  2 {\delta_t^j} \bar{h} \left\|g_t^j \right\|_2 \cdot \left\| H_t g_t^j - g_t^0 \right\|_2 + \left\| H_t g_t^j - g_t^0 \right\|_2^2 \nonumber \\
						&\leq \left( {\delta_t^j} \bar{h} + \left( {\delta_t^j} \bar{h}\right)^2 \right) \cdot \left\|g_t^j \right\|_2^2  + \left(1+ {\delta_t^j} \bar{h}\right) \cdot \| H_t g_t^j - g_t^0 \|_2^2 \nonumber \\
						& \leq 2 \left( {\delta_t^j} \bar{h} + \left( {\delta_t^j} \bar{h}\right)^2 \right) \cdot \left ( \left\|g_t^j - H_t^{-1} g_t^0 \right\|_2^2+ \left\|H_t^{-1} g_t^0 \right\|_2^2 \right) + \left(1+ {\delta_t^j} \bar{h} \right) \cdot \left\| H_t g_t^j - g_t^0 \right\|_2^2 \nonumber \\
						& \leq \tfrac{2 \left( {\delta_t^j} \bar{h} + \left( {\delta_t^j} \bar{h}\right)^2 \right)}{\alpha^2} \cdot \left\| g_t^0 \right\|_2^2 + \left(1+ 3 {\delta_t^j} \bar{h}+  2\left( {\delta_t^j} \bar{h}\right)^2 \right) \cdot \left\| H_t g_t^j - g_t^0 \right\|_2^2 \nonumber \\
						&\leq \tfrac{4{\delta_t^j} \bar{h} }{\alpha^2} \cdot \left\| g_t^0 \right\|_2^2 + \left(1+ 5 {\delta_t^j} \bar{h} \right) \cdot \| H_t g_t^j - g_t^0 \|_2^2. 
\end{align*}

For term \ref{append:proof:thm:spnd-bound:lem:hessian-taylor-approx:tmp3}, because we sample uniformly without replacement, we obtain:
\begin{align*}
\Exp_{I_i} \left[ \left\| e_t^j \right\|_2^2 \mid g_t^j, \theta_t \right] = \tfrac{1}{S_i} \left (1 - \tfrac{S_i - 1}{n - 1} \right) \cdot \Exp_k \left[ \left\| \tfrac{\nabla f_k(\theta_t +  {\delta_t^j} g_t^j)  - \nabla f_k(\theta_t) }{\delta_t^j} - 
	\tfrac{\nabla f(\theta_t + {\delta_t^j} g_t^j) - \nabla f(\theta_t) }{\delta_t^j} \right\|_2^2 \right] , 
\end{align*}
where $k$ is uniformly sampled from $[n]$. 
Denote $H_t^k = \nabla^2 f_k(\theta_t)$,  and by Lipschitz gradient we have $\|H_t^k\|_2 \leq \beta_k$. 
We can bound  the above 
\begin{align*}
\Bigg \| \tfrac{\nabla f_k(\theta_t +  \delta_t^j g_t^j) - \nabla f_k(\theta_t)}{\delta_t^j} &-  \tfrac{\nabla f(\theta_t + {\delta_t^j} g_t^j) - \nabla f(\theta_t) }{\delta_t^j} \Bigg \|_2^2 \nonumber \\
& = \left\| \tfrac{\nabla f_k(\theta_t +  {\delta_t^j} g_t^j)  - \nabla f_k(\theta_t)}{\delta_t^j} - H_t^k g_t^j + H_t^k g_t^j - 
	\tfrac{\nabla f(\theta_t + {\delta_t^j} g_t^j) - \nabla f(\theta_t) }{\delta_t^j} + H_t g_t^j - H_t g_t^j \right\|_2^2 \nonumber \\
& \leq 3 \left(  \left \| \left(H_t - H_t^k \right) g_t^j \right\|_2^2  + \left\| \tfrac{\nabla f_k(\theta_t +  {\delta_t^j} g_t^j)  - \nabla f_k(\theta_t)}{\delta_t^j} - H_t^k g_t^j \right\|_2^2 + \left\|\tfrac{\nabla f(\theta_t + {\delta_t^j} g_t^j) - \nabla f(\theta_t) }{\delta_t^j} - H_t g_t^j \right\|_2^2 \right) \nonumber \\
&\leq 3\left( \left\|H_t - H_t^k \right\|_2^2+ (\delta_t^j)^2 \left(\bar{h}^2 + h_k^2\right) \right) \cdot \left\|g_t^j\right\|_2^2 \nonumber \\
&\blue{\|H_t - H_t^k\|_2^2 \leq 2 (\bar{\beta}^2 + \beta_k^2)} \nonumber \\
&\leq 3 \left(2 \left(\bar{\beta}^2 + \beta_k^2\right) + (\delta_t^j)^2 (\bar{h}^2 + h_k^2) \right) \cdot \left\|g_t^j \right\|_2^2 \nonumber \\
&\leq 6 \left(2 \left(\bar{\beta}^2 + \beta_k^2\right) + (\delta_t^j)^2 (\bar{h}^2 + h_k^2) \right) \cdot \left( \left \|g_t^j -  H_t^{-1} g_t^0 \right\|_2^2 + \left\| H_t^{-1} g_t^0 \right\|_2^2 \right). 
\end{align*}

Taking the expectation over inner loop's random indices, for term \ref{append:proof:thm:spnd-bound:lem:hessian-taylor-approx:tmp3}, we have 
\begin{align}
\Exp_{I_i} \left[ \left\| e_t^j \right\|_2^2 \mid g_t^j, \theta_t \right] &\leq 6 \left(\tfrac{1}{S_i} \cdot \left(1 - \tfrac{S_i - 1}{n - 1}\right) \right) \left( \left({\delta_t^j} \bar{h} \right)^2  + 2 \bar{\beta}^2  + (\delta_t^j)^2 \bar{h}_2 + 2 \bar{\beta}_2\right) \cdot \left( \left\|g_t^j -  H_t^{-1} g_t^0 \right\|_2^2 + \tfrac{1}{\alpha^2} \cdot \left\| g_t^0 \right\|_2^2 \right) \nonumber \\
&\leq 18 \left(\tfrac{1}{S_i} \left(1 - \tfrac{S_i - 1}{n - 1} \right) \right) \cdot \left( (\delta_t^j)^2 \bar{h}_2 +  \bar{\beta}_2 \right) \cdot \left( \left\|g_t^j -  H_t^{-1} g_t^0 \right\|_2^2 + \tfrac{1}{\alpha^2} \left\| g_t^0 \right\|_2^2 \right) \label{eq:append:proof:thm:spnd-bound:lem:hessian-taylor-approx:tmp3:bound} . 
\end{align}

Combining all above, we have
\begin{align*}
\Exp \left[ \left\|g_t^{j+1} -  H_t^{-1} g_t^0 \right\|_2^2 \mid g_t^j, \theta_t \right] \nonumber &\leq \left\|g_t^{j} -  H_t^{-1} g_t^0 \right\|_2^2 \nonumber \\
	& - \tau_j \left(g_t^j - H_t^{-1} g_t^0 \right)^\top H_t \left(g_t^j - H_t^{-1} g_t^0 \right) + \tfrac{3 \tau_j {\delta_t^j} \bar{h}}{2}  \left\| g_t^j - H_t^{-1} g_t^0 \right\|_2^2    + \tfrac{ \tau_j {\delta_t^j} \bar{h} }{\alpha^2} \left\| g_t^0 \right\|_2^2 \nonumber \\
	& +  \tfrac{4\tau_j^2 {\delta_t^j} \bar{h} }{\alpha^2} \cdot \left\| g_t^0 \right\|_2^2 + \tau_j^2 \left (1+ 5 {\delta_t^j} \bar{h} \right) \cdot \left \| H_t g_t^j - g_t^0 \right\|_2^2 \nonumber \\
	& + 18 \tau_j^2 \left(\tfrac{1}{S_i} \left(1 - \tfrac{S_i - 1}{n - 1} \right) \right) \cdot \left(  (\delta_t^j)^2 \bar{h}_2 +  \bar{\beta}_2 \right) \cdot \left( \left\|g_t^j -  H_t^{-1} g_t^0 \right\|_2^2 + \tfrac{1}{\alpha^2} \| g_t^0 \|_2^2 \right)  . 
\end{align*}

When we choose the Hessian vector product approximation scaling constant  $\delta_t^j$ to be sufficiently small 
\begin{align*}
\delta_t^j \bar{h} \leq \delta_t^j \sqrt{\bar{h}_2} &\leq 0.01 , \nonumber \\ 
\tfrac{3{\delta_t^j} \bar{h}}{2} &\leq 0.01 \alpha , \nonumber \\
{\delta_t^j} \bar{h} \leq {\delta_t^j} \sqrt{\bar{h}_2} &\leq \tfrac{0.01 }{S_i} \left(1 - \tfrac{S_i - 1}{n - 1} \right) \bar{\beta}^2 \leq  \tfrac{0.01}{S_i} \left(1 - \tfrac{S_i - 1}{n - 1} \right)  \bar{\beta}_2 , \nonumber \\
 {\delta_t^j} \bar{h} \leq  {\delta_t^j} \sqrt{\bar{h}_2} &\leq \tfrac{0.01 \tau_j }{S_i} \left(1 - \tfrac{S_i - 1}{n - 1} \right) \bar{\beta}^2 \leq  \tfrac{0.01  \tau_j}{S_i}  \left(1 - \tfrac{S_i - 1}{n - 1} \right)  \bar{\beta}_2  , \nonumber \\
 {\delta_t^j} \bar{h} \leq  {\delta_t^j} \sqrt{\bar{h}_2} & \leq 0.01 \alpha \leq 0.01 \bar{\beta} \leq  0.01 \sqrt{\bar{\beta}_2} ,
\end{align*}
we have 
\begin{small}
\begin{align*}
\Exp \left[ \left\|g_t^{j+1} -  H_t^{-1} g_t^0 \right\|_2^2 \mid g_t^j, \theta_t \right] & \leq \left\|g_t^{j} -  H_t^{-1} g_t^0 \right\|_2^2 \underbrace{ - \tau_j \left(g_t^j - H_t^{-1} g_t^0 \right)^\top H_t \left(g_t^j - H_t^{-1} g_t^0 \right) + 1.05  \tau_j^2 \| H_t g_t^j - g_t^0 \|_2^2 }_{\mytag{[4]}{append:proof:thm:spnd-bound:lem:hessian-taylor-approx:tmp4}} \nonumber \\
	& \quad \quad + 18.5 \tau_j^2  \left(\tfrac{1}{S_i} \left(1 - \tfrac{S_i - 1}{n - 1} \right) \right) \bar{\beta}_2  \left\|g_t^j -  H_t^{-1} g_t^0 \right\|_2^2  \nonumber \\
	& \quad \quad \quad \quad + 18.5 \tau_j^2 \left(\tfrac{1}{S_i} \left(1 - \tfrac{S_i - 1}{n - 1} \right) \right) \tfrac{\bar{\beta}_2}{\alpha^2} \cdot \left\| g_t^0 \right\|_2^2  . 
\end{align*}
\end{small}

For term \ref{append:proof:thm:spnd-bound:lem:hessian-taylor-approx:tmp4}, 
let us consider the $\alpha$ strongly convex and $\beta$ smooth quadratic function
\begin{align*}
F(g) = \tfrac{1}{2} g^\top H_t g - \langle {g_t^0} , g  \rangle , 
\end{align*}
who attains its minimum at $g = H_t^{-1} g_t^0$. 
Using a well known property of $\alpha$ strongly convex and $\beta$ smooth functions (\Cref{lem:bubeck:strongly-convex-bound}), 
we have  
\begin{small}
\begin{align*}
- \left(g_t^j - H_t^{-1} g_t^0 \right)^\top H_t \left(g_t^j - H_t^{-1} g_t^0 \right) +  \tfrac{1}{2 \beta} \| H_t g_t^j - g_t^0 \|_2^2
\leq & -  \left(g_t^j - H_t^{-1} g_t^0 \right)^\top H_t \left(g_t^j - H_t^{-1} g_t^0 \right) +  \tfrac{1}{\alpha + \beta} \| H_t g_t^j - g_t^0 \|_2^2 \nonumber  \\
\leq & -\tfrac{\alpha \beta}{\alpha+ \beta}  \|g_t^j - H_t^{-1} g_t^0\|_2^2 \nonumber \\
\leq & -\tfrac{\alpha}{2} \|g_t^j - H_t^{-1} g_t^0\|_2^2 . 
\end{align*}
\end{small}

Thus, when we choose 
\begin{align*}
\tau_j \leq \tfrac{ 0.476}{\beta} , 
\end{align*}
we have 
\begin{small}
\begin{align*}
- \tau_j \left(g_t^j - H_t^{-1} g_t^0 \right)^\top H_t \left(g_t^j - H_t^{-1} g_t^0 \right) +  1.05 \tau_j^2 \cdot \left \| H_t g_t^j - g_t^0 \right\|_2^2 &\leq - \tau_j \left(g_t^j - H_t^{-1} g_t^0 \right)^\top H_t \left(g_t^j - H_t^{-1} g_t^0 \right) + \tfrac{ \tau_j }{2 \beta} \left\| H_t g_t^j - g_t^0 \right\|_2^2 , \nonumber \\
&\leq - \tfrac{\tau_j  \alpha }{2} \cdot \|g_t^j - H_t^{-1} g_t^0\|_2^2 , 
\end{align*}
\end{small}
and we have 
\begin{align*}
\Exp \left[ \left\|g_t^{j+1} -  H_t^{-1} g_t^0 \right\|_2^2 \mid g_t^j, \theta_t \right]  &\leq \left(1 - \tau_j \alpha + 18.5 \tau_j^2 \left(\tfrac{1}{S_i} \left(1 - \tfrac{S_i - 1}{n - 1} \right) \right)  \bar{\beta}_2 \right) \cdot \left \|g_t^{j} -  H_t^{-1} g_t^0 \right\|_2^2 \nonumber \\
	& \quad \quad +18.5   \tau_j^2 \left(\tfrac{1}{S_i} \left(1 - \tfrac{S_i - 1}{n - 1} \right) \right) \cdot \tfrac{\bar{\beta}_2}{\alpha^2} \cdot \| g_t^0 \|_2^2  . 
\end{align*}

Next, we set 
\begin{align}
\tau_0 &= \min \left\{ \tfrac{0.476}{\beta}, \tfrac{0.025 \cdot \alpha}{ \tfrac{1}{S_i} \left(1 - \tfrac{S_i - 1}{n - 1} \right) \bar{\beta}_2} \right\}, \quad D_j = (j+1)^{-d_i}, \quad \tau_j = \tau_0 D_j \label{eq:proof:thm:spnd-bound:tau_j} , 
\end{align}
where $d_i$ is inner loop's step size decay rate, and we have:
\begin{align*}
\Exp \left[ \left\|g_t^{j+1} -  H_t^{-1} g_t^0 \right\|_2^2 \mid \theta_t\right] &\leq \left(1 - \min \left\{\tfrac{\alpha}{2\beta}, \tfrac{0.013 \cdot \alpha^2}{\tfrac{1}{S_i} \left(1 - \tfrac{S_i - 1}{n - 1} \right)\bar{\beta}_2} \right\} D_j\right) \cdot\Exp \left[ \left\|g_t^{j} -  H_t^{-1} g_t^0 \right\|_2^2 \mid \theta_t \right] \nonumber \\
	& \quad \quad + 18.5  D_j^2  \tau_0^2 \left(\tfrac{1}{S_i} \left(1 - \tfrac{S_i - 1}{n - 1} \right)\right)  \tfrac{\bar{\beta}_2}{\alpha^2} \cdot \left\| g_t^0 \right\|_2^2 . 
\end{align*}

To satisfy the above requirements, for the Hessian vector product approximation scaling constant, we choose: 
\begin{align}
\delta_t^j  &= o\left(\min \left\{1,\tfrac{1}{\bar{h}} \right\} \cdot \min\left\{1, \alpha, \min \left\{1, \tau_0^4 \left(\frac{\tau_j}{\tau_0}\right)^4 \right\} \tfrac{1}{S_i} \left(1 - \tfrac{S_i - 1}{n - 1} \right) \right\} \right) \cdot \delta_t^0 = o\left ( \left(j+1 \right)^{-2}\right) \cdot \delta_t^0, \nonumber  \\ 
\delta_t^0 &= O(\rho_t^4) = o((t+1)^{-2}) =  o(1) . \label{eq:proof:thm:spnd-bound:delta_j}  
\end{align}
which is trivially satisfied for quadratic functions because all $h_i =0$. 

Note that:
\begin{align*}
18.5 \tau_0^2 \left(\tfrac{1}{S_i} \left(1 - \tfrac{S_i - 1}{n - 1} \right) \right) \cdot \tfrac{\bar{\beta}_2}{\alpha^2} = \Theta \left( \min \left\{ \left(\tfrac{1}{S_i} \left(1 - \tfrac{S_i - 1}{n - 1} \right) \right) \cdot  \tfrac{\bar{\beta}_2}{\beta^2 \alpha^2}, \tfrac{1}{\tfrac{1}{S_i} \left(1 - \tfrac{S_i - 1}{n - 1} \right) \cdot \bar{\beta}_2} \right\} \right) . 
\end{align*}

Applying \Cref{lem:geometric-like-series-bound}, we have:
\begin{align}
\Exp \left[ \left\|{g}_t^j -  H_t^{-1} g_t^0 \right\|_2^2 \mid \theta_t \right] =  O\left( t^{-d_i} \cdot \|g_t^0\|_2^2 \right) \label{eq:proof:thm:spnd-bound:j-th-iter:bound} , 
\end{align}
where we have assumed that $\alpha$, $\beta$, $S_i$, etc. are (data dependent) constants. 
Further, \eqref{eq:proof:thm:spnd-bound:j-th-iter:bound} implies:
\begin{align}
\Exp \left [\left\|g_t^j \right\|_2^2 \right] \leq 2 \Exp \left[ \left\|{g}_t^j -  H_t^{-1} g_t^0 \right\|_2^2 + \left\| H_t^{-1} g_t^0 \right\|_2^2 \mid \theta_t \right] \lesssim \| g_t^0 \|_2^2 \label{eq:proof:thm:spnd-bound:j-th-iter:norm:bound}, \quad \text{for all $j$.} 
\end{align}

In  \Cref{alg:stat-inf-spnd}, we have 
\begin{align*}
g_t^{j+1} -H_t^{-1} g_t^0  =  (I - \tau_j H_t )  (g_t^{j} -H_t^{-1} g_t^0) + \tau_j \left(-  e_t^j     -\tfrac{\nabla f(\theta_t + {\delta_t^j} g_t^j) - \nabla f(\theta_t) }{\delta_t^j} + H_t g_t^j \right).  
\end{align*} 
By unrolling the recursion we have:
\begin{align}
g_t^{j+1} -H_t^{-1} g_t^0 = \sum_{k=0}^j  \left(\prod_{l=k+1}^j (I - \tau_l H_t) \right) \cdot \tau_k \cdot \left( -  e_t^k     -\tfrac{\nabla f(\theta_t + {\delta_t^k} g_t^k) - \nabla f(\theta_t) }{\delta_t^k} + H_t g_t^k \right). \label{eq:proof:thm:spnd-bound:j-th-iter:recursion:unroll} 
\end{align} 

For the average $\bar{g}_t$, we have:
\begin{align}
\bar{g}_t - H_t^{-1} g_t^0 & =  \tfrac{1}{L+1} \sum_{j=0}^L ( g_t^j - H_t^{-1} g_t^0)   \nonumber \\
				        & =  \tfrac{1}{L+1}   \sum_{j=0}^L \sum_{k=0}^{j-1}  \left(\prod_{l=k+1}^{j-1} (I - \tau_l H_t) \right) \cdot \tau_k \left(-  e_t^k     -\tfrac{\nabla f(\theta_t + {\delta_t^k} g_t^k) - \nabla f(\theta_t) }{\delta_t^k} + H_t g_t^k\right) \nonumber  \\ 
& =  \tfrac{1}{L+1} \sum_{k=0}^{L-1}   \underbrace{   \tau_k \sum_{j=k+1}^L  \prod_{l=k+1}^{j-1} \left(I - \tau_l H_t \right)  }_{\mytag{[5]}{append:proof:thm:spnd-bound:lem:hessian-taylor-approx:tmp5}}  \left(-  e_t^k     -\tfrac{\nabla f(\theta_t + {\delta_t^k} g_t^k) - \nabla f(\theta_t) }{\delta_t^k} + H_t g_t^k\right) \nonumber \\
& = \underbrace{  \tfrac{1}{L+1} \sum_{k=0}^{L-1}      \tau_k \sum_{j=k+1}^L  \prod_{l=k+1}^{j-1} (I - \tau_l H_t)   \left(-  e_t^k \right) }_{\mytag{[6]}{append:proof:thm:spnd-bound:lem:hessian-taylor-approx:tmp6}}  \nonumber \\
	& \quad \quad + \underbrace{ \tfrac{1}{L+1} \sum_{k=0}^{L-1}   \tau_k \sum_{j=k+1}^L  \prod_{l=k+1}^{j-1} (I - \tau_l H_t)   \left(-\tfrac{\nabla f(\theta_t + {\delta_t^k} g_t^k) - \nabla f(\theta_t) }{\delta_t^k} + H_t g_t^k \right) }_{\mytag{[7]}{append:proof:thm:spnd-bound:lem:hessian-taylor-approx:tmp7}} \label{eq:append:proof:thm:spnd-bound:bar-g-diff:decompose}. 
\end{align}

For the term \ref{append:proof:thm:spnd-bound:lem:hessian-taylor-approx:tmp5}, we have:
\begin{align}
\left\|  \tau_k \sum_{j=k+1}^L  \prod_{l=k+1}^{j-1} (I - \tau_l H_t) \right\|_2 &\leq \tau_k  \sum_{j=k+1}^L  \prod_{l=k+1}^{j-1} \| I - \tau_l H_t \|_2 \nonumber \\
&  \blue{ I - \tau_l H_t \text{ is positive definite by our choice of  $\tau_l$ \eqref{eq:proof:thm:spnd-bound:tau_j}  and } \| I - \tau_l H_t \|_2 \leq 1 - \tau_l \alpha } \nonumber \\
&\leq \tau_k  \sum_{j=k+1}^L  \prod_{l=k+1}^{j-1} (  1 - \tau_l \alpha ) \nonumber \\ 
%
%
%
%
&\leq \tau_k  \sum_{j=k+1}^L  \prod_{l=k+1}^{j-1} \left(  1 - \frac{1}{2}\tau_l \alpha \right)^2 \nonumber \\
& \blue{ \tau_k  \prod_{l=k+1}^{j-1} (  1 - \frac{1}{2}\tau_l \alpha ) \leq \tau_k \exp( -  \frac{1}{2}  \alpha \sum_{l=k+1}^{j-1} \tau_l  ) 
	\lesssim k^{-d_i} \exp ( \Theta (- j^{1-d_i} + k^{1-d_i}) ) \lesssim j^{-d_i} \lesssim \tau_j } \nonumber \\ 
& \blue{\text{because for a fixed $d_i$ } x^{-d_i} e^{\Theta(x^{1-d_i})} \text{ is an increasing function  when $x$ is sufficiently large} }  \nonumber \\ 
&\lesssim \sum_{j=k+1}^L  \tau_j \prod_{l=k+1}^{j-1} \left(  1 - \tfrac{\tau_l \alpha}{2} \right) = \tfrac{2}{\alpha}\sum_{j=k+1}^L  \frac{1}{2}\tau_j\alpha \prod_{l=k+1}^{j-1} \left(  1 - \tfrac{\tau_l \alpha}{2} \right) \nonumber \\ 
&=  \tfrac{2}{\alpha} \left(1 - \prod_{j=k+1}^{L}\left(  1 - \tfrac{\tau_l \alpha}{2} \right) \right) =O(1) \label{eq:append:proof:thm:spnd-bound:lem:hessian-taylor-approx:tmp5:bound},  
\end{align} 
where we have assumed that $\alpha$, $\beta$, $S_i$, etc. are (data-dependent) constants. 

For the term \ref{append:proof:thm:spnd-bound:lem:hessian-taylor-approx:tmp6}, its norm  is bounded by:
\begin{align}
\Exp \left[  \left\|  \tfrac{1}{L+1} \sum_{k=0}^{L-1}      \tau_k \sum_{j=k+1}^L  \prod_{l=k+1}^{j-1} (I - \tau_l H_t)   (-  e_t^k ) \right\|_2^2  \mid \theta_t \right]  &= \tfrac{1}{(L+1)^2}  \Exp \left[  \sum_{k=0}^{L-1} \left\| \tau_k \sum_{j=k+1}^L  \prod_{l=k+1}^{j-1} (I - \tau_l H_t)   (-  e_t^k ) \right\|_2^2   \mid \theta_t \right]  \nonumber \\
& \blue{\text{ using \eqref{eq:append:proof:thm:spnd-bound:lem:hessian-taylor-approx:tmp5:bound}}} \nonumber \\
& \lesssim \tfrac{1}{(L+1)^2}  \Exp \left[ \sum_{k=0}^{L-1} \| e_t^k  \|_2^2  \mid \theta_t \right]   \nonumber \\ 
& \blue{\text{ using \eqref{eq:append:proof:thm:spnd-bound:lem:hessian-taylor-approx:tmp3:bound}  and \eqref{eq:proof:thm:spnd-bound:j-th-iter:bound}}} \nonumber \\
& \lesssim \tfrac{1}{L} \|g_t^0\|_2^2 . \label{eq:append:proof:thm:spnd-bound:lem:hessian-taylor-approx:tmp6:bound}
\end{align}
where the first equality is due to $a<b$, $\Exp[ {e_t^a}^\top e_t^b  \mid \theta_t ] = 0$, when we first condition on $b$.

For the term \ref{append:proof:thm:spnd-bound:lem:hessian-taylor-approx:tmp7}, its norm is bounded by:
\begin{align}
&\Exp \left[  \left\| \tfrac{1}{L+1} \sum_{k=0}^{L-1}   \tau_k \sum_{j=k+1}^L  \prod_{l=k+1}^{j-1} (I - \tau_l H_t)   \left(-\tfrac{\nabla f(\theta_t + {\delta_t^k} g_t^k) - \nabla f(\theta_t) }{\delta_t^k} + H_t g_t^k \right) \right\|_2^2   \mid \theta_t \right]    \nonumber \\
& \quad \quad \quad \quad =  \tfrac{1}{(L+1)^2} \Exp \Bigg[ \sum_{0 \leq a, b, \leq L- 1} 
	\Bigg\langle  \tau_a \sum_{j=a+1}^L  \prod_{l=a+1}^{j-1} (I - \tau_l H_t)   \left(-\tfrac{\nabla f(\theta_t + {\delta_t^a} g_t^a) - \nabla f(\theta_t) }{\delta_t^a} + H_t g_t^a \right) , \nonumber \\
& \quad \quad \quad \quad \quad \quad \quad \quad \quad \quad \quad \quad \tau_b \sum_{j=b+1}^L  \prod_{l=b+1}^{j-1} (I - \tau_l H_t)   \left(-\tfrac{\nabla f(\theta_t + {\delta_t^b} g_t^b) - \nabla f(\theta_t) }{\delta_t^b} + H_t g_t^b \right)
	\Bigg\rangle \mid \theta_t \Bigg]  \nonumber \\
& \quad \quad \quad \quad \leq \tfrac{1}{(L+1)^2} \Exp \Bigg[ \sum_{0 \leq a, b, \leq L- 1} 
	\left\|  \tau_a \sum_{j=a+1}^L  \prod_{l=a+1}^{j-1} (I - \tau_l H_t)   \left(-\tfrac{\nabla f(\theta_t + {\delta_t^a} g_t^a) - \nabla f(\theta_t) }{\delta_t^a} + H_t g_t^a \right) \right\|_2 \nonumber \\ 
& \quad \quad \quad \quad \quad \quad \quad \quad \quad \quad \quad \quad \cdot \left\| \tau_b \sum_{j=b+1}^L   \prod_{l=b+1}^{j-1} (I - \tau_l H_t)   \left(-\tfrac{\nabla f(\theta_t + {\delta_t^b} g_t^b) - \nabla f(\theta_t) }{\delta_t^b} + H_t g_t^b \right) \right\|_2 \mid \theta_t \Bigg] \nonumber \\ 
& \quad \quad \quad \quad  \blue{\text{using \eqref{eq:append:proof:thm:spnd-bound:lem:hessian-taylor-approx:tmp5:bound} and \Cref{append:proof:thm:spnd-bound:lem:hessian-taylor-approx} } } \nonumber \\ 
& \quad \quad \quad \quad \lesssim \tfrac{1}{(L+1)^2} \Exp\left[ \sum_{0 \leq a, b, \leq L- 1} \delta_t^a \bar{h} \|g_t^a\|_2 \delta_t^b \bar{h} \|g_t^b\|_2 \mid \theta_t  \right] \leq \tfrac{2\bar{h}^2}{(L+1)^2} \sum_{0 \leq a, b, \leq L- 1} \delta_t^a \delta_t^b \cdot  \Exp\left[ \|g_t^a\|_2^2 +  \|g_t^b\|_2^2  \mid \theta_t  \right] \nonumber \\
& \quad \quad \quad \quad  \lesssim \tfrac{\|g_t^0\|_2^2}{(L+1)^2} \sum_{0 \leq a, b, \leq L- 1} \delta_t^a \delta_t^b \lesssim \tfrac{\|g_t^0\|_2^2}{L^2} \left (\sum_{k=0}^L \delta_t^k \right)^2 \label{eq:append:proof:thm:spnd-bound:lem:hessian-taylor-approx:tmp7:bound:delta} \\ 
& \quad \quad \quad \quad  \blue{\text{using \eqref{eq:proof:thm:spnd-bound:j-th-iter:norm:bound} and our choice of $\delta_t^k$ \eqref{eq:proof:thm:spnd-bound:delta_j} } } \nonumber \\ 
& \quad \quad \quad \quad  \lesssim \tfrac{1}{L^2} {\delta_t^0}^2 \left(\sum_{k=0}^L \tau_k \right)^2 \cdot \|g_t^0\|_2^2 \lesssim \tfrac{1}{L^2} {\delta_t^0}^2 \left(\sum_{k=0}^L (k+1)^{-d_i}\right)^2 \cdot \|g_t^0\|_2^2 \nonumber \\
& \quad \quad \quad \quad  \blue{\text{because $ \left(\sum_{k=0}^L (k+1)^{-d_i} \right)^2 = O\left(L^{1-d_i} \right)$ and $d_i \in \left(\tfrac{1}{2}, 1\right)$} } \nonumber \\
& \quad \quad \quad \quad  \ll \tfrac{1}{L} \|g_t^0\|_2^2 .   \label{eq:append:proof:thm:spnd-bound:lem:hessian-taylor-approx:tmp7:bound}
\end{align}

Combining \eqref{eq:append:proof:thm:spnd-bound:lem:hessian-taylor-approx:tmp6:bound} and \eqref{eq:append:proof:thm:spnd-bound:lem:hessian-taylor-approx:tmp7:bound}, we have
\begin{align*}
\| \bar{g}_t - H_t^{-1} g_t^0 \|_2^2 = O\left(\tfrac{1}{L} \|g_t^0\|_2^2 \right). 
\end{align*}

\subsubsection{Proof of \eqref{eq:thm:spnd-bound:newton:4th-moment}}

Using \eqref{eq:append:proof:thm:spnd-bound:inner:eq-expand}, we have 
\begin{align}
& \Exp[ \| g_t^{j+1} - H_t^{-1} g_t^0 \|_2^4 \mid g_t^j] \nonumber \\  
& =  \Exp[ \|  g_t^j - H_t^{-1} g_t^0 
	- \tau_j \tfrac{\nabla f(\theta_t + {\delta_t^j} g_t^j) - \nabla f(\theta_t) }{\delta_t^j} +\tau_j g_t^0 - \tau_j e_t^j \|_2^4 \mid g_t^j] \nonumber \\ 
& = \Exp[ ( \| g_t^j - H_t^{-1} g_t^0 
	- \tau_j \tfrac{\nabla f(\theta_t + {\delta_t^j} g_t^j) - \nabla f(\theta_t) }{\delta_t^j} +\tau_j g_t^0  \|_2^2 \nonumber \\ 
	& -2 \langle \tau_j e_t^j , g_t^j - H_t^{-1} g_t^0   - \tau_j \tfrac{\nabla f(\theta_t + {\delta_t^j} g_t^j) - \nabla f(\theta_t) }{\delta_t^j} +\tau_j g_t^0 \rangle  + \tau_j^2 \|e_t^j\|_2^2 )^2 \mid g_t^j] \nonumber \\
& = \Exp[   \| g_t^j - H_t^{-1} g_t^0 
	- \tau_j \tfrac{\nabla f(\theta_t + {\delta_t^j} g_t^j) - \nabla f(\theta_t) }{\delta_t^j} +\tau_j g_t^0  \|_2^4 \nonumber \\
		& + 4( \langle \tau_j e_t^j , g_t^j - H_t^{-1} g_t^0   - \tau_j \tfrac{\nabla f(\theta_t + {\delta_t^j} g_t^j) - \nabla f(\theta_t) }{\delta_t^j} +\tau_j g_t^0 \rangle)^2 + \tau_j^4 \|e_t^j\|_2^4 \nonumber \\
		& + 2 \| g_t^j - H_t^{-1} g_t^0 
	- \tau_j \tfrac{\nabla f(\theta_t + {\delta_t^j} g_t^j) - \nabla f(\theta_t) }{\delta_t^j} +\tau_j g_t^0  \|_2^2 \tau_j^2 \|e_t^j\|_2^2 \nonumber \\
	& -4  \langle \tau_j e_t^j , g_t^j - H_t^{-1} g_t^0   - \tau_j \tfrac{\nabla f(\theta_t + {\delta_t^j} g_t^j) - \nabla f(\theta_t) }{\delta_t^j} +\tau_j g_t^0 \rangle \| g_t^j - H_t^{-1} g_t^0 
	- \tau_j \tfrac{\nabla f(\theta_t + {\delta_t^j} g_t^j) - \nabla f(\theta_t) }{\delta_t^j} +\tau_j g_t^0  \|_2^2 \nonumber \\ 
	& - 4 \langle \tau_j e_t^j , g_t^j - H_t^{-1} g_t^0   - \tau_j \tfrac{\nabla f(\theta_t + {\delta_t^j} g_t^j) - \nabla f(\theta_t) }{\delta_t^j} +\tau_j g_t^0 \rangle  \tau_j^2 \|e_t^j\|_2^2
	\mid g_t^j  ]  \label{eq:thm:spnd-bound:newton:4th-moment:append:proof:expand-tmp1}  .   
\end{align}

Because we have  
\begin{align*}
\Exp[e_t^j \mid g_t^j  ]  =  0,
\end{align*}
\begin{align*}
& \| g_t^j - H_t^{-1} g_t^0 
	- \tau_j \tfrac{\nabla f(\theta_t + {\delta_t^j} g_t^j) - \nabla f(\theta_t) }{\delta_t^j} +\tau_j g_t^0  \|_2^4 \nonumber \\
& = \| (I- \tau_j H_t)( g_t^j - H_t^{-1} g_t^0 )
	+ \tau_j (-\tfrac{\nabla f(\theta_t + {\delta_t^j} g_t^j) - \nabla f(\theta_t) }{\delta_t^j} + H_t g_t^j)  \|_2^4 \nonumber \\
& = (\|(I- \tau_j H_t)( g_t^j - H_t^{-1} g_t^0 )\|_2^2 \nonumber \\
	& + 2 \tau_j \langle (I- \tau_j H_t)( g_t^j - H_t^{-1} g_t^0 ),  -\tfrac{\nabla f(\theta_t + {\delta_t^j} g_t^j) - \nabla f(\theta_t) }{\delta_t^j} + H_t g_t^j \rangle \nonumber \\ 
	& + \tau_j^2 \|-\tfrac{\nabla f(\theta_t + {\delta_t^j} g_t^j) - \nabla f(\theta_t) }{\delta_t^j} + H_t g_t^j\|_2^2 )^2 \nonumber \\ 
& \blue{\text{using \Cref{append:proof:thm:spnd-bound:lem:hessian-taylor-approx}}} \nonumber \\
\leq & (\|(I- \tau_j H_t)( g_t^j - H_t^{-1} g_t^0 )\|_2^2  + 2 \tau_j \|I- \tau_j H_t\|_2 \| g_t^j - H_t^{-1} g_t^0 \|_2  \delta_t^j \| g_t^j\|_2   + \tau_j^2 {\delta_t^j}^2 \| g_t^j\|_2 ^2 )^2 \nonumber \\ 
& =\|(I- \tau_j H_t)( g_t^j - H_t^{-1} g_t^0 )\|_2^4  \nonumber \\ 
	& + 2 \tau_j \|(I- \tau_j H_t)( g_t^j - H_t^{-1} g_t^0 )\|_2^2 (2 \delta_t^j   \|I- \tau_j H_t\|_2 \| g_t^j - H_t^{-1} g_t^0 \|_2  \| g_t^j\|_2 + \tau_j {\delta_t^j}^2 \| g_t^j\|_2 ^2) \nonumber \\ 
	&	+ \tau_j^2(2 \delta_t^j   \|I- \tau_j H_t\|_2 \| g_t^j - H_t^{-1} g_t^0 \|_2  \| g_t^j\|_2 + \tau_j {\delta_t^j}^2 \| g_t^j\|_2 ^2)^2 \nonumber \\
& \blue{\text{by our choice of $\tau_j =\Theta((j+1)^{-d_i})=o(1)$ \eqref{eq:proof:thm:spnd-bound:tau_j} }} \nonumber \\ %
	& \blue{\text{and using $\|g_t^j\|_2 \leq \|g_t^j - H_t^{-1} g_t^0 \|_2 +\|H_t^{-1} g_t^0\|_2 \lesssim \|g_t^j - H_t^{-1} g_t^0 \|_2 +^4 \| g_t^0\|_2$ }} \nonumber \\
& = (1- \Theta(\tau_j))\| g_t^j - H_t^{-1} g_t^0 \|_2^4  \nonumber \\
	& + O(\tau_j \delta_t^j ( \|g_t^j - H_t^{-1} g_t^0\|_2^4+ \|g_t^j - H_t^{-1} g_t^0\|_2^3 \| g_t^0 \|_2) + 2\tau_j^2{\delta_t^j}^3(\|g_t^j - H_t^{-1} g_t^0\|_2^4 +  \|g_t^j - H_t^{-1} g_t^0\|_2^2 \| g_t^0 \|_2^2) \nonumber \\
	& \quad + \tau_j^2{\delta_t^j}^2 (\| g_t^j - H_t^{-1} g_t^0 \|_2^4+\| g_t^j - H_t^{-1} g_t^0 \|_2 ^2 \| g_t^0\|_2^2 + \tau_j \delta_t^j (\| g_t^j - H_t^{-1} g_t^0 \|_2^4+ \| g_t^0\|_2^4) ) ) , 
\end{align*}
\begin{align*}
&  \Exp[  \|e_t^j\|_2^4 \mid g_t^j  ] \nonumber\\ 
 =& \Exp[  \| \left( \frac{1}{S_i} \frac{1}{{\delta_t^j}} \sum_{k \in I_i}    (\nabla f_k(\theta_t +  {\delta_t^j} g_t^j)  - \nabla f_k(\theta_t) )\right) - 
	\tfrac{\nabla f(\theta_t + {\delta_t^j} g_t^j) - \nabla f(\theta_t) }{\delta_t^j} \|_2^4 \mid g_t^j  ]\nonumber\\ 
 =& \Exp[  \| \left( \frac{1}{S_i} \frac{1}{{\delta_t^j}} \sum_{k \in I_i}    ((\nabla f_k(\theta_t +  {\delta_t^j} g_t^j)  - \nabla f_k(\theta_t)) - H_t^k g_t^j + H_t^k g_t^j )\right) \nonumber \\
	&- 
	( \frac{1}{{\delta_t^j}} (\nabla f(\theta_t + {\delta_t^j} g_t^j) - \nabla f(\theta_t)) - H_t g_t^j + H_t g_t^j ) \|_2^4 \mid g_t^j  ]\nonumber\\
& \blue{\text{using \Cref{append:proof:thm:spnd-bound:lem:hessian-taylor-approx} and repeatedly applying the AM-GM inequality}} \nonumber \\
\lesssim & (1+{\delta_t^j}^4) \|g_t^j\|_2^4 \nonumber \\ 
\lesssim & (1+{\delta_t^j}^4) {\delta_t^j}^4  (\| g_t^j - H_t^{-1} g_t^0 \|_2^4+ \| g_t^0\|_2^4) , 
\end{align*}
and by our choice of $\tau_j =\Theta((j+1)^{-d_i})=o(1)$ \eqref{eq:proof:thm:spnd-bound:tau_j} and $\delta_t^j =O(\tau_j^4)$ \eqref{eq:proof:thm:spnd-bound:delta_j}, 
after repeatedly applying the AM-GM inequality, \Cref{append:proof:thm:spnd-bound:lem:hessian-taylor-approx}, triangle inequality, and \eqref{eq:proof:thm:spnd-bound:j-th-iter:bound}, 
we can bound \eqref{eq:thm:spnd-bound:newton:4th-moment:append:proof:expand-tmp1}  by 
\begin{align}
& \Exp[ \| g_t^{j+1} - H_t^{-1} g_t^0 \|_2^4 \mid g_t^j] \nonumber \\  
\leq &  (1- \Theta(\tau_j))\| g_t^j - H_t^{-1} g_t^0 \|_2^4 + O(\tau_j^3 \|g_t^0 \|_2^4)
	) .
\end{align}
Applying \Cref{lem:geometric-like-series-bound}, we have 
\begin{align}
\Exp[ \| g_t^{j+1} - H_t^{-1} g_t^0 \|_2^4 \mid \theta_t] = O((j+1)^{-2d_i}\|g_t^0 \|_2^4) ,  \label{eq:proof:thm:spnd-bound:j-th-iter:bound:4th-moment}
\end{align}
and using the AM-GM in equality we have 
\begin{align}
\Exp[ \| g_t^{j+1}  \|_2^4 \mid \theta_t] = O(\|g_t^0 \|_2^4) .  \label{eq:proof:thm:spnd-bound:j-th-iter:bound:norm:4th-moment}
\end{align}


\subsubsection{Proof of \eqref{eq:thm:spnd-bound:outer-iter}}

To prove bounds on $\|\theta_t - \htheta\|_2^2$, we will use the following lemma
\begin{lemma}
\label{append:proof:thm:spnd-bound:lem:g-t-L_dot_grad-lower-bound}
\begin{align*}
\Exp[ \langle \nabla  f(\theta_t) , -g_t^L  \rangle  \mid \theta_t ] \gtrsim &
	 \rho_t   \|\nabla f(\theta_t)\|_2^2 - \delta_t^0 \|\nabla f(\theta_t)\|_2 \|g_t^0\|_2 \nonumber \\
	 \gtrsim & \rho_t   \|\nabla f(\theta_t)\|_2^2  -  {\delta_t^0}^2\|g_t^0\|_2^2 . 
\end{align*}
\end{lemma}
\begin{proof}

Using \eqref{eq:proof:thm:spnd-bound:j-th-iter:recursion:unroll}, and because $\Exp[e_t^j \mid \theta_t = 0]$, we have 
\begin{align*}
&\Exp[ \langle \nabla  f(\theta_t) , -g_t^L  \rangle  \mid \theta_t ] \nonumber \\
& = \rho_t \nabla f(\theta_t)^\top H_t^{-1}  \nabla f(\theta_t)
	- \Exp \left[ \left\langle \left. \nabla  f(\theta_t) , \sum_{k=0}^{L-1}  (\prod_{l=k+1}^{L-1} (I - \tau_l H_t) ) \tau_k (   \tfrac{\nabla f(\theta_t + {\delta_t^k} g_t^k) - \nabla f(\theta_t) }{\delta_t^k} - H_t g_t^k ) \right\rangle \,\right|\, \theta_t \right]  \nonumber \\
& \blue{\text{using strong convexity and \Cref{append:proof:thm:spnd-bound:lem:hessian-taylor-approx}}} \nonumber \\ 
& \geq \frac{1}{\beta} \rho_t   \|\nabla f(\theta_t)\|_2^2 
	- \|\nabla f(\theta_t)\|_2 \underbrace{ \Exp \left[  \sum_{k=0}^{L-1}  \left. \prod_{l=k+1}^{L-1} \|I - \tau_l H_t \|_2 \tau_k \delta_t^k \| g_t^k \|_2  \, \right| \, \theta_t \right]  }_{\mytag{[8]}{append:proof:thm:spnd-bound:lem:hessian-taylor-approx:tmp8}} .
\end{align*}

By our choice of  $\tau_j =\Theta((j+1)^{-d_i})=o(1)$ \eqref{eq:proof:thm:spnd-bound:tau_j} and $\delta_t^j =O(\delta_t^0 \tau_j^4)$ \eqref{eq:proof:thm:spnd-bound:delta_j}, 
and using \eqref{eq:proof:thm:spnd-bound:j-th-iter:norm:bound}, 
term \ref{append:proof:thm:spnd-bound:lem:hessian-taylor-approx:tmp8} is bounded by 
\begin{align*}
& \Exp \left[  \sum_{k=0}^{L-1}  \prod_{l=k+1}^{L-1} \|I - \tau_l H_t \|_2 \tau_k \delta_t^k \| g_t^k \|_2  \mid \theta_t \right] \nonumber \\
\lesssim &   \sum_{k=0}^{L-1} \tau_k \delta_t^k  \nonumber \\
\lesssim & \|g_t^0\|_2 \delta_t^0 \underbrace{\sum_{k=0}^{L-1} \tau_k^5}_{\blue{=O(1)}}. 
\end{align*}

And we can conclude 
\begin{align*}
& \Exp[ \langle \nabla  f(\theta_t) , -g_t^L  \rangle  \mid \theta_t ] \nonumber \\
& \geq
	 C_1 \rho_t   \|\nabla f(\theta_t)\|_2^2 - C_2 \delta_t^0 \|\nabla f(\theta_t)\|_2 \|g_t^0\|_2 \nonumber \\
& = C_1 \rho_t   \|\nabla f(\theta_t)\|_2^2 
	 - \frac{C_1}{2} {\delta_t^0}^2 \left[ 2 \frac{\|\nabla f(\theta_t)\|_2 }{\delta_t^0}  \frac{ C_2}{C_1} \|g_t^0\|_2 \right] \nonumber \\
& \geq  C_1 \rho_t   \|\nabla f(\theta_t)\|_2^2  - \frac{C_1}{2} {\delta_t^0}^2 \left(\left(\frac{\|\nabla f(\theta_t)\|_2 }{\delta_t^0})^2 + (\frac{ C_2}{C_1} \|g_t^0\|_2 \right)^2\right) \nonumber \\ 
& = \frac{C_1}{2} \rho_t   \|\nabla f(\theta_t)\|_2^2  - \frac{ C_2^2}{2C_1} {\delta_t^0}^2 \|g_t^0\|_2^2   , 
\end{align*} 
for some (data dependent) positive constants $C_1, C_2$.

\end{proof}

Now, we continue our proof of \eqref{eq:thm:spnd-bound:outer-iter}. 

In \Cref{alg:stat-inf-spnd}, because $f$ is $\beta$ smooth, we have
\begin{align}
&\Exp[ f(\theta_{t+1}) - f(\htheta) \mid \theta_t] \nonumber \\
 & = \Exp[  f(\theta_t + g_t^L) - f(\htheta) \mid \theta_t] \nonumber \\
\leq & f(\theta_t) - f(\htheta) + \Exp\left[ \left\langle \nabla f(\theta_t) , g_t^L \right\rangle + \frac{\beta}{2} \| g_t^L\|_2^2 \mid \theta_t\right] \nonumber \\ 
& \blue{\text{using \Cref{append:proof:thm:spnd-bound:lem:g-t-L_dot_grad-lower-bound} and \eqref{eq:proof:thm:spnd-bound:j-th-iter:norm:bound}}} \nonumber \\
\leq & f(\theta_t) - f(\htheta) - \Omega( \rho_t \|\nabla f(\theta_t)\|_2^2 ) + \Exp[ O( \|g_t^0\|_2^2  + \delta_t^0 \|g_t^0\|_2 \|\nabla f(\theta_t)\|_2 ) \mid \theta_t] \label{eq:proof:thm:spnd-bound:outer-f-val-recursion:tmp1} . 
\end{align}

For $g_t^0$, we have 
\begin{align}
\frac{g_t^0}{\rho_t} & = \frac{1}{S_o} \sum_{i \in I_o} \nabla f_i(\theta_t) \nonumber \\
& =  \frac{1}{S_o} \sum_{i \in I_o} \nabla f_i(\htheta) + \frac{1}{S_o}  \sum_{i \in I_o} (\nabla f_i(\theta_t) - \nabla f_i(\htheta)) \label{eq:proof:thm:spnd-bound:sgd:value-decompose} ,  
\end{align}
which implies that 
\begin{align}
& \Exp\left[\left\| \frac{g_t^0}{\rho_t} \right\|_2^2 \mid \theta_t\right] \nonumber \\
\leq & 2 \Exp\left[ \left\| \frac{1}{S_o} \sum_{i \in I_o} \nabla f_i(\htheta) \|_2^2 \mid \theta_t  ] + 2 \Exp[ \| \frac{1}{S_o}  \sum_{i \in I_o} (\nabla f_i(\theta_t) - \nabla f_i(\htheta)) \right\|_2^2  \mid \theta_t  \right] \nonumber \\
& \blue{\text{because we sample uniformly with replacement and } \nabla f(\htheta) = 0} \nonumber \\ 
\leq &  \frac{2}{S_o}  \sum_{i=1}^n  \|\nabla f_i(\htheta)\|_2^2 + \Exp[ \|\nabla f_i(\theta_t) - \nabla f_i(\htheta)\|_2^2 \mid \theta_t  ] \nonumber \\
\leq & \frac{2}{S_o}  \sum_{i=1}^n  \|\nabla f_i(\htheta)\|_2^2 + \| \theta_t  -   \htheta \|_2^2 \Exp[  \beta_i^2  \mid \theta_t  ] \nonumber \\
\lesssim & 1 + \| \theta_t  -   \htheta \|_2^2 . \label{eq:proof:thm:spnd-bound:sgd:norm-decompose}  
\end{align}

Thus, continuing \eqref{eq:proof:thm:spnd-bound:outer-f-val-recursion:tmp1}, 
using \eqref{eq:proof:thm:spnd-bound:sgd:norm-decompose} and  strong convexity $\alpha^2 \|\theta_t -\htheta\|_2^2 \leq  \|\nabla f(\theta_t)\|_2^2$, 
we have 
\begin{align}
&\Exp[ f(\theta_{t+1}) - f(\htheta) \mid \theta_t] \nonumber \\
\leq & f(\theta_t) - f(\htheta) - C_1 \rho_t \|\nabla f(\theta_t)\|_2^2  
	 +  C_2 \rho_t\delta_t^0 (1+ \|\nabla f(\theta_t)\|_2) \|\nabla f(\theta_t)\|_2 
	 + C_3 \rho_t^2 (1+ \| \nabla f(\theta_t)\|_2^2 ) \nonumber \\
& = f(\theta_t) - f(\htheta) 
	 - \rho_t (C_1 - C_2  \delta_t^0 - C_3 \rho_t) \|\nabla f(\theta_t)\|_2^2 
	 + C_3 \rho_t^2   
	 +  C_2 \rho_t \delta_t^0  \|\nabla f(\theta_t)\|_2 \nonumber \\ 
  & \blue{ \text{because we have }} 
  	\blue{C_2 \rho_t \delta_t^0  \|\nabla f(\theta_t)\|_2 }  
  		\blue{ = \frac{1}{2} C_1 \rho_t {\delta_t^0}^2 2 \frac{\frac{C_2}{C_1} \|\nabla f(\theta_t)\|_2}{\delta_t^0}}  
  			\blue{ \leq \frac{1}{2} C_1 \rho_t {\delta_t^0}^2 \left(\left(\frac{C_2}{C_1})^2+ (\frac{ \|\nabla f(\theta_t)\|_2}{\delta_t^0}\right)^2 \right) } \nonumber \\ 
\leq & f(\theta_t) - f(\htheta)
	 - \rho_t (\tfrac{1}{2} C_1 - C_2 \delta_t^0 - C_3 \rho_t) \|\nabla f(\theta_t)\|_2^2  
	+ C_3  \rho_t^2    
	+ \tfrac{C_2^2 }{C_1} \rho_t {\delta_t^0}^2 \nonumber \\
& \blue{\text{using strong convexity } \frac{1}{2\alpha} \|\nabla f(\theta_t)\|_2^2 \geq f(\theta_t) - f(\htheta) \text{ and smoothness } \frac{1}{2\beta} \|\nabla f(\theta_t)\|_2^2 \leq f(\theta_t) - f(\htheta) } \nonumber \\
\leq & [f(\theta_t) - f(\htheta)] 
	  - \rho_t (\tfrac{1}{2} C_1 - C_2 \delta_t^0 - C_3 \rho_t) \tfrac{1}{2 \alpha} [f(\theta_t) - f(\htheta)] 
	 + C_3  \rho_t^2     
	 + \tfrac{C_2^2 }{C_1} \rho_t {\delta_t^0}^2  \nonumber \\ 
& \blue{ \text{when we set } \delta_t^0 = O(\rho_t) \text{ in \eqref{eq:proof:thm:spnd-bound:delta_j}} }  \nonumber \\
\leq &
	[f(\theta_t) - f(\htheta)] 
	 - \rho_t (\tfrac{1}{2} C_1 - C_2 \delta_t^0 - C_3 \rho_t) \tfrac{1}{2 \alpha} [f(\theta_t) - f(\htheta)] 
	+ (C_3 +O(1))  \rho_t^2 
 \label{eq:spnd-bound:sgd:outer:recursion:f} , 
\end{align}
for some (data dependent) positive constants $C_1, C_2, C_3$.

In \eqref{eq:spnd-bound:sgd:outer:recursion:f} we choose $\rho_t = \Theta((t+1)^{-d_o})$ for some $d_o \in (\frac{1}{2}, 1)$, and after applying \Cref{lem:geometric-like-series-bound} we have 
\begin{align}
& \Exp[ \|\theta_t - \htheta\|_2^2 ] \nonumber \\
\leq & \Exp[ \tfrac{2}{\alpha} (f(\theta_t) - f(\htheta))] \nonumber \\
\lesssim & t^{-d_o} + e^{-\Theta(t^{1-d_o})} \|\theta_0 - \htheta\|_2^2 , \label{eq:thm:spnd-bound:outer-iter:2nd:f-bound}
\end{align}
which is $O(t^{-d_o})$ when $\|\theta_0 - \htheta\|_2=O(1)$.

\subsubsection{Proof of \eqref{eq:thm:spnd-bound:outer-iter:4th-moment}}

In \Cref{alg:stat-inf-spnd}, because $f$ is $\beta$ smooth, 
and $\forall \theta$ $f(\theta) - f(\htheta) \geq 0 $, we have
\begin{align*}
& (f(\theta_{t+1}) - f(\htheta))^2  \nonumber \\
 & =  ( f(\theta_t + g_t^L) - f(\htheta)  )^2 \nonumber \\
\leq & ( f(\theta_t) - f(\htheta) +  \langle \nabla f(\theta_t) , g_t^L \rangle + \tfrac{\beta}{2} \| g_t^L\|_2^2 )^2 \nonumber \\
& = ( f(\theta_t) - f(\htheta) )^2 +  2 \langle \nabla f(\theta_t) , g_t^L \rangle( f(\theta_t) - f(\htheta) ) \nonumber \\
	 & + \langle \nabla f(\theta_t) , g_t^L \rangle^2 + \tfrac{\beta^2}{4} \| g_t^L\|_2^4 
	 + 2 ( f(\theta_t) - f(\htheta) + \langle \nabla f(\theta_t) , g_t^L \rangle)  \tfrac{\beta}{2} \| g_t^L\|_2^2 . 
\end{align*}

Because we have 
\begin{align*}
& \Exp[  \langle \nabla f(\theta_t) , g_t^L \rangle( f(\theta_t) - f(\htheta) )  \mid \theta_t]   \nonumber \\
\lesssim & - \rho_t \| \nabla f(\theta_t)\|_2^2 ( f(\theta_t) - f(\htheta) )  + \delta_t^0 \|g_t^0\|_2^2 ( f(\theta_t) - f(\htheta) ) , 
\end{align*}
\begin{align*}
& \Exp\left[ \left\| \tfrac{g_t^0}{\rho_t} \right\|_2^4  \mid \theta_t\right] \nonumber \\
& =  \Exp \left[ \left\| \tfrac{1}{S_o} \sum_{i \in I_o}  ( \nabla f_i(\theta_t) - \nabla f_i(\htheta) + \nabla f_i(\htheta) ) \right\|_2^4  \mid \theta_t \right] \nonumber \\ 
\lesssim & 1 + \|\theta_t - \htheta\|_2^4 , 
\end{align*}
\begin{align*}
f(\theta_t) - f(\htheta) = \Theta( \|\theta_t - \htheta\|_2^2) = \Theta(\|\nabla f(\theta_t) \|_2^2)  ,  
\end{align*}
and by our choice of $\rho_t = \Theta((t+1)^{-d_o}) = o(1)$  and $\delta_t^0 = O(\rho_t^4)$ \eqref{eq:proof:thm:spnd-bound:delta_j}, after repeatedly applying the AM-GM inequality and \eqref{eq:thm:spnd-bound:outer-iter:2nd:f-bound}, 
we have 
\begin{align*}
&\Exp[ (f(\theta_{t+1}) - f(\htheta))^2  \mid \theta_t ] \nonumber \\
\leq &  (1- \Theta(\rho_t)) (f(\theta_t) - f(\htheta))^2 +O(\rho_t^3) . 
\end{align*}
Applying \Cref{lem:geometric-like-series-bound}, we have 
\begin{align}
& \Exp[ \|\theta_t - \htheta\|_2^4 ] \nonumber \\
\leq & \Exp\left[ \tfrac{4}{\alpha^2} (f(\theta_t) - f(\htheta))^2 \right] \nonumber \\
\lesssim & t^{-2d_o}  . \label{eq:thm:spnd-bound:outer-iter:4th:f-bound}
\end{align}

\subsubsection{Proof of  \eqref{eq:thm:spnd-bound:covariance-bound}}
\label{subsec:proof:thm:spnd-bound:eq:thm:spnd-bound:covariance-bound}

For $\frac{\bar{g}_t}{\rho_t}$, we have 
\begin{align}
\frac{\bar{g}_t}{\rho_t} &= \underbrace{ -H^{-1}\frac{1}{S_o} \sum_{i \in I_o} \nabla f_i(\htheta) }_{\mytag{[1]}{append:proof:thm:spnd-bound:cov:tmp1}} \nonumber \\
	& +  \underbrace{ H^{-1}\frac{1}{S_o} \sum_{i \in I_o} \nabla f_i(\htheta) -H_t^{-1} \frac{1}{S_o} \sum_{i \in I_o} \nabla f_i(\htheta) +H_t^{-1} \frac{1}{S_o} \sum_{i \in I_o} \nabla f_i(\htheta) -H_t^{-1} \frac{1}{S_o} \sum_{i \in I_o} \nabla f_i(\theta_t)   }_{\mytag{[2]}{append:proof:thm:spnd-bound:cov:tmp2}}  \nonumber \\   
	&\underbrace{ - H_t^{-1} \frac{g_t^0}{\rho_t}
	+ \frac{\bar{g}_t}{\rho_t} }_{\mytag{[3]}{append:proof:thm:spnd-bound:cov:tmp3}}  .  \label{eq:thm:spnd-bound:outer-sgd:decompose:terms}
\end{align}

Thus, for the ``covariance'' of our replicates, we have
\begin{align*}
&    \left\| H^{-1} G H^{-1}  - \frac{S_o}{T} \sum_{t=1}^T \frac{\bar{g}_t \bar{g}_t^\top}{\rho_t^2}  
		 \right\|_2  \nonumber \\
\lesssim  & \left\| H^{-1} G H^{-1}  -\frac{S_o}{T} \sum_{t=1}^T    \ref{append:proof:thm:spnd-bound:cov:tmp1}_t  \ref{append:proof:thm:spnd-bound:cov:tmp1}_t^\top 
		 \right\|_2  \nonumber \\ 
	& +  \left\| \frac{S_o}{T} \sum_{t=1}^T \ref{append:proof:thm:spnd-bound:cov:tmp1}_t ( \ref{append:proof:thm:spnd-bound:cov:tmp2}_t + \ref{append:proof:thm:spnd-bound:cov:tmp3}_t )^\top \right\|_2 +   \left\|\frac{S_o}{T}\sum_{t=1}^T (\ref{append:proof:thm:spnd-bound:cov:tmp2}_t + \ref{append:proof:thm:spnd-bound:cov:tmp3}_t )  \ref{append:proof:thm:spnd-bound:cov:tmp1}_t^\top\right\|_2 \nonumber \\ 
	& +  \left\|\frac{S_o}{T}\sum_{t=1}^T (\ref{append:proof:thm:spnd-bound:cov:tmp2}_t + \ref{append:proof:thm:spnd-bound:cov:tmp3}_t ) ( \ref{append:proof:thm:spnd-bound:cov:tmp2}_t + \ref{append:proof:thm:spnd-bound:cov:tmp3}_t  )^\top \right\|_2 \nonumber \\ 
& \blue{\text{because for two vectors $a, b$ the operator norm $\|a b^\top\|_2 \leq \|a\|_2\|b\|_2 $}} \nonumber \\ 
\lesssim & \left\|  H^{-1} G H^{-1}  -\frac{S_o}{T} \sum_{t=1}^T    \ref{append:proof:thm:spnd-bound:cov:tmp1}_t  \ref{append:proof:thm:spnd-bound:cov:tmp1}_t^\top \right\|_2 \nonumber \\
	& + \frac{1}{T} \sum_{t=1}^T \|\ref{append:proof:thm:spnd-bound:cov:tmp1}_t\|_2( \|\ref{append:proof:thm:spnd-bound:cov:tmp2}_t\|_2+\|\ref{append:proof:thm:spnd-bound:cov:tmp3}_t\|_2 ) \nonumber \\ 
	& + \frac{1}{T}\sum_{t=1}^T( \|\ref{append:proof:thm:spnd-bound:cov:tmp2}_t\|_2^2+\| \ref{append:proof:thm:spnd-bound:cov:tmp3}_t\|_2^2) . 
\end{align*}

Because $\sum_{t=1}^T    \ref{append:proof:thm:spnd-bound:cov:tmp1}_t $ consists of $S_o \cdot T$ i.i.d. samples from $\{H^{-1} \nabla f_i(htheta) \}_{i=1}^n$ and the mean $H^{-1} \nabla f(\htheta) = 0$, using matrix concentration \cite{tropp2015introduction}, we know that 
\begin{align*}
\Exp\left[ \left\| H^{-1} G H^{-1}  -\frac{S_o}{T} \sum_{t=1}^T    \ref{append:proof:thm:spnd-bound:cov:tmp1}_t  \ref{append:proof:thm:spnd-bound:cov:tmp1}_t^\top 
		 \right\|_2 \right] \lesssim \frac{1}{\sqrt{T}}. 
\end{align*}

For term \ref{append:proof:thm:spnd-bound:cov:tmp3}, using \eqref{eq:append:proof:thm:spnd-bound:bar-g-diff:decompose}, 
because we have 
\begin{align*}
&\sum_{t=1}^T \ref{append:proof:thm:spnd-bound:cov:tmp3}_t \nonumber \\
& = \underbrace{ \sum_{t=1}^T  \frac{1}{L+1} \sum_{k=0}^{L-1}      \tau_k \sum_{j=k+1}^L  \prod_{l=k+1}^{j-1} (I - \tau_l H_t)   (-  e_t^k ) }_{\blue{\text{when $a\neq b$ } \Exp[\langle e_t^a, e_t^b \rangle] = 0 }}  \nonumber \\ 
	 & + \sum_{t=1}^T     \frac{1}{L+1} \sum_{k=0}^{L-1}   \tau_k \sum_{j=k+1}^L  \prod_{l=k+1}^{j-1} (I - \tau_l H_t)   (-\tfrac{\nabla f(\theta_t + {\delta_t^k} g_t^k) - \nabla f(\theta_t) }{\delta_t^k} + H_t g_t^k )  , 
\end{align*}
by using \eqref{eq:append:proof:thm:spnd-bound:lem:hessian-taylor-approx:tmp5:bound} and \eqref{eq:append:proof:thm:spnd-bound:lem:hessian-taylor-approx:tmp7:bound:delta}, 
we have 
\begin{align}
& \Exp\left[\left\|\frac{1}{\sqrt{T}}\sum_{t=1}^T\ref{append:proof:thm:spnd-bound:cov:tmp3}_t \right\|_2^2 \right] \nonumber \\
\lesssim & \Exp\left[ \frac{1}{T}(\sum_{t=1}^T \frac{1}{L} + (\sum_{t=1}^T \frac{\sum_{k=0}^L\delta_t^k}{L} )^2) \left\|\tfrac{g_t^0}{\rho_t}\right\|_2^2\right] \nonumber \\
& \blue{\text{using \eqref{eq:proof:thm:spnd-bound:sgd:norm-decompose}, and by our choice of $\delta_t^k =\delta_t^0 o((k+1)^{-2})$  and $\delta_t^0 =o((t+1)^{-2})$ \eqref{eq:proof:thm:spnd-bound:delta_j}}}\nonumber \\
\lesssim & \Exp\left[\left(  \tfrac{1}{L} + \tfrac{\sum_{t=1}^T {\delta_t^0}^2}{T} \right) \left(1+\|\theta_t - \htheta|\|_2^2\right) \right]\nonumber \\
\lesssim & \frac{1}{L} + \frac{1}{T}. \label{eq:thm:spnd-bound:covariance-bound:proof:bar-g-t-covariance-bound}
\end{align}

And because we have 
\begin{align*}
\Exp[\|\ref{append:proof:thm:spnd-bound:cov:tmp1}_t\|_2] 
=  \Exp[\|-H^{-1}\frac{1}{S_o} \sum_{i \in I_o} \nabla f_i(\htheta)\|_2]
=  O(1) , 
\end{align*}
\begin{align}
&\Exp[\|\ref{append:proof:thm:spnd-bound:cov:tmp2}_t\|_2^2 \mid \theta_t]  \nonumber \\
\lesssim& \Exp\left[\left\|(H^{-1} - H_t^{-1}) \frac{1}{S_o} \sum_{i \in I_o} \nabla f_i(\htheta) \right\|_2^2
	+\left\|H_t^{-1} \frac{1}{S_o} \sum_{i \in I_o} (\nabla f_i(\htheta) - \nabla f_i(\theta_t)) \right\|_2^2
	 \mid \theta_t\right] \nonumber \\ 
& \blue{\text{because $H^{-1} - H_t^{-1} = H^{-1} (H_t - H) H_t^{-1}$  and using \Cref{append:proof:thm:spnd-bound:lem:hessian-taylor-approx}}} \label{eq:append:proof:thm:spnd-bound:cov:tmp2:proof:hessian-approximation} \\ 
\lesssim &  \Exp[\|\theta_t - \htheta|\|_2^2 \mid \theta_t] \nonumber \\ 
\lesssim & (t+1)^{-d_o} ,  
\end{align}
by repeatedly applying Cauchy-Schwarz inequality and AM-GM inequality, we can conclude that 
\begin{align*}
& \Exp\left[ \left\| H^{-1} G H^{-1}  - \frac{S_o}{T}  \sum_{t=1}^T \frac{\bar{g}_t \bar{g}_t^\top}{\rho_t^2}  
		 \right\|_2 \right] \nonumber \\
\lesssim & \frac{1}{\sqrt{T}} +  \frac{1}{T}\sum_{t=1}^T (t+1)^{-\frac{d_o}{2}} + \frac{1}{T}\sum_{t=1}^T (t+1)^{-d_o} + \frac{1}{\sqrt{L}} + \frac{1}{L} \nonumber \\ 
& \blue{\text{because $\sum_{t=1}^T (t+1)^{-\frac{d_o}{2}}  = T^{1-\frac{d_o}{2}}$ for $d_o \in (\frac{1}{2},1)$}} \nonumber \\
\lesssim & \frac{1}{T^{\frac{d_o}{2}}}  + \frac{1}{\sqrt{L}}. 
\end{align*}

\subsection{Proof of \Cref{cor:foasnd:asymptotic-normality:outer-avg}}

For $\frac{g_t^L}{\rho_t}$, we have 
\begin{align}
\frac{g_t^L}{\rho_t} &= \underbrace{ -H^{-1}\frac{1}{S_o} \sum_{i \in I_o} \nabla f_i(\htheta) }_{\mytag{[1]}{append:proof:thm:spnd-bound:normality:tmp1}} \nonumber \\
	& +  \underbrace{ H^{-1}\frac{1}{S_o} \sum_{i \in I_o} \nabla f_i(\htheta) -H_t^{-1} \frac{1}{S_o} \sum_{i \in I_o} \nabla f_i(\htheta) +H_t^{-1} \frac{1}{S_o} \sum_{i \in I_o} \nabla f_i(\htheta) -H_t^{-1} \frac{1}{S_o} \sum_{i \in I_o} \nabla f_i(\theta_t) + H_t^{-1} \nabla f(\theta_t)  }_{\mytag{[2]}{append:proof:thm:spnd-bound:normality:tmp2}}  \nonumber \\   
	& \underbrace{- H_t^{-1} \nabla f(\theta_t) + (\theta_t - \htheta)}_{\mytag{[3]}{append:proof:thm:spnd-bound:normality:tmp3}}  \underbrace{ - H_t^{-1} \frac{g_t^0}{\rho_t}
	+ \frac{g_t^L}{\rho_t} }_{\mytag{[4]}{append:proof:thm:spnd-bound:normality:tmp4}}   - (\theta_t - \htheta),  \label{eq:thm:spnd-bound:outer-sgd:decompose:terms:normality}
\end{align}
which gives 
\begin{align*}
& \theta_t - \htheta \nonumber \\
& = (1-\rho_{t-1}) (\theta_{t-1} - \htheta) + \rho_{t-1} (\ref{append:proof:thm:spnd-bound:normality:tmp1}_{t-1} + \ref{append:proof:thm:spnd-bound:normality:tmp2}_{t-1} + \ref{append:proof:thm:spnd-bound:normality:tmp3}_{t-1} + \ref{append:proof:thm:spnd-bound:normality:tmp4}_{t-1}) \nonumber \\ 
& = (\prod_{i=0}^{t-1} (1-\rho_i)) (\theta_0  - \htheta) + \sum_{i=0}^{t-1} (\prod_{j=i+1}^{t-1}(1-\rho_j)) \rho_i (\ref{append:proof:thm:spnd-bound:normality:tmp1}_i + \ref{append:proof:thm:spnd-bound:normality:tmp2}_i + \ref{append:proof:thm:spnd-bound:normality:tmp3}_i + \ref{append:proof:thm:spnd-bound:normality:tmp4}_i) . 
\end{align*}
And we have
\begin{align}
& \sqrt{T}(\frac{\sum_{t=1}^T \theta_t}{T} - \htheta)  \nonumber \\ 
& = \frac{1}{\sqrt{T}}(\sum_{t=1}^T \prod_{i=0}^{t-1} (1-\rho_i))  (\theta_0  - \htheta) 
	+  \frac{1}{\sqrt{T}}  \sum_{t=1}^T \sum_{i=0}^{t-1} (\prod_{j=i+1}^{t-1}(1-\rho_j)) \rho_i (\ref{append:proof:thm:spnd-bound:normality:tmp1}_i + \ref{append:proof:thm:spnd-bound:normality:tmp2}_i + \ref{append:proof:thm:spnd-bound:normality:tmp3}_i + \ref{append:proof:thm:spnd-bound:normality:tmp4}_i)
 \nonumber \\ 
& = \frac{1}{\sqrt{T}}(\sum_{t=1}^T \prod_{i=0}^{t-1} (1-\rho_i))  (\theta_0  - \htheta) 
	+  \frac{1}{\sqrt{T}} \sum_{i=0}^{T-1} \sum_{t=i+1}^T (\prod_{j=i+1}^{t-1}(1-\rho_j)) \rho_i (\ref{append:proof:thm:spnd-bound:normality:tmp1}_i + \ref{append:proof:thm:spnd-bound:normality:tmp2}_i + \ref{append:proof:thm:spnd-bound:normality:tmp3}_i + \ref{append:proof:thm:spnd-bound:normality:tmp4}_i) .
	\label{eq:thm:spnd-bound:outer-sgd:decompose:sum-terms:normality}
\end{align}

For the first  term in \eqref{eq:thm:spnd-bound:outer-sgd:decompose:sum-terms:normality}, which is non-stochastic, we have
\begin{align*}
\left\| \frac{1}{\sqrt{T}}(\sum_{t=1}^T \prod_{i=0}^{t-1} (1-\rho_i))  (\theta_0  - \htheta) \right\|_2 \lesssim  \frac{1}{\sqrt{T}} . 
\end{align*} 

For the second  term in \eqref{eq:thm:spnd-bound:outer-sgd:decompose:sum-terms:normality}, which is stochastic, we first consider $\rho_i \sum_{t=i+1}^T \prod_{j=i+1}^{t-1}(1-\rho_j) $, which is $O(1)$ (similar to \eqref{eq:append:proof:thm:spnd-bound:lem:hessian-taylor-approx:tmp5:bound}) and satisfies 
\begin{align*}
& \rho_i \sum_{t=i+1}^T \prod_{j=i+1}^{t-1}(1-\rho_j) \nonumber \\
& = \sum_{t=i+1}^T \frac{\rho_i}{\rho_t} \rho_t \prod_{j=i+1}^{t-1}(1-\rho_j) \nonumber \\
\leq & \frac{\rho_i}{\rho_s} \sum_{t=i+1}^s  \rho_t \prod_{j=i+1}^{t-1}(1-\rho_j) + \rho_i  (\prod_{j=i+1}^{s}(1-\rho_j)) \sum_{t=s+1}^T \prod_{j=s+1}^{t-1}(1-\rho_j) \nonumber \\
& = (1 + \frac{\rho_i - \rho_s}{\rho_s}) (1-\prod_{t=i+1}^s(1-\rho_t))  + \rho_i  (\prod_{j=i+1}^{s}(1-\rho_j)) \sum_{t=s+1}^T \prod_{j=s+1}^{t-1}(1-\rho_j)  \nonumber \\
\leq & (1 + \frac{\rho_i - \rho_s}{\rho_s}) (1-(1-\rho_s)^{s-i}) + \rho_i (1-\rho_s)^{s-i} \sum_{t=s+1}^T \prod_{j=s+1}^{t-1}(1-\rho_j) \nonumber \\
\leq & 1+ ((1+ \frac{s-i}{i+1})^{d_o} - 1) + \rho_i e^{-(s-i)\rho_s} \sum_{t=s+1}^\infty \prod_{j=s+1}^{t-1}(1-\rho_j) \nonumber \\
\leq & 1 + \frac{s-i}{i} + \rho_i e^{-(s-i)\rho_s} \sum_{t=s+1}^\infty \prod_{j=s+1}^{t-1}(1-\rho_j) ,   
\end{align*}
for all $i \leq s \leq T $, and 
\begin{align*}
& \rho_i \sum_{t=i+1}^T \prod_{j=i+1}^{t-1}(1-\rho_j) \nonumber \\
& \geq \sum_{t=i+1}^T (\prod_{j=i+1}^{t-1}(1-\rho_j))\rho_t \nonumber \\ 
& = 1 - \prod_{t=i+1}^T(1-\rho_t) \nonumber \\
& \geq 1 - \exp(-\sum_{t=i+1}^T \rho_t) \nonumber \\
& \geq 1 - \exp(- \frac{1}{1-d_o}((T+2)^{1-d_o}- (i+2)^{1-d_o}))
\end{align*}
When we choose $s = i + \lceil (i+1)^{\frac{d_o+1}{2}} \rceil$, we have 
$\frac{s-i}{i} \lesssim i^{\frac{-1+d_o}{2}}$, $(s-i)\rho_s \gtrsim  (i+1)^{\frac{1-d_o}{2}}$, and $ \rho_i e^{-\frac{1}{2}(s-i)\rho_s} \lesssim \rho_s$. 
And these imply $| \rho_i \sum_{t=i+1}^T \prod_{j=i+1}^{t-1}(1-\rho_j) - 1|= O(\max\{(i+1)^{\frac{-1+d_o}{2}}, \exp(- \frac{1}{1-d_o}((T+2)^{1-d_o}- (i+2)^{1-d_o}) \})$. 
Thus, for term \ref{append:proof:thm:spnd-bound:normality:tmp1}, we have
\begin{align*}
&\frac{1}{\sqrt{T}} \sum_{i=0}^{T-1} \sum_{t=i+1}^T \prod_{j=i+1}^{t-1}(1-\rho_j) \rho_i \ref{append:proof:thm:spnd-bound:normality:tmp1}_i 
= \frac{1}{\sqrt{T}} \sum_{i=0}^{T-1} \ref{append:proof:thm:spnd-bound:normality:tmp1}_i 
	 +  \frac{1}{\sqrt{T}} \sum_{i=0}^{T-1} (\sum_{t=i+1}^T \prod_{j=i+1}^{t-1}(1-\rho_j) \rho_i -1) \ref{append:proof:thm:spnd-bound:normality:tmp1}_i , 
\end{align*}
where the first term weakly converges to $\Norm(0, \frac{1}{S_o} H^{-1} G H^{-1})$ by Central Limit Theorem, and the second term satisfies $\Exp[\|\frac{1}{\sqrt{T}} \sum_{i=0}^{T-1} (\sum_{t=i+1}^T \prod_{j=i+1}^{t-1}(1-\rho_j)) \rho_i -1) \ref{append:proof:thm:spnd-bound:normality:tmp1}_i\|_2^2]=\Exp[\frac{1}{T} \sum_{i=0}^{T-1}|(\sum_{t=i+1}^T \prod_{j=i+1}^{t-1}(1-\rho_j)) \rho_i -1)|^2 \| \ref{append:proof:thm:spnd-bound:normality:tmp1}_i\|_2^2] \lesssim T^{d_o-1} + \frac{1}{T}$.

For term \ref{append:proof:thm:spnd-bound:normality:tmp2}, we have
\begin{align*}
\|\ref{append:proof:thm:spnd-bound:normality:tmp2}_t\|_2 \lesssim \|\theta_t - \htheta\|_2, 
\end{align*}
and $\Exp[\langle \ref{append:proof:thm:spnd-bound:normality:tmp2}_a, \ref{append:proof:thm:spnd-bound:normality:tmp2}_b \rangle ]=0$ when $a\neq b$. 
Thus 
\begin{align*}
\Exp\left[ \left\| \frac{1}{\sqrt{T}} \sum_{i=0}^{T-1} \sum_{t=i+1}^T \prod_{j=i+1}^{t-1}(1-\rho_j) \rho_i \ref{append:proof:thm:spnd-bound:normality:tmp2}_i \right\|_2^2  \right] 
\lesssim \frac{1}{T} \sum_{i=0}^{T-1} \|\theta_t - \htheta\|_2^2 \lesssim T^{-d_o} .
\end{align*}

For term \ref{append:proof:thm:spnd-bound:normality:tmp3}, we have
\begin{align*}
\| - H_t^{-1} \nabla f(\theta_t) + (\theta_t - \htheta) \|_2 \lesssim \|\theta_t - \htheta\|_2^2 . 
\end{align*}
By using \eqref{eq:thm:spnd-bound:outer-iter:4th-moment} and  Cauchy-Schwarz inequality, we have 
\begin{align*}
\Exp\left[ \left\| \frac{1}{\sqrt{T}} \sum_{i=0}^{T-1} \sum_{t=i+1}^T \prod_{j=i+1}^{t-1}(1-\rho_j) \rho_i \ref{append:proof:thm:spnd-bound:normality:tmp3}_i \right\|_2^2 \right] \lesssim T^{1-2d_o} . 
\end{align*}

For term \ref{append:proof:thm:spnd-bound:normality:tmp4}, similar to similar to  \eqref{eq:thm:spnd-bound:covariance-bound:proof:bar-g-t-covariance-bound}, we have
\begin{align*}
\Exp\left[ \left\| \frac{1}{\sqrt{T}} \sum_{i=0}^{T-1} \sum_{t=i+1}^T \prod_{j=i+1}^{t-1}(1-\rho_j) \rho_i \ref{append:proof:thm:spnd-bound:normality:tmp4}_i \right\|_2^2 \right]  \lesssim \frac{1}{T} + \frac{1}{L}. 
\end{align*}

\subsection{Proof of \Cref{cor:svrg-foasnd:bounds}}

Using Theorem 6.5 of \cite{bubeck2015convex}, we have 
\begin{align*}
\Exp[\|\theta_t - \htheta\|_2^2] \lesssim 0.9^t.
\end{align*}

Similar to \eqref{eq:thm:spnd-bound:newton} in \Cref{thm:spnd-bound} (\Cref{subsec:proof:thm:spnd-bound:eq:thm:spnd-bound:newton}), we have 
\begin{align*}
\Exp \left[ \left\| \frac{\bar{g}_t}{\rho_t} - [\nabla^2 f(\theta_t)]^{-1} g_t^0 \right\|_2^2 \mid \theta_t \right] \lesssim \tfrac{1}{L} \|g_t^0\|_2^2 . 
\end{align*}

Similar to the proof of \eqref{eq:thm:spnd-bound:covariance-bound} in \Cref{thm:spnd-bound} (\Cref{subsec:proof:thm:spnd-bound:eq:thm:spnd-bound:covariance-bound}), using \eqref{eq:thm:spnd-bound:outer-sgd:decompose:terms}, we have
\begin{align*}
	  \Exp \left[ \left\| H^{-1} G H^{-1}  - \frac{S_o}{T} \sum_{t=1}^T \frac{\bar{g}_t \bar{g}_t^\top}{\rho_t^2}  
		 \right\|_2 \right] \lesssim L^{-\frac{1}{2}} .   
\end{align*}

For $\frac{g_t^L}{\rho_t}$, we have 
\begin{align}
\frac{\bar{g}_t}{\rho_t} &= \underbrace{ -H^{-1}\frac{1}{S_o} \sum_{i \in I_o} \nabla f_i(\htheta) }_{\mytag{[1]}{append:proof:thm:svrg-bound:normality:tmp1}} \nonumber \\
	& +  \underbrace{ H^{-1}\frac{1}{S_o} \sum_{i \in I_o} \nabla f_i(\htheta) -H_t^{-1} \frac{1}{S_o} \sum_{i \in I_o} \nabla f_i(\htheta) +H_t^{-1} \frac{1}{S_o} \sum_{i \in I_o} \nabla f_i(\htheta) -H_t^{-1} \frac{1}{S_o} \sum_{i \in I_o} \nabla f_i(\theta_t) + H_t^{-1} \nabla f(\theta_t)  }_{\mytag{[2]}{append:proof:thm:svrg-bound:normality:tmp2}}  \nonumber \\   
	& \underbrace{- H_t^{-1} \nabla f(\theta_t) }_{\mytag{[3]}{append:proof:thm:svrg-bound:normality:tmp3}}  \underbrace{ - H_t^{-1} \frac{g_t^0}{\rho_t}
	+ \frac{g_t^L}{\rho_t} }_{\mytag{[4]}{append:proof:thm:svrg-bound:normality:tmp4}}   .  \label{eq:thm:svrg-bound:outer-sgd:decompose:terms:normality}
\end{align}

For term \ref{append:proof:thm:svrg-bound:normality:tmp1}, we have
\begin{align*}
\frac{1}{\sqrt{T}} \sum_{i=1}^T \ref{append:proof:thm:svrg-bound:normality:tmp1}_t = \frac{1}{\sqrt{T}} \sum_{i=1}^T \left( -H^{-1}\frac{1}{S_o} \sum_{i \in I_o} \nabla f_i(\htheta)  \right)_t , 
\end{align*}
which consists of $S_o \cot T$ i.i.d samples from 0 mean set $\{H_{-1} \nabla f_i(\htheta)\}_{i=1}^n$, 
and weakly converges to $\Norm(0, \frac{1}{S_o} H^{-1} G H^{-1})$ by the Central Limit Theorem. 

For term \ref{append:proof:thm:svrg-bound:normality:tmp2}, similar to \eqref{eq:append:proof:thm:spnd-bound:cov:tmp2:proof:hessian-approximation},  we have
\begin{align*}
\Exp\left[\left\|\frac{1}{\sqrt{T}} \sum_{i=1}^T \ref{append:proof:thm:svrg-bound:normality:tmp2}_t \right\|_2^2 \right]  
	= \frac{1}{T}\Exp[ \sum_{i=1}^T\| \ref{append:proof:thm:svrg-bound:normality:tmp2}_t \|_2^2] 
 \lesssim \frac{1}{T} \sum_{t=1}^T\Exp[ \|\theta_t - \htheta\|_2^2] \lesssim  \frac{1}{T}  . 
\end{align*}

For term \ref{append:proof:thm:svrg-bound:normality:tmp3},   we have
\begin{align*}
\Exp \left[ \left\|\frac{1}{\sqrt{T}} \sum_{i=1}^T \ref{append:proof:thm:svrg-bound:normality:tmp3}_t \right\|_2 \right]  \lesssim \frac{1}{\sqrt{T}}  \Exp[ \|\theta_t - \htheta\|_2] \lesssim \frac{1}{\sqrt{T}} .
\end{align*}

For term \ref{append:proof:thm:svrg-bound:normality:tmp4}, similar to  \eqref{eq:thm:spnd-bound:covariance-bound:proof:bar-g-t-covariance-bound},  we have
\begin{align*}
\Exp\left[\left\|\frac{1}{\sqrt{T}} \sum_{i=1}^T \ref{append:proof:thm:svrg-bound:normality:tmp4}_t \right\|_2 \right]  \lesssim \frac{1}{\sqrt{T}}  + \frac{1}{\sqrt{L}} . 
\end{align*}

%
%

%
%

\subsection{Proof of \Cref{thm:lasso-mod-cov-bounds}}

The error bound proof is similar to standard LASSO proofs \cite{buhlmann2011statistics, negahban2012unified}.

We will use \Cref{lem:proof:lasso:linear:cov:soft-thresh:guarantees} for the  covariance estimate  using soft thresholding.

We denote ``soft thresholding by $\omega$''   as an element-wise procedure $\Shrink_\omega(A) = \sign(A)(|A| - \omega)_+$ , where $A$ is an  arbitrary number, vector, or matrix,  and $\omega$ is non-negative.

\begin{lemma}
\label{lem:proof:lasso:linear:cov:soft-thresh:guarantees}

Under our assumptions in \Cref{sec:high-dim:lasso:linear-regression},  
we choose soft threshold $\frac{1}{n} \sum_{i=1}^n X_i X_i^\top$ using 
\begin{align*}
\omega =  \Theta\left(\sqrt{\frac{\log p}{n}}\right) .
\end{align*} 

When $n \gg \log p $,
the matrix max norm of $\frac{1}{n}\sum_{i=1}^n x_i x_i^\top - \Sigma$ is bounded by
\begin{align*}
\max_{1 \leq i,j \leq p} \left|\left(\frac{1}{n}\sum_{i=1}^n x_i x_i^\top\right)_{ij} - \Sigma_{ij}\right| \lesssim \sqrt{\frac{\log p}{n}}, 
\end{align*}
with probability at least $1-p^{-\Theta(1)}$.

Under this event, 
$\ell_2$ operator norm of  $\widehat{S} - \Sigma$ satisfies  
\begin{align*}
\|\widehat{S} - \Sigma\|_2 \lesssim b \sqrt{\frac{\log p}{n}} , 
\end{align*}
$\ell_1$  and $\ell_\infty$ operator norm of  $\widehat{S} - \Sigma$ satisfies  
\begin{align*}
\|\widehat{S} - \Sigma\|_\infty = \|\widehat{S} - \Sigma\|_1 \lesssim b \sqrt{\frac{\log p}{n}} . 
\end{align*}

\end{lemma}
\begin{proof}

The proof is similar to that of Theorem 1, \cite{bickel2008covariance}. 

Our assumption that $\Sigma$ is well conditioned implies that each  off diagonal entry is bounded, and each diagonal entry is $\Theta(1)$ and positive.

Omitting the subscript for the $i^\text{th}$ sample, 
for each i.i.d. sample $x=[x(1), x(2), \dots, x(p)]^\top \sim \Norm(0, \Sigma) $, each $x(j) x(k)$ satisfies 
\begin{align*}
x(j) x(k) = \frac{1}{4} (x(j) + x(k))^2 - \frac{1}{4} (x(j) - x(k))^2 , 
\end{align*}
where $x(j) \pm x(k)$ are Gaussian random variables with variance $\Sigma_{jj} \pm 2 \Sigma_{jk} + \Sigma_{kk} = \Theta(1)$, because all of $\Sigma$'s eigenvalues are upper and lower bounded. 
Thus, $x(j) \pm x(k)$ are  $\chi_1^2$ random variables scaled by $\Sigma_{jj} \pm 2 \Sigma_{jk} + \Sigma_{kk} = \Theta(1)$, and they are sub-exponential with parameters that are $\Theta(1)$ \cite{Wainwright2015concentration-lecture}. 
And this implies that, $x(j) x(k) - \Sigma_{jk} $ is sub-exponential 
\begin{align*}
\Pr[|x(j) x(k) - \Sigma_{jk} |>t] \lesssim \exp(-\Theta(\min\{t^2, t\})) , 
\end{align*} 
for all $1 \leq j,k\leq p$. 

Using  Bernstein inequality  \cite{Wainwright2015concentration-lecture}, we have
\begin{align*}
\Pr\left[ \left|\left(\frac{1}{n}\sum_{i=1}^n x_i x_i^\top\right)_{jk} - \Sigma_{jk}\right|  > t\right] \lesssim \exp(-n\Theta(\min\{t^2, t\})) ,
\end{align*}
for all $1 \leq j,k\leq p$.

Taking a union bound over all matrix entries, and using $n \gg \log p $, we have  
\begin{align*}
\max_{1 \leq j,k\leq p} \left|\left((\frac{1}{n}\sum_{i=1}^n x_i x_i^\top\right)_{jk} - \Sigma_{jk}\right| \lesssim \sqrt{\frac{\log p + \log \frac{1}{\delta}}{n}} , 
\end{align*}
with probability at least $1-\delta$. 

Under this event, the soft thresholding estimate  $\Shrink_\omega(\frac{1}{n}\sum_{i=1}^n x_i x_i^\top)_{ij}$ with $\omega=\Theta(\sqrt{\frac{\log p}{n}} )$ is $0$ when $\Sigma_{ij}=0$, 
and $|\Sigma_{ij}-\Shrink_\omega(\frac{1}{n}\sum_{i=1}^n x_i x_i^\top)_{ij}| \leq \omega$ (even when $|\Sigma_{ij}|\leq \omega$). 
And this implies our bounds.

\end{proof}

\Cref{lem:proof:lasso:linear:cov:soft-thresh:guarantees}  guarantees that the optimization problem \eqref{eq:lasso-mod-cov} is well defined with high probability when $n \gg b \sqrt{\frac{\log p}{n}}$.  
Because the $\ell_2$ operator norm $\| \widehat{S} - \Sigma \|_2 \lesssim b \sqrt{\frac{\log p}{n}}  \ll 1$, and the positive definite matrix $\Sigma$'s eigenvalues are all $\Theta(1)$, the symmetric matrix $\widehat{S}$ is positive definite, and  $\widehat{S}$'s eigenvalues are all $\Theta(1)$, and for all $v\in \Real^p$ we have 
\begin{align}
0 \leq v^\top \widehat{S} v = \Theta(\|v\|_2^2) . \label{eq:proof:thm:lasso-mod-cov-bounds:tmp5:cov-thresholded-PD}
\end{align}

Because $\htheta$ attains the minimum, by definition, we have 
\begin{align*}
& \tfrac{1}{2} \htheta^\top \left(\widehat{S}- \tfrac{1}{n} \sum_{i=1}^n x_i x_i^\top \right) \htheta 
	+ \tfrac{1}{n} \sum_{i=1}^n \tfrac{1}{2} (x_i^\top \htheta -y_i)^2 + \lambda \ltynorm{\htheta}{1} \nonumber \\
\leq & \tfrac{1}{2} {\theta^{\star}}^\top \left(\widehat{S}- \frac{1}{n} \sum_{i=1}^n x_i x_i^\top\right) \theta^{\star} 
	+ \tfrac{1}{n} \sum_{i=1}^n \tfrac{1}{2} \left(x_i^\top \theta^{\star} -y_i\right)^2 + \lambda \ltynorm{\theta^{\star}}{1} , 
\end{align*}
which, after rearranging terms, is equivalent to 
\begin{align}
\tfrac{1}{2} (\htheta - \theta^{\star})^\top \widehat{S} (\htheta - \theta^{\star})  
	+ \left\langle \left(\widehat{S}- \tfrac{1}{n} \sum_{i=1}^n x_i x_i^\top\right)\theta^{\star} + \tfrac{1}{n} \sum_{i=1} \epsilon_i x_i , \htheta - \theta^{\star} \right\rangle  \leq \lambda (\|\theta^{\star}\|_1 -\|\htheta \|_1) . 
	\label{eq:proof:thm:lasso-mod-cov-bounds:tmp1}
\end{align}

Because $\widehat{S} = \Shrink_\omega( \frac{1}{n} \sum_{i=1}^n x_i x_i^\top)$ soft thresholds each entry of $ \frac{1}{n} \sum_{i=1}^n x_i x_i^\top$ with $\omega =\Theta(\sqrt{\frac{\log p}{n}})$, each entry of $\widehat{S}- \frac{1}{n} \sum_{i=1}^n x_i x_i^\top$ will lie in the interval $[-\omega, \omega]$. And this implies , with probability at least $1-p^{-\Theta(1)}$, we have 
\begin{align*}
\left\| \left(\widehat{S}- \frac{1}{n} \sum_{i=1}^n x_i x_i^\top\right)\theta^{\star} \right\|_\infty \lesssim \sqrt{\frac{\log p}{n}} \|\theta^{\star}\|_1 \lesssim \sqrt{\frac{s \log p}{n}}, 
\end{align*}
where we used the assumption that $\theta^{\star}$ is $s$ sparse and $\|\theta^{\star}\|_2=O(1)$, which implies $\|\theta^{\star}\|_1 \lesssim \sqrt{s}$.

For the $j^{\text{th}}$ coordinate of $\epsilon_i x_i$, because $\epsilon_i$ and $x_i$ are independent Gaussian random variables, we know that it is sub-exponential \cite{Wainwright2015concentration-lecture}
\begin{align}
\Pr[|\epsilon_i x_i(j)| > t] \lesssim \exp\left(-\Theta\left(\min\left\{\tfrac{t^2}{\sigma^2},\tfrac{t}{\sigma}\right\}\right)\right) ,
	\label{eq:proof:thm:lasso-mod-cov-bounds:tmp2:inf-bound-1} 
\end{align}
for all $1 \leq i \leq n$ and $1 \leq j \leq p$. 

Using Bernstein inequality, we have 
\begin{align*}
\Pr[|\frac{1}{n}\sum_{i=1}^n \epsilon_i x_i(j)| > t] \lesssim \exp\left(-\Theta\left(n\min\left\{\tfrac{t^2}{\sigma^2},\tfrac{t}{\sigma}\right\}\right)\right) ,
\end{align*}
for all   $1 \leq j \leq p$. 

Taking a union bound over all $p$ coordinates, with probability at least $1 - p^{-\Theta(1)}$, we have
\begin{align}
\|\frac{1}{n}\sum_{i=1}^n \epsilon_i x_i\|_\infty \lesssim \sigma \sqrt{\frac{\log p}{n}} ,
	\label{eq:proof:thm:lasso-mod-cov-bounds:tmp2:inf-bound-2}
\end{align}
when $n \gg \log p $.

Thus, we set the regularization parameter 
\begin{align}
\lambda =& \Theta\left( (\sigma + \|\theta^{\star}\|_1) \sqrt{\frac{\log p}{n}} \right) \nonumber \\
\geq & 2\left\|\left(\widehat{S}- \frac{1}{n} \sum_{i=1}^n x_i x_i^\top\right)\theta^{\star} + \frac{1}{n} \sum_{i=1} \epsilon_i x_i\right\|_\infty, \label{eq:proof:thm:lasso-mod-cov-bounds:tmp3:reg:lambda}
\end{align}
which holds under the events  in \Cref{lem:proof:lasso:linear:cov:soft-thresh:guarantees} and \eqref{eq:proof:thm:lasso-mod-cov-bounds:tmp2:inf-bound-2}. 

For a vector $v \in \Real^p$, 
let $v^S$ indicate the sub-vector of on the support of $\theta^{\star}$, 
and $v^{\bar{S}}$ the sub-vector not on the support of $\theta^{\star}$. 

\eqref{eq:proof:thm:lasso-mod-cov-bounds:tmp1} and \eqref{eq:proof:thm:lasso-mod-cov-bounds:tmp3:reg:lambda} implies that 
\begin{align*}
- \tfrac{1}{2} \lambda (\|(\theta - \theta^{\star})^S\|_1 +\|\theta^{\bar{S}}\|_1) = - \tfrac{1}{2}  \lambda \|\theta - \theta^{\star}\|_1 \leq  \lambda (\|\theta^{\star}\|_1 -\|\htheta \|_1) \leq \lambda (\|(\theta - \theta^{\star})^S\|_1 -\|\theta^{\bar{S}}\|_1), 
\end{align*}
which is equivalent to 
\begin{align}
\|\theta^{\bar{S}}\|_1 \leq 3 \|(\theta - \theta^{\star})^S\|_1 , \label{eq:proof:thm:lasso-mod-cov-bounds:tmp4:RE-cone}
\end{align}
because $\lambda > 0$. 

For any vector $v \in \Real^p$, it satisfies $\|v\|_2^2 \geq  \|v^S\|_2^2 \geq \frac{1}{s} \|v^S\|_1^2$. 
Using this in \eqref{eq:proof:thm:lasso-mod-cov-bounds:tmp1}, we have 
\begin{align*}
\tfrac{1}{s} \|(\theta - \theta^{\star})^S\|_1^2 \lesssim \lambda \|(\theta - \theta^{\star})^S\|_1 , 
\end{align*}
which implies that 
\begin{align}
\|(\theta - \theta^{\star})^S\|_1 \lesssim s (\sigma + \|\theta^{\star}\|_1) \sqrt{\frac{\log p}{n}} . \label{eq:proof:thm:lasso-mod-cov-bounds:tmp5}
\end{align}
Combining \eqref{eq:proof:thm:lasso-mod-cov-bounds:tmp5} and \eqref{eq:proof:thm:lasso-mod-cov-bounds:tmp4:RE-cone}, we have proven \eqref{eq:thm:lasso-mod-cov-bounds:l1} 
\begin{align*}
\|\theta - \theta^{\star}\|_1 \lesssim s \left(\sigma + \|\theta^{\star}\|_1\right) \sqrt{\frac{\log p}{n}} \lesssim s \left( \sigma + \sqrt{s} \right)\sqrt{\frac{\log p}{n}} . 
\end{align*}

In \eqref{eq:proof:thm:lasso-mod-cov-bounds:tmp1} because $\langle (\widehat{S}- \frac{1}{n} \sum_{i=1}^n x_i x_i^\top)\theta^{\star} + \frac{1}{n} \sum_{i=1} \epsilon_i x_i , \htheta - \theta^{\star} \rangle \geq 0$ by convexity, and using \eqref{eq:proof:thm:lasso-mod-cov-bounds:tmp5:cov-thresholded-PD}, we have proven \eqref{eq:thm:lasso-mod-cov-bounds:l2} 
\begin{align*}
\|\theta - \theta^{\star}\|_2^2 \lesssim \lambda \|(\theta - \theta^{\star})^S\|_1 
	\lesssim  s \left(\sigma + \|\theta^{\star}\|_1\right)^2 \frac{\log p}{n} \lesssim s \left( \sigma + \sqrt{s} \right)^2 \frac{\log p}{n} . 
\end{align*}

\subsection{Proof of \Cref{thm:lasso-mod:de-bias:stat-inf}}

At the solution $\htheta$ of the optimization problem \eqref{eq:lasso-mod-cov}, 
using the KKT condition,  we have 
\begin{align}
\left(\widehat{S} - \frac{1}{n}\sum_{i=1}^n x_i x_i^\top\right) \htheta + \frac{1}{n} \sum_{i=1}^n x_i (x_i^\top \htheta - y_i) + \lambda \widehat{g} = 0, \label{eq:proof:thm:lasso-mod:de-bias:stat-inf:tmp1:kkt}
\end{align}
where $\widehat{g} \in \partial \|\htheta\|_1$. And this is equivalent to 
\begin{align}
\widehat{S} \htheta - \frac{1}{n} \sum_{i=1}^n x_i ( x_i^\top \theta^{\star} + \epsilon_i) + \lambda \widehat{g} = 0, \label{eq:proof:thm:lasso-mod:de-bias:stat-inf:tmp2:kkt} .
\end{align}

By \Cref{lem:proof:lasso:linear:cov:soft-thresh:guarantees}, we know that $\widehat{S}$ is invertible when $n \gg b^2 \log p$. 

Plugging \eqref{eq:lasso-mod-cov:de-bias:theta} into \eqref{eq:proof:thm:lasso-mod:de-bias:stat-inf:tmp2:kkt}, we have 
\begin{align*}
\widehat{S} (\htheta^\diff - \widehat{S}^{-1} \left[  \frac{1}{n} \sum_{i=1}^n y_i x_i - \left( \frac{1}{n}\sum_{i=1}^n x_i x_i^\top \right) \htheta   \right] )- \frac{1}{n} \sum_{i=1}^n x_i ( x_i^\top \theta^{\star} + \epsilon_i) + \lambda \widehat{g} = 0, 
\end{align*}
which is equivalent to 
\begin{align}
\widehat{S} ( \htheta^\diff- \theta^{\star}) - \frac{1}{n} \sum_{i=1}^n \epsilon_i x_i + \left( \frac{1}{n}\sum_{i=1}^n x_i x_i^\top  -\widehat{S} \right) (\htheta - \theta^{\star})= 0, \label{eq:proof:thm:lasso-mod:de-bias:stat-inf:tmp3} 
\end{align}
where we used the fact that $\lambda \widehat{g} = -\widehat{S} \htheta + \frac{1}{n} \sum_{i=1}^n x_i ( x_i^\top \theta^{\star} + \epsilon_i)$. 

Rewriting \eqref{eq:proof:thm:lasso-mod:de-bias:stat-inf:tmp3}, we have
\begin{align}
\htheta^\diff- \theta^{\star} = \widehat{S}^{-1}  \frac{1}{n} \sum_{i=1}^n \epsilon_i x_i  + \left(I - \widehat{S}^{-1} \left(\frac{1}{n}\sum_{i=1}^n x_i x_i^\top\right) \right) (\htheta - \theta^{\star}). \label{eq:proof:thm:lasso-mod:de-bias:stat-inf:tmp4:decomposed:de-bias:diff} 
\end{align}

For $\max_{1 \leq j,k \leq p} \left|\left(I - \widehat{S}^{-1} \left(\frac{1}{n}\sum_{i=1}^n x_i x_i^\top\right)\right)_{jk}\right|$, we have
\begin{align}
& \max_{1 \leq j,k \leq p} \left|\left(I - \widehat{S}^{-1} \left(\frac{1}{n}\sum_{i=1}^n x_i x_i^\top\right)\right)_{jk}\right| \nonumber \\
=& \max_{1 \leq j,k \leq p} \left| \left(\widehat{S}^{-1}\left(S - \frac{1}{n}\sum_{i=1}^n x_i x_i^\top\right)\right)_{jk}\right|  \nonumber \\
\leq & \|\widehat{S}^{-1}\|_\infty  \max_{1 \leq j,k \leq p}\left| \left(S - \frac{1}{n}\sum_{i=1}^n x_i x_i^\top\right)_{jk}\right| . \label{eq:proof:thm:lasso-mod:de-bias:stat-inf:tmp5:I-S_inv_cov}
\end{align}

Under the event in \Cref{lem:proof:lasso:linear:cov:soft-thresh:guarantees}, we have 
\begin{align}
 \max_{1 \leq j,k \leq p} \left| \left(S - \frac{1}{n}\sum_{i=1}^n x_i x_i^\top\right)_{jk}\right| \lesssim \sqrt{\frac{\log p}{n}}.  \label{eq:proof:thm:lasso-mod:de-bias:stat-inf:tmp6}
\end{align}

Also under the event in \Cref{lem:proof:lasso:linear:cov:soft-thresh:guarantees}, 
we have 
\begin{align*}
\widehat{S}_{ii} - \sum_{j \neq i} |\widehat{S}_{ij}| \geq \Sigma_{ii} - \Theta\left(\sqrt{\tfrac{\log p}{n}}\right) - \sum_{j \neq i} |\Sigma_{ij}| \geq D_\Sigma - \Theta\left(\sqrt{\tfrac{\log p}{n}}\right) , 
\end{align*}
where we used $\widehat{S}_{ii} > 0$ and $|\Sigma_{ij}| \geq |\widehat{S}_{ij}|$ by definition of the soft thresholding operation.  

Thus, when $n \gg   \tfrac{1}{{D_\Sigma}^2} \log p  $, 
we have
\begin{align*}
\widehat{S}_{ii} - \sum_{j \neq i} |\widehat{S}_{ij}|  \gtrsim D_\Sigma , 
\end{align*} 
which implies that $\widehat{S}$ is also diagonally dominant. 
Thus, using Theorem 1, \cite{varah1975lower}, when $n \gg \tfrac{1}{{D_\Sigma}^2} \log p$, we have
\begin{align}
\|\widehat{S}\|_\infty \lesssim \frac{1}{D_\Sigma} , \label{eq:proof:thm:lasso-mod:de-bias:stat-inf:tmp7}
\end{align}
with probability at least $1 - p^{-\Theta(1)}$

And using \eqref{eq:proof:thm:lasso-mod:de-bias:stat-inf:tmp6} and \eqref{eq:proof:thm:lasso-mod:de-bias:stat-inf:tmp7} in \eqref{eq:proof:thm:lasso-mod:de-bias:stat-inf:tmp5:I-S_inv_cov}, we have 
\begin{align}
\max_{1 \leq j,k \leq p} \left|\left(I - \widehat{S}^{-1} (\frac{1}{n}\sum_{i=1}^n x_i x_i^\top\right)_{jk}\right| 
	\lesssim \frac{1}{D_\Sigma} \sqrt{\frac{\log p}{n}} . \label{eq:proof:thm:lasso-mod:de-bias:stat-inf:tmp7:bias:mat-mult:infty}
\end{align}

Using \eqref{eq:proof:thm:lasso-mod:de-bias:stat-inf:tmp7:bias:mat-mult:infty} and the bound on $\|\htheta - \theta^{\star}\|_1$ \eqref{eq:thm:lasso-mod-cov-bounds:l1}, 
in \eqref{eq:proof:thm:lasso-mod:de-bias:stat-inf:tmp4:decomposed:de-bias:diff}, we have
\begin{align}
\left\| \left(I - \widehat{S}^{-1} \left(\frac{1}{n}\sum_{i=1}^n x_i x_i^\top\right) \right) (\htheta - \theta^{\star}) \right\|_\infty 
	\lesssim \frac{1}{D_\Sigma} s \left(\sigma + \|\theta^{\star}\|_1\right) \frac{\log p}{n} \lesssim \frac{1}{D_\Sigma} s \left(\sigma + \sqrt{s}\right) \frac{\log p}{n}  . \label{eq:proof:thm:lasso-mod:de-bias:stat-inf:tmp8:bias:l_infty-bound}
\end{align}

Combining \eqref{eq:proof:thm:lasso-mod:de-bias:stat-inf:tmp8:bias:l_infty-bound} and \eqref{eq:proof:thm:lasso-mod:de-bias:stat-inf:tmp4:decomposed:de-bias:diff}, 
we have proven \Cref{thm:lasso-mod:de-bias:stat-inf}, 
when $n \gg  \max\{ b^2,\tfrac{1}{{D_\Sigma}^2} \} \log p $, we have
\begin{align*}
\sqrt{n} (\htheta^\diff - \theta^{\star}) = Z + R,
\end{align*} 
where $Z$ $\mid$ $\{x_i\}_{i=1}^n$ $\sim$ $\Norm\left(0,  \sigma^2 \widehat{S}^{-1} \left( \frac{1}{n} \sum_{i=1}^n x_i x_i^\top \right) \widehat{S}^{-1} \right)$, 
and $\|R\|_\infty \lesssim \frac{1}{D_\Sigma} s \left(\sigma + \|\theta^{\star}\|_1\right) \frac{\log p}{\sqrt{n}} \lesssim \frac{1}{D_\Sigma} s \left(\sigma + \sqrt{s}\right) \frac{\log p}{\sqrt{n}} $ with probability at least $1-p^{-\Theta(1)}$.

\subsection{Proof of \Cref{thm:alg:stat-inf:high-dim:linear:proximal:svrg}}

We analyze the optimization problem conditioned on the data set $\{x_i\}_{i=1}^n$, which satisfies  \Cref{lem:proof:lasso:linear:cov:soft-thresh:guarantees} with probability at least $1 - p^{\Theta(-1)}$ when $n \gg b^2 \log p$. 

Here, we denote the objective function as 
\begin{align*}
P(\theta) = \frac{1}{2} \theta^\top \left(\widehat{S}- \frac{1}{n} \sum_{i=1}^n x_i x_i^\top\right) \theta 
	+ \frac{1}{n} \sum_{i=1}^n \frac{1}{2} \left(x_i^\top \theta -y_i\right)^2 + \lambda \ltynorm{\theta}{1} . 
\end{align*}

In \Cref{alg:stat-inf:high-dim:linear:proximal:svrg}, 
lines \ref{line-alg:stat-inf:high-dim:linear:proximal:svrg:newton:start} to \ref{line-alg:stat-inf:high-dim:linear:proximal:svrg:newton:end} are using SVRG \cite{johnson2013accelerating} to solve the  Newton step
\begin{align}
\min_\Delta \frac{1}{2} \Delta^\top \widehat{S} \Delta + \left\langle \frac{1}{S_o} \sum_{k \in I_o }  \nabla f_k(\theta_t ) , \Delta \right\rangle , \label{eq:proof:thm:alg:stat-inf:high-dim:linear:proximal:svrg:stat-inf:newton:obj}
\end{align}
and using proximal SVRG \cite{xiao2014proximal} to solve the proximal Newton step
\begin{align}
\min_\Delta \frac{1}{2} \Delta^\top \widehat{S} \Delta + \left\langle \frac{1}{n} \sum_{k =1}^n  \nabla f_k(\theta_t ) , \Delta \right\rangle + \lambda \|\theta + \Delta\|_1 .
	\label{eq:proof:thm:alg:stat-inf:high-dim:linear:proximal:svrg:opt:proximal-newton:obj}
\end{align}

The gradient of \eqref{eq:proof:thm:alg:stat-inf:high-dim:linear:proximal:svrg:stat-inf:newton:obj} is 
\begin{align*}
\widehat{S} \Delta + \frac{1}{S_o} \sum_{k \in I_o }  \nabla f_k(\theta_t ) 
	=  \underbrace{ \frac{1}{p}\sum_{k=1}^p \left[ p \widehat{S}_k \right] \Delta(k) }_{\text{sample \emph{by feature} in SVRG}} + \underbrace{ \frac{1}{S_o} \sum_{k \in I_o }  \nabla f_k(\theta_t ) }_{\text{compute exactly in SVRG}},
\end{align*}
where $\widehat{S}_k$ is the $k$\textsuperscript{th} column of $\widehat{S}$ and $\Delta(k)$ is the $k$\textsuperscript{th} coordinate of $\Delta$.

Line \ref{line-alg:stat-inf:high-dim:linear:proximal:svrg:newton:stat-inf:1} corresponds to SVRG's outer loop part that computes the full gradient.  
Line \ref{line-alg:stat-inf:high-dim:linear:proximal:svrg:newton:stat-inf:2} corresponds to SVRG's inner loop update. 

By \Cref{lem:proof:lasso:linear:cov:soft-thresh:guarantees}, when $n \gg b^2 \log p$, the $\ell_2$ operator norm of $\| \widehat{S} \|_2 = O(1)$. 
And this implies $\| \widehat{S}^\top \widehat{S} \|_2 = O(1)$. 
Because $\|\widehat{S}_k\|_2^2$ is the $k$\textsuperscript{th} diagonal element of $\widehat{S}^\top \widehat{S}$, we have $\|\widehat{S}_k\|_2^2 = O(1)$ for all $1\leq k \leq p$. 
Thus, each $\left[ p \widehat{S}_k \right] \Delta(k) $ is a $O(p)$-Lipschitz function. 

By  Theorem 6.5 of \cite{bubeck2015convex}, 
when conditioned on $\theta_t$, 
and choosing
\begin{align*}
\tau &= \Theta\left(\tfrac{1}{p}\right), \\
L_i & \gtrsim p , 
\end{align*}
after $L_o^t$ SVRG outer steps, we have 
\begin{align*}
\Exp\left[ \left\| \left. \bar{g}_t  + \widehat{S}^{-1} \left( \frac{1}{S_o} \sum_{k \in I_o }  \nabla f_k(\theta_t ) \right) \right\|_2^2 \, \right| \, \theta_t , \{x_i\}_{i=1}^n \right] \lesssim & 0.9^{L_o^t} \left\| \frac{1}{S_o} \sum_{k \in I_o }  \nabla f_k(\theta_t ) \right\|_2^2 \nonumber \\
\lesssim & 0.9^{L_o^t} (1+ \|\theta_t - \htheta\|_2) , 
\end{align*}
where $\bar{g}_t = \frac{1}{L_o^t} \sum_{j=0}^{L_o^t} g_t^j$. 

The gradient of the smooth component $\frac{1}{2} \Delta^\top \widehat{S} \Delta + \left\langle \frac{1}{n} \sum_{k =1}^n  \nabla f_k(\theta_t ) , \Delta \right\rangle $ in  \eqref{eq:proof:thm:alg:stat-inf:high-dim:linear:proximal:svrg:opt:proximal-newton:obj} is 
\begin{align*}
\widehat{S} \Delta + \frac{1}{n} \sum_{k =1}^n  \nabla f_k(\theta_t )  = \underbrace{ \frac{1}{p}\sum_{k=1}^p \left[ p \widehat{S}_k \right] \Delta(k) }_{\text{sample \emph{by feature} in proximal SVRG}} + \underbrace{ \frac{1}{n} \sum_{k =1 }^n  \nabla f_k(\theta_t ) }_{\text{compute exactly in proximal SVRG}} .  
\end{align*}

Line \ref{line-alg:stat-inf:high-dim:linear:proximal:svrg:newton:opt:1} corresponds to proximal SVRG's outer loop part that computes the full gradient. 
Line \ref{line-alg:stat-inf:high-dim:linear:proximal:svrg:newton:opt:2} corresponds to proximal SVRG's inner loop update. 

By Theorem 3.1 of \cite{xiao2014proximal},
 when conditioned on $\theta_t$, 
 and choosing 
 \begin{align*}
 \eta &= \Theta\left(\tfrac{1}{p}\right) , \\
 L_i & \gtrsim p , 
 \end{align*}
after $L_o^t$ proximal SVRG outer steps, we have 
\begin{align*}
\Exp[ P(\theta_{t+1} - P(\htheta))\mid \theta_t ] =& \Exp \left[  \left. P(\theta_t + \bar{d}_t - \htheta) - P(\htheta) \, \right| \, \theta_t , \{x_i\}_{i=1}^n \right] \nonumber  \\ 
\lesssim & 0.9^{L_o^t}(P(\theta_t) - P(\htheta)) , 
\end{align*}
where $\bar{d}_t = \frac{1}{L_o^t} \sum_{j=0}^{L_o^t} d_t^j$. 
And this implies 
\begin{align*}
\Exp[\|\theta_t -\htheta\|_2^2] \lesssim 0.9^{\sum_{i=0}^{t-1} L_o^t} (P(\theta_0) - P(\htheta)) . 
\end{align*}

At each $\theta_t$, we have
\begin{align*}
x_i  ( x_i^\top \theta_t - y_i) = x_i x_i^\top (\theta_t - \htheta) + x_i (x_i^\top \htheta - y_i) . 
\end{align*}

For the first term, we have 
\begin{align*}
\| x_i x_i^\top (\theta_t - \htheta) \|_\infty  \leq & |x_i^\top (\theta_t - \htheta)| \| x_i\|_\infty \nonumber \\ 
\leq & \|x_i\|_2 \|\theta_t - \htheta\|_2 \| x_i\|_\infty  \nonumber \\ 
\leq & \sqrt{p} \| x_i\|_\infty^2 \|\theta_t - \htheta\|_2  ,  
\end{align*}
which implies that 
\begin{align*}
\max_{1 \leq j,k \leq p} \left| \left[ \left(  x_i x_i^\top (\theta_t - \htheta)  \right) \left(  x_i x_i^\top (\theta_t - \htheta)  \right)^\top  \right]_{jk} \right| 
	\leq & \| x_i x_i^\top (\theta_t - \htheta) \|_\infty^2 \nonumber \\
\leq & p \| x_i\|_\infty^4 \|\theta_t - \htheta\|_2^2 . 
\end{align*}

For the second term, we have
\begin{align*}
\|x_i (x_i^\top \htheta - y_i)\|_\infty 
	\leq & \|x_i x_i^\top (\htheta - \theta^\star)\|_\infty + \| x_i \epsilon_i \|_\infty \nonumber \\
\leq & \|x_i\|_\infty^2 \|\htheta - \theta^\star \|_1 + |\epsilon_i| \|x_i\|_\infty
\end{align*}

Because when $n \gg \log p$, from \eqref{eq:proof:lem:sec:high-dim:lasso:linear-regression:plug-in:sandwich:tmp2} we have  with probability at least $1 - p^{-\Theta(1)}$ 
\begin{align*}
\max_{1 \leq i \leq n} \| x_i \|_\infty \lesssim \sqrt{\log p + \log n} , 
\end{align*} 
and  from \eqref{eq:proof:lem:sec:high-dim:lasso:linear-regression:plug-in:sandwich:tmp4:bound-2-1} we have with probability at least $1 - n^{-\Theta(1)}$ 
\begin{align*}
\max_{1\leq i \leq n} |\epsilon_i| \lesssim \sigma \sqrt{\log n} , 
\end{align*}
when conditioned on $\theta_t$ (and the data set $\{x_i\}_{i=1}^n$) we have 
\begin{align*}
& \max_{1\leq j,k \leq p} \left| \left[ \left(\widehat{S}^{-1} g_t^0\right) \left(\widehat{S}^{-1} g_t^0\right)^\top - \left( \widehat{S}^{-1}  \tfrac{1}{S_o} \tsum_{k \in I_o }  \nabla f_k(\theta_t )   \right) \left( \widehat{S}^{-1}  \tfrac{1}{S_o} \tsum_{k \in I_o }  \nabla f_k(\theta_t )  \right)^\top \right] _{jk} \right| \nonumber \\ 
\lesssim &  \frac{1}{{D_\Sigma}^2} (\| x_i x_i^\top (\theta_t - \htheta) \|_\infty^2 + 2 \| x_i x_i^\top (\theta_t - \htheta) \|_\infty \|x_i (x_i^\top \htheta - y_i)\|_\infty  ) \nonumber \\
\lesssim & \frac{1}{{D_\Sigma}^2} (p (\log p + \log  n)^2 \|\theta_t - \htheta\|_2^2 + \sqrt{p} (\log p + \log n)  \|\theta_t - \htheta\|_2 ((\log p + \log n)\| \htheta - \theta^\star \|_1+\sigma \sqrt{(\log p + \log n) \log n }) ) \nonumber \\ 
\lesssim & \frac{1}{{D_\Sigma}^2}  (p  \|\theta_t - \htheta\|_2^2 + \sqrt{p}   \|\theta - \htheta\|_2  (\sigma + \|\htheta - \theta^\star\|_1) ) \polylog(p, n)
\end{align*}
under the events of \eqref{eq:proof:lem:sec:high-dim:lasso:linear-regression:plug-in:sandwich:tmp2}, \eqref{eq:proof:thm:lasso-mod:de-bias:stat-inf:tmp7} , and \eqref{eq:proof:lem:sec:high-dim:lasso:linear-regression:plug-in:sandwich:tmp4:bound-2-1}, where we used the fact  \eqref{eq:proof:thm:lasso-mod:de-bias:stat-inf:tmp7} that the $\ell_\infty$ operator norm $ \| {\widehat{S}}^{-1} \|_\infty \lesssim \frac{1}{D_\Sigma}$ with probability at least $1-p^{-\Theta(1)}$ when $n \gg \max\{b^2 , \frac{1}{{D_\Sigma}^2}\} \log p $. 

Thus, we can conclude that, conditioned on the data set  $\{x_i\}_{i=1}^n$, 
and the events \eqref{eq:proof:lem:sec:high-dim:lasso:linear-regression:plug-in:sandwich:tmp2},  \eqref{eq:proof:lem:sec:high-dim:lasso:linear-regression:plug-in:sandwich:tmp4:bound-2-1}, and \eqref{eq:proof:thm:lasso-mod:de-bias:stat-inf:tmp7}, we have 
we  have an asymptotic normality result 
\begin{align*}
\tfrac{1}{\sqrt{t}}\left( \tsum_{t=1}^T \sqrt{S_o} \bar{g_t} + \tfrac{1}{n}\tsum_{i=1}^n x_i (x_i^\top \htheta - y_i)  \right) = W + R,
\end{align*}
where $W$ weakly converges to $\scriptstyle \Norm\left(0, \widehat{S}^{-1} \left[ \frac{1}{n}\sum_{i=1}^n (x_i^\top \htheta - y_i)^2 x_i x_i^\top  - \left( \frac{1}{n} \sum_{i=1}^n x_i (x_i^\top \htheta - y_i) \right) \left( \frac{1}{n} \sum_{i=1}^n x_i (x_i^\top \htheta - y_i) \right)^\top \right] \widehat{S}^{-1}  \right)$, 
and 
\begin{align*}
\|R\|_\infty 
\leq & \frac{1}{\sqrt{t}}\sum_{t=1}^T \left( \| \bar{g}_t - \widehat{S}^{-1} g_t^0 \|_\infty + \| \widehat{S}^{-1} g_t^0 - \tfrac{1}{S_o} \tsum_{k \in I_o} \nabla f_k(\htheta)\|_\infty \right) \nonumber \\
\leq & \frac{1}{\sqrt{t}}\sum_{t=1}^T \left( \| \bar{g}_t - \widehat{S}^{-1} g_t^0 \|_2 + \| \widehat{S}^{-1} g_t^0 - \tfrac{1}{S_o} \tsum_{k \in I_o} \nabla f_k(\htheta)\|_\infty \right) , 
\end{align*} 
which implies 
\begin{align*}
& \Exp\left[ \|R\|_\infty \left. \, \right| \, \{x_i\}_{i=1}^n, \eqref{eq:proof:lem:sec:high-dim:lasso:linear-regression:plug-in:sandwich:tmp2},  \eqref{eq:proof:lem:sec:high-dim:lasso:linear-regression:plug-in:sandwich:tmp4:bound-2-1},  \eqref{eq:proof:thm:lasso-mod:de-bias:stat-inf:tmp7} \right] \nonumber \\ 
\lesssim &  \Exp\left[   \frac{1}{\sqrt{t}}\sum_{t=1}^T 0.95^{L_o^t} (1+ \|\theta_t - \htheta\|_2) + \sqrt{p}(\log p + \log n)  \|\theta_t - \htheta\|_2   \left. \, \right| \, \{x_i\}_{i=1}^n, \eqref{eq:proof:lem:sec:high-dim:lasso:linear-regression:plug-in:sandwich:tmp2},  \eqref{eq:proof:lem:sec:high-dim:lasso:linear-regression:plug-in:sandwich:tmp4:bound-2-1},  \eqref{eq:proof:thm:lasso-mod:de-bias:stat-inf:tmp7} \right] \nonumber \\
\lesssim &  \frac{1}{\sqrt{T}}\sum_{t=1}^T 0.95^{L_o^t} (1+\sqrt{P(\theta_0)-P(\htheta)} 0.95^{\sum_{i=0}^{t-1} L_o^t}) + \sqrt{p}(\log p + \log n) \sqrt{P(\theta_0)-P(\htheta)} 0.95^{\sum_{i=0}^{t-1} L_o^t} .
\end{align*}

And, because $\left( \frac{1}{S_o} \sum_{k\in I_o} \nabla f_k(\htheta) \right)_t$ are i.i.d., and bounded when conditioned on the data set  $\{x_i\}_{i=1}^n$, 
and the events \eqref{eq:proof:lem:sec:high-dim:lasso:linear-regression:plug-in:sandwich:tmp2},  \eqref{eq:proof:lem:sec:high-dim:lasso:linear-regression:plug-in:sandwich:tmp4:bound-2-1}, and \eqref{eq:proof:thm:lasso-mod:de-bias:stat-inf:tmp7}, using a union bound over all matrix entries, and sub-Gaussian concentration inequalities \cite{Wainwright2015concentration-lecture} similar to \Cref{lem:sec:high-dim:lasso:linear-regression:plug-in:sandwich}'s proof, 
when $T \gg \left( (\log p + \log n) \|\htheta - \theta^\star\|_1 + \sigma \sqrt{(\log p + \log n) \log n} \right) \log p $, we also have 
\begin{align*}
&\left\| \tfrac{S_o}{T} \tsum_{t=1}^T \bar{g}_t \bar{g}_t^\top - {\widehat{S}}^{-1} \left( \tfrac{1}{n} \tsum_{i=1}^n (x_i^\top \htheta - y_i)^2 x_i x_i^\top \right) {\widehat{S}}^{-1}  \right\|_{\max } \nonumber \\
\lesssim & \sqrt{  \left( (\log p + \log n) \|\htheta - \theta^\star\|_1 + \sigma \sqrt{(\log p + \log n) \log n} \right) \tfrac{\log p}{T} } \nonumber \\
	& + \tfrac{1}{u} \left[ \tfrac{1}{\sqrt{T}}\tsum_{t=1}^T 0.95^{L_o^t} (1+ \sqrt{P(\theta_0)-P(\htheta)} 0.95^{\sum_{i=0}^{t-1} L_o^t}) + \sqrt{p}(\log p + \log n)  \sqrt{P(\theta_0)-P(\htheta)} 0.95^{\sum_{i=0}^{t-1} L_o^t} \right]  , 
\end{align*}
with probability at least $1-p^{-\Theta(-1)} - u$, where we used Markov inequality for the remainder term.

\subsection{Proof of \Cref{lem:sec:high-dim:lasso:linear-regression:plug-in:sandwich}}

We analyze the optimization problem conditioned on the data set $\{x_i\}_{i=1}^n$, which satisfies  \Cref{lem:proof:lasso:linear:cov:soft-thresh:guarantees} with probability at least $1 - p^{\Theta(-1)}$ when $n \gg b^2 \log p$.

Because we have 
\begin{align*}
& (x_i^\top \htheta - y_i)^2 \nonumber \\
= & (x_i^\top(\htheta - \theta^{\star}) - \epsilon_i)^2 \nonumber \\
= & \epsilon_i^2 - 2 \epsilon_i x_i^\top(\htheta - \theta^{\star}) + (x_i^\top(\htheta - \theta^{\star}))^2 , 
\end{align*}
we can write 
\begin{align}
& \sigma^2 \tfrac{1}{n} \tsum_{i=1}^n x_i x_i^\top - \tfrac{1}{n} \tsum_{i=1}^n (x_i^\top \htheta - y_i)^2 x_i x_i^\top  \nonumber \\
=&  \tfrac{1}{n} \tsum_{i=1}^n (\sigma^2 - \epsilon_i^2) x_i x_i^\top 
	 + \tfrac{1}{n} \tsum_{i=1}^n ( 2 \epsilon_i x_i^\top(\htheta - \theta^{\star}) - (x_i^\top(\htheta - \theta^{\star}))^2) x_i x_i^\top  . \label{eq:proof:lem:sec:high-dim:lasso:linear-regression:plug-in:sandwich:tmp6:decompose}
\end{align}

Conditioned on $\{x_i\}_{i=1}^n$, because $\epsilon_i \sim \Norm(0, \sigma^2)$ are i.i.d., and $\epsilon_i^2$ is sub-exponential, using Bernstein inequality \cite{Wainwright2015concentration-lecture}, we have 
\begin{align}
& \Pr \left[ \left| \tfrac{1}{n} \tsum_{i=1}^n \left(1 - \tfrac{\epsilon_i^2}{\sigma^2}\right) x_i(j) x_i(k) \right| > t \mid \{x_i\}_{i=1}^n \right] \nonumber \\
\lesssim & \exp\left(-n \min\left\{ \tfrac{t}{\max_{1\leq i \leq n} |x_i(j) x_i(k)|} , \left(\tfrac{t}{\max_{1\leq i \leq n} |x_i(j) x_i(k)|}\right)^2\right\}\right) , \label{eq:proof:lem:sec:high-dim:lasso:linear-regression:plug-in:sandwich:tmp1}
\end{align}
for $1 \leq j,k \leq p$, where $x_i(j)$ is the $j$\textsuperscript{th} coordinate of $x_i$. 

Because each $x_i(j)$ is $\Norm(0, \Theta(1))$ by our assumptions, using a union bound over all samples' coordinates we have 
\begin{align}
\max_{\substack{1 \leq i \leq n \\ 1 \leq j \leq p}} |x_i(j)| \lesssim \sqrt{\log p + \log n} , \label{eq:proof:lem:sec:high-dim:lasso:linear-regression:plug-in:sandwich:tmp2}
\end{align}
with probability at least $1 - (pn)^{-\Theta(1)}$ .

Combining \eqref{eq:proof:lem:sec:high-dim:lasso:linear-regression:plug-in:sandwich:tmp1} and \eqref{eq:proof:lem:sec:high-dim:lasso:linear-regression:plug-in:sandwich:tmp2}, 
and taking a union bound over all entries of the matrix $\tfrac{1}{n} \tsum_{i=1}^n (\sigma^2 - \epsilon_i^2) x_i x_i^\top $, 
when $n \gg \log p $, we have 
\begin{align}
\max_{1 \leq j, k \leq p } |(\tfrac{1}{n} \tsum_{i=1}^n (\sigma^2 - \epsilon_i^2) x_i x_i^\top)|_{jk} \lesssim \sigma^2 (\log p + \log n) \sqrt{\frac{\log p}{n}} , \label{eq:proof:lem:sec:high-dim:lasso:linear-regression:plug-in:sandwich:tmp3:bound-1}
\end{align}
with probability at least $(1 - (pn)^{-\Theta(1)})(1 - p^{-\Theta(1)}) = 1 - (pn)^{-\Theta(1)} - p^{-\Theta(1)}$. 

Because $\epsilon_i \sim \Norm(0, \sigma^2)$, by a union bound, we have 
\begin{align}
\max_{1 \leq i \leq n} |\epsilon_i| \lesssim \sigma \sqrt{\log n } , \label{eq:proof:lem:sec:high-dim:lasso:linear-regression:plug-in:sandwich:tmp4:bound-2-1}
\end{align}
with probability at least $1 - n^{-\Theta(1)}$. 

Using \eqref{eq:proof:lem:sec:high-dim:lasso:linear-regression:plug-in:sandwich:tmp2}, we have  
\begin{align}
& \max_{1 \leq i \leq n } |x_i^\top(\htheta - \theta^{\star})|  \nonumber \\
\leq & \| \htheta - \theta^\star \|_1 \max_{1 \leq i \leq n } \max_{1 \leq j \leq p} |x_i(j)| \nonumber \\ 
\lesssim & s \left(\sigma + \|\theta^{\star}\|_1\right) \sqrt{\tfrac{\log p}{n} (\log p + \log n)} \lesssim  s \left( \sigma + \sqrt{s} \right)\sqrt{\tfrac{\log p}{n} (\log p + \log n)} ,  \label{eq:proof:lem:sec:high-dim:lasso:linear-regression:plug-in:sandwich:tmp5:bound-2-2}
\end{align}
with probability at least $1- p^{-\Theta(1)} - (pn)^{-\Theta(1)}$.

Combining \eqref{eq:proof:lem:sec:high-dim:lasso:linear-regression:plug-in:sandwich:tmp2}, \eqref{eq:proof:lem:sec:high-dim:lasso:linear-regression:plug-in:sandwich:tmp3:bound-1}, \eqref{eq:proof:lem:sec:high-dim:lasso:linear-regression:plug-in:sandwich:tmp4:bound-2-1}, \eqref{eq:proof:lem:sec:high-dim:lasso:linear-regression:plug-in:sandwich:tmp5:bound-2-2}, 
and using \eqref{eq:proof:lem:sec:high-dim:lasso:linear-regression:plug-in:sandwich:tmp6:decompose}, 
when $n \gg \log p$, we have
\begin{align}
& \max_{1 \leq j, k \leq p} \left| \left(  \tfrac{1}{n} \tsum_{i=1}^n (x_i^\top \htheta - y_i)^2 x_i x_i^\top  -  \sigma^2 \tfrac{1}{n} \tsum_{i=1}^n x_i x_i^\top  \right)_{jk}  \right| \nonumber \\ 
\lesssim &  \sigma^2 (\log p + \log n) \sqrt{\tfrac{\log p}{n}} + \sigma s \left(\sigma + \|\theta^{\star}\|_1\right) (\log p + \log n)^{\frac{3}{2}} \sqrt{\tfrac{\log p \cdot \log n}{n} } \nonumber \\
	& + s^2 \left( \sigma +\|\theta^{\star}\|_1 \right)^2  (\log p + \log n)^2 \tfrac{\log p}{n} , \label{eq:proof:lem:sec:high-dim:lasso:linear-regression:plug-in:sandwich:tmp6}
\end{align}
with probability at least $1 - p^{-\Theta(1)} - n^{-\Theta(1)}$. 

Combining \eqref{eq:proof:lem:sec:high-dim:lasso:linear-regression:plug-in:sandwich:tmp6} and \eqref{eq:proof:thm:lasso-mod:de-bias:stat-inf:tmp7}, when $n \gg \max\{b^2,  \tfrac{1}{{D_\Sigma}^2}\} \log p$, we have
\begin{align*}
& \max_{1 \leq j, k \leq p} \left| \left( \widehat{S}^{-1} \left( \tfrac{1}{n} \tsum_{i=1}^n (x_i^\top \htheta - y_i)^2 x_i x_i^\top \right) \widehat{S}^{-1} -  \sigma^2 \widehat{S}^{-1} \left( \tfrac{1}{n} \tsum_{i=1}^n x_i x_i^\top \right) \widehat{S}^{-1} \right)_{jk}  \right| \nonumber \\ 
\lesssim & \tfrac{1}{{D_\Sigma}^2} \left( \sigma^2   + \sigma s \left(\sigma + \|\theta^{\star}\|_1\right) \sqrt{\log p + \log n} \sqrt{ \log n} + s^2 \left( \sigma +\|\theta^{\star}\|_1 \right)^2  (\log p + \log n) \sqrt{\tfrac{\log p}{n}} \right) (\log p + \log n) \sqrt{\tfrac{\log p}{n}} , 
\end{align*}
with probability at least $1 - p^{-\Theta(1)} - n^{-\Theta(1)}$.

%
%

%
%
%
%
%

\section{Technical lemmas}

\subsection{\Cref{lem:bubeck:strongly-convex-bound}}

Next lemma is a well known property of convex functions (Lemma 3.11  of \cite{bubeck2015convex}).
\begin{lemma}
\label{lem:bubeck:strongly-convex-bound}
For a $\alpha$ strongly convex and $\beta$ smooth function $F(x)$, 
we have 
\begin{align*}
\langle \nabla F(x_1) - \nabla F(x_2) , x_1 - x_2 \rangle 
	\geq& \frac{\alpha \beta}{\alpha+ \beta}\| x_1 - x_2 \|_2^2 + \frac{1}{\beta +\alpha} \| \nabla F(x_1) - \nabla F(x_2) \|_2^2  \nonumber \\
\geq & \frac{1}{2} \alpha \| x_1 - x_2 \|_2^2 + \frac{1}{2 \beta} \| \nabla F(x_1) - \nabla F(x_2) \|_2^2 . 
\end{align*}
\end{lemma}

\subsection{\Cref{lem:geometric-like-series-bound}}

Next lemma provides a bound on a geometric-like sequence.
\begin{lemma}
\label{lem:geometric-like-series-bound}

Suppose we have a sequence 
\begin{align*}
a_{t+1} = (1 - \kappa t^{-d}) a_t + C t^{-pd} ,
\end{align*}
where $a_1 \geq 0$, $ 0  < \kappa < 1$, $p\geq 2$ and $d \in (\frac{1}{2}, 1)  $ is the decaying rate.

Then, $\forall 1 \leq s \leq t$ we have
\begin{align*}
a_t &\leq C \frac{1}{pd-1} (1 - t^{1-pd} )\exp\left(-\kappa \frac{1}{1-d} \left((t+1)^{1-d} - (s+1)^{1-d}\right)\right)
	+ a_1 s^{-(p-1)d} \frac{1}{\kappa} . 
\end{align*}

When we assume that $a_1, C, \kappa, p, d$ are all constants, we have
\begin{align*}
a_t = O(t^{-(p-1)d}) . 
\end{align*}

\end{lemma}

\begin{proof}

Unrolling the recursion, we have

\begin{align*}
a_t = C \underbrace{ \sum_{i=1}^{t-1} ( \prod_{j=i+1}^{t-1} (1-\kappa j^{-d}) ) i^{-pd} }_{\mytag{[1]}{eq:lem:geometric-like-series-bound:tmp1}}+ a_1  \underbrace{ \prod_{i=1}^{t-1} (1 - \kappa i^{-d}) }_{\mytag{[2]}{eq:lem:geometric-like-series-bound:tmp2}}  .
\end{align*}

Splitting term \ref{eq:lem:geometric-like-series-bound:tmp1} into two parts, we have 
\begin{align*}
& \sum_{i=1}^{t-1} \left( \prod_{j=i+1}^{t-1} (1-\kappa j^{-d}) \right) i^{-pd} \nonumber \\
=  & \sum_{i=1}^{s-1} \left( \prod_{j=i+1}^{t-1} (1-\kappa j^{-d}) \right) i^{-pd} + \sum_{i=s}^{t-1} \left( \prod_{j=i+1}^{t-1} (1-\kappa j^{-d}) \right) i^{-pd} . 
\end{align*}

For the first part, 
we have 
\begin{align*}
& \sum_{i=1}^{s-1} ( \prod_{j=i+1}^{t-1} (1-\kappa j^{-d}) ) i^{-pd}  \nonumber \\
\leq & \left( \prod_{j=s}^{t-1}   (1-\kappa j^{-d} ) \right) \sum_{i=1}^{s-1}  i^{-pd} \nonumber \\
\leq & \frac{1}{pd-1} (1 - t^{1-pd} )\exp\left(-\kappa \frac{1}{1-d} ((t+1)^{1-d} - (s+1)^{1-d})\right)
\end{align*}
where we used  
\begin{align*}
&\sum_{i=r}^s i^{-pd} \nonumber \\
\leq & \int_r^{s+1} u^{-pd}  \, du \nonumber \\
\leq & \frac{1}{pd-1} (r^{1-pd} -(s+1)^{1-pd}) . 
\end{align*}

For term \ref{eq:lem:geometric-like-series-bound:tmp2}, notice that for $1 \leq r \leq s$, using $1 - x \leq \exp(-x)$ when $x \in [0, 1]$, we have
\begin{align*}
\prod_{i=r}^s (1-\kappa i^{-d}) \leq \exp(-\kappa \tsum_{i=r}^s i^{-d}) ,
\end{align*}
and using the fact that 
\begin{align*}
\sum_{i=r}^s i^{-d} & \geq \int_{r}^{s+1} (u+1)^{-d} \,du \nonumber \\
& =  \frac{1}{1-d} \left((s+2)^{1-d} - (r+1)^{1-d}\right), 
\end{align*}
we have 
\begin{align*}
\prod_{i=1}^{t-1} (1 - \kappa i^{-d}) \leq& \exp\left(-\kappa \tfrac{1}{1-d} (t^{1-d} - 2^{1-d})\right) . 
\end{align*}

For the second part, 
we have
\begin{align*}
& \sum_{i=s}^{t-1} \left( \prod_{j=i+1}^{t-1} (1-\kappa j^{-d}) \right) i^{-pd} \nonumber \\
\leq & s^{-(p-1)d} \sum_{i=s}^{t-1} \left( \prod_{j=i+1}^{t-1} (1-\kappa j^{-d}) \right) i^{-d} \nonumber \\
= & s^{-(p-1)d} \frac{1}{\kappa}  \sum_{i=s}^{t-1} \left( \prod_{j=i+1}^{t-1} (1-\kappa j^{-d}) \right) \kappa i^{-d} \nonumber \\
= & s^{-(p-1)d}  \frac{1}{\kappa} \left(1 - \prod_{i=s}^{t-1} (1-\kappa i^{-d})\right) \nonumber\\ 
\leq& s^{-(p-1)d} \frac{1}{\kappa} ,
\end{align*}
where we used the fact that 
\begin{align*}
& \sum_{i=s}^{t-1} \kappa i^{-d}  \prod_{j=i+1}^{t-1} (1-\kappa j^{-d} ) \nonumber \\
= & 1-\prod_{i=s}^{t-1} (1-\kappa j^{-d} ) \nonumber \\
< & 1 . 
\end{align*}

When we assume that $a_1, C, \kappa, p, d$ are all constants, setting $s=\lfloor \frac{n}{2} \rfloor $, we have
\begin{align*}
a_t = O(t^{-(p-1)d}) .
\end{align*}


\end{proof}

%
%

\newpage

\section{Experiments}

\subsection{Synthetic data}

\subsubsection{Low dimensional problems}
\label{append:subsubsec:exp:sim:low-dim}

Here, we provide the exact configurations for linear/logistic regression examples  provided in \Cref{tab:exp:coverage:linear} and \Cref{tab:exp:coverage:logistic}.

\paragraph{Linear regression.}
We consider the model $y = \left\langle \sfrac{ [1, \cdots , 1]^\top }{\sqrt{10}} , x \right\rangle + \epsilon$, where $x \sim \Norm(0, \Sigma) \in \Real^{10}$ and $\epsilon \sim \Norm(0, 0.7^2)$, 
with  100 i.i.d. data points. 

\textbf{\textit{Lin1:}} We used $\Sigma = I$. 
For \Cref{alg:stat-inf-spnd}, we set $T=100$, $d_o =d_i =\sfrac{2}{3}$, $\rho_0 = 0.1$, $L = 200$, $\tau_0 = 20$, $S_o =S_i= 10$. 
In bootstrap we used 100 replicates.
For averaged SGD, we used 100 averages each of length 50, with step size $0.7 \cdot (t+1)^{-\sfrac{2}{3}}$ and batch size 10. 

\textbf{\textit{Lin2:}} We used $\Sigma_{jk} = 0.4^{|j-k|}$. 
For \Cref{alg:stat-inf-spnd}, we set $T=100$, $d_o =d_i =\sfrac{2}{3}$, $\rho_0 = 0.7$, $L = 100$, $\tau_0 = 1$, $S_o =S_i= 10$. 
In bootstrap we used 100 replicates.
For averaged SGD, we used 100 averages each of length 50, with step size $(t+1)^{-\sfrac{2}{3}}$ and batch size 10. 

\paragraph{Logistic regression.} 
Although logistic regression does not satisfy strong convexity, experimentally \Cref{alg:stat-inf-spnd} still gives valid confidence intervals (\cite{gadat2017optimal} recently has shown that SGD in logistic regression behaves similar to strongly convex problems). 
We consider the model $\Pr[y=1] = \Pr[y=0] = \sfrac{1}{2}$ 
and $x \mid y \sim \Norm(\sfrac{0.1}{\sqrt{10}} \cdot  [1, \cdots, 1]^\top, \Sigma) \in \Real^{10}$, 
with  100 i.i.d. data points.  
Because in bootstrap resampling the Hessian is singular for some replicates, 
we use jackknife and solve each replicate using Newton's method, 
which approximately needs 25 steps per replicate. 

\textbf{\textit{Log1:}} 
We used $\Sigma = I$. 
For \Cref{alg:stat-inf-spnd}, we set $T=50$, $d_o =d_i =\sfrac{2}{3}$, $\rho_0 = 0.1$, $L = 100$, $\tau_0 = 2$, $S_o =S_i= 10$, $\delta_0 = 0.01$.
For averaged SGD, we used 50 averages each of length 100, with step size $2 \cdot (t+1)^{-\sfrac{2}{3}}$ and batch size 10. 

\textbf{\textit{Log2:}}  
We used $\Sigma_{jk} = 0.4^{|j-k|}$. 
For \Cref{alg:stat-inf-spnd}, we set $T=50$, $d_o =d_i =\sfrac{2}{3}$, $\rho_0 = 0.1$, $L = 100$, $\tau_0 = 5$, $S_o =S_i= 10$, $\delta_0 = 0.01$
For averaged SGD, we used 50 averages each of length 100, with step size $ 5 \cdot (t+1)^{-\sfrac{2}{3}}$ and batch size 10.

\paragraph{Calibration.} 
\label{append:subsubsec:exp:sim:low-dim:calibration} 
Here, we give empirical results on calibrating confidence intervals (\cite{efron1994introduction}, Ch.18; \cite{politis2012subsampling}, Ch. 9) produced by our approximate Newton procedure.   
We consider the model $y = \left\langle \sfrac{ [1, \cdots , 1]^\top }{\sqrt{20}} , x \right\rangle + \epsilon$, where $x \sim \Norm(0, \Sigma) \in \Real^{20}$ and $\epsilon \sim \Norm(0, 0.7^2)$, 
with  200 i.i.d. data points. 
We ran 100 simulations.
In each simulation, 
we bootstrapped the dataset 100 times, and computed confidence intervals on each bootstrap replicate using our approximate Newton procedure, bootstrap, and inverse Fisher information. 
For each method, we then used grid search to find a multiplier such that the empirical point estimate is covered by the bootstrap confidence intervals 95\% of the time. 
Average 95\% confidence interval coverage and length after calibration are given in \Cref{tab:exp:coverage:calibrated}. 

\begin{table}[!t]
\rowcolors{2}{white}{black!05!white}
\centering
 \begin{tabular}{  c c c c c   }
 \toprule
   Approximate Newton & & Bootstrap & & Inverse Fisher information  \\
 \cmidrule{1-1} \cmidrule{3-3} \cmidrule{5-5}   
  (0.951, 0.224) & & (0.946  0.205) & & (0.966, 0.212)  \\
 \bottomrule
 \end{tabular}
\caption{Average 95\% confidence interval (coverage, length) after calibration}\label{tab:exp:coverage:calibrated}
\end{table}

\subsubsection{High dimensional linear regression}
\label{append:subsubsec:exp:sim:high-dim-linear}

For comparison with de-biased LASSO \cite{javanmard2015biasing,van2014asymptotically}, 
we use the de-biased LASSO estimator with known covariance (``oracle'' de-biased LASSO estimator)
\begin{align*}
\htheta^\diff_{\mathrm{oracle}} = \htheta_{\mathrm{LASSO}} 
	+ \tfrac{1}{n} \cdot \Sigma^{-1} \left(  \tfrac{1}{n} \sum_{i=1}^n y_i x_i 
	 - \sum_{i=1}^n x_i x_i^\top \htheta_{\mathrm{LASSO}}  \right), 
\end{align*} 
and its corresponding statistical error covariance estimate 
\begin{align*}
\sigma^2 \cdot \Sigma^{-1} \left( \tfrac{1}{n} \sum_{i=1}^n x_i x_i^\top \right) \Sigma^{-1} ,
\end{align*}
which assumes that the true inverse covariance $\Sigma^{-1}$ and observation noise variance $\sigma^2$ are known.

\paragraph{Confidence interval visualization.} 
We use 600 i.i.d. samples from a model with $\Sigma = I$, $\sigma = 0.7$, 
$\theta^\star = [  \sfrac{1}{\sqrt{8}}, \cdots,  \sfrac{1}{\sqrt{8}} , 0, \cdots, 0]^\top \in \Real^{1000}$ which is 8-sparse.  
\Cref{fig:exp:sim:high-dim-linear:CI:plot} shows 95\% confidence intervals for the first 20 coordinates. 
The average confidence interval length is 0.14 and average coverage is 0.83.
Additional experimental results, including p-value distribution   under the null hypothesis, are presented in \Cref{append:subsubsec:exp:sim:high-dim-linear}. 

\begin{figure}
  \begin{center}
    \includegraphics[width=0.5\textwidth]{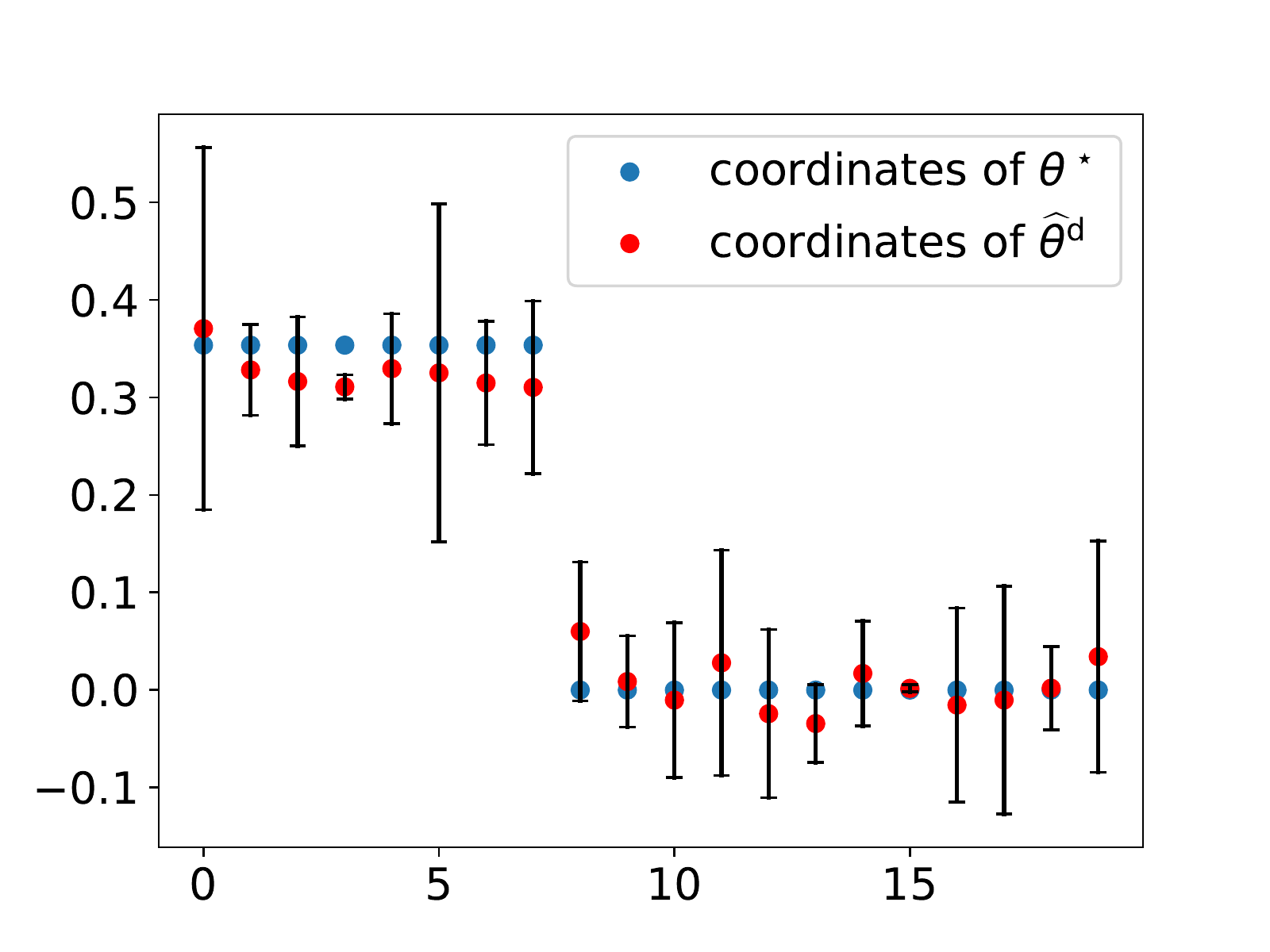}
  \end{center}
  \caption{95\% confidence intervals}
  \label{fig:exp:sim:high-dim-linear:CI:plot}
\end{figure}

\paragraph{Comparison with de-biased LASSO.}
We use 600 i.i.d. samples from a model with $\Sigma = I$, $\sigma = 0.7$, 
$\theta^\star = [  \sfrac{1}{\sqrt{8}}, \cdots,  \sfrac{1}{\sqrt{8}} , 0, \cdots, 0]^\top \in \Real^{1000}$ which is 8-sparse.

For our method, the average confidence interval length is 0.14 and average coverage is 0.83.
For the de-biased LASSO estimator with known covariance, the average confidence interval length is 0.11 and average coverage is 0.98.

\begin{figure}[!h]
\centering
\includegraphics[width=.5\textwidth]{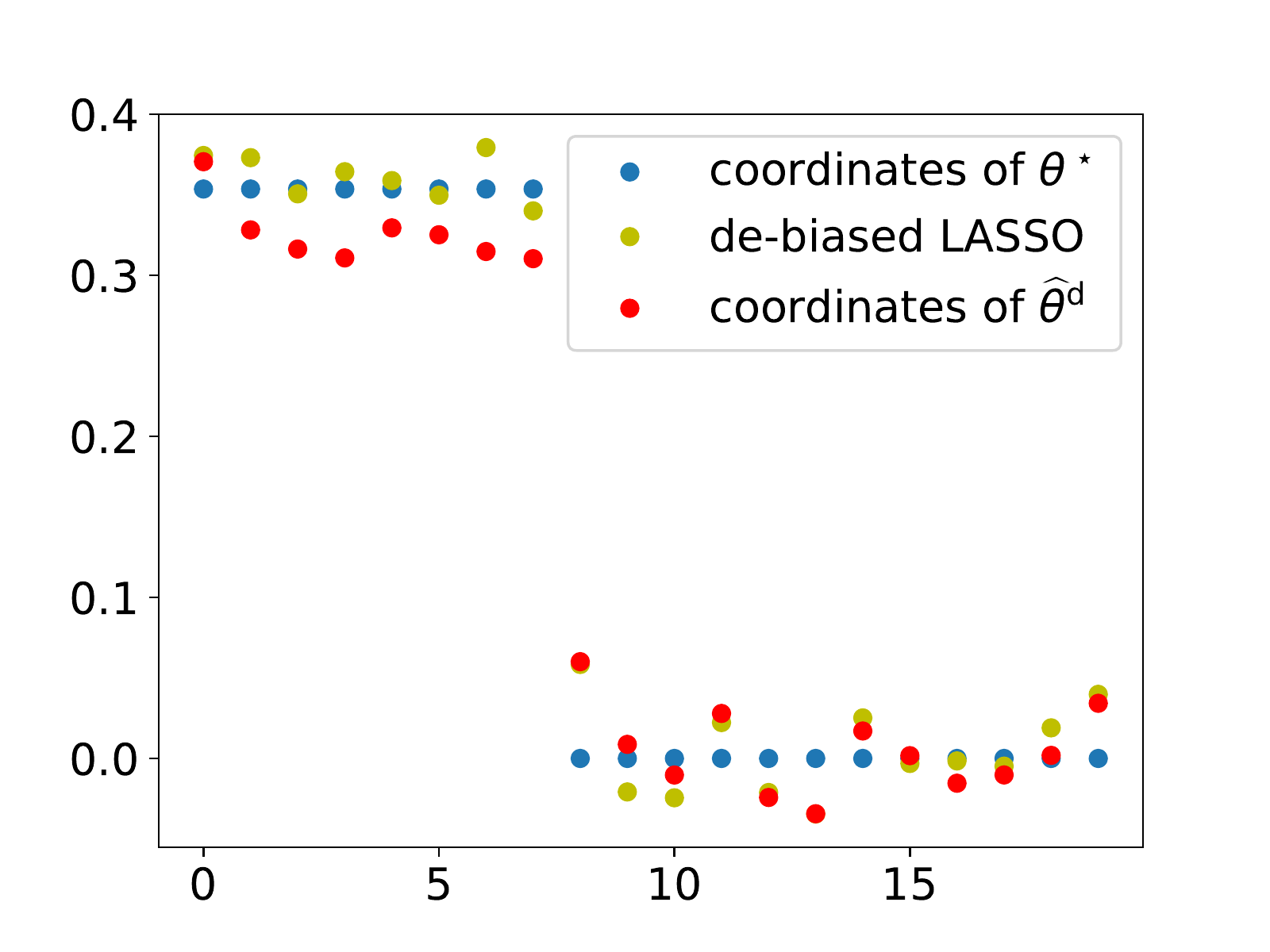}
\caption{Comparison of our de-biased estimator and oracle de-biased LASSO estimator}
\label{fig:exp:high-dim-linear:cmp-lasso-de-biased}
\end{figure}

%
%

\subsubsection{Time series analysis}
\label{append:subsubsec:exp:sim:time-series}

In our linear regression simulation, we generate i.i.d. random explanatory variables, 
and the observation noise is a 0-mean moving average (MA) process independent of the explanatory variables. 

For the linear model 
\begin{align*}
y_i = \langle x_i, \theta^\star \rangle + \epsilon_i ,
\end{align*}
$x_i \in \Real^{20}$ are i.i.d. samples generated from $\Norm \left( [1, 1, \dots, 1]^\top / \sqrt{k}, I \right)$ , 
and $\epsilon_i$ is a 0-mean moving average process
\begin{align*}
\epsilon_i = 0.6 \cdot z_i + 0.8 \cdot z_{i-1} , 
\end{align*}
where $z_i$ are i.i.d. $\Norm(0, 0.7^2)$. 

We ran 100 simulations, with each time series containing 200 samples. 
For our approximate Newton statistical inference procedure (\Cref{alg:stat-inf-spnd:time-series}), 
average 95\% confidence interval (coverage, length) is (0.929, 0.145), and it matches our theory.

\subsection{Real data}

\subsubsection{Neural network adversarial attack detection} 
\label{append:subsubsec:exp:nn:adversarial}

The adversarial perturbation  used in our experiments is shown in \Cref{fig:exp:real:nn:advsarial-pertub}. 
It is generated using the fast gradient sign method \cite{goodfellow2014explaining}
\Cref{fig:exp:real:nn:example:shirt} shows images in a ``Shirt'' example. 
\Cref{fig:exp:real:nn:example:tshirt-top} shows images in a ``T-shirt/top'' example.

\begin{figure}[!h]
\centering
\includegraphics[width=.3\textwidth]{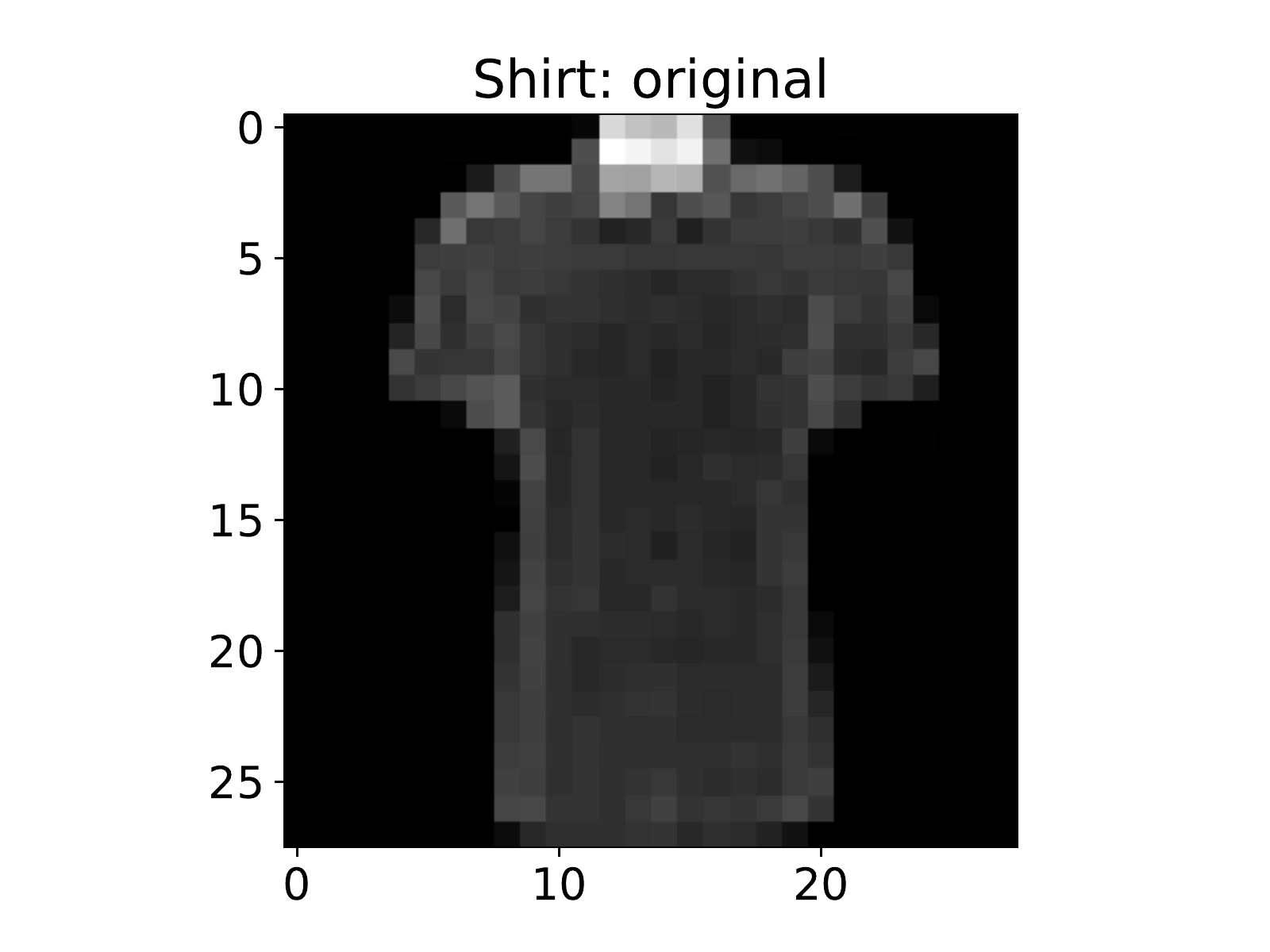}
\includegraphics[width=.3\textwidth]{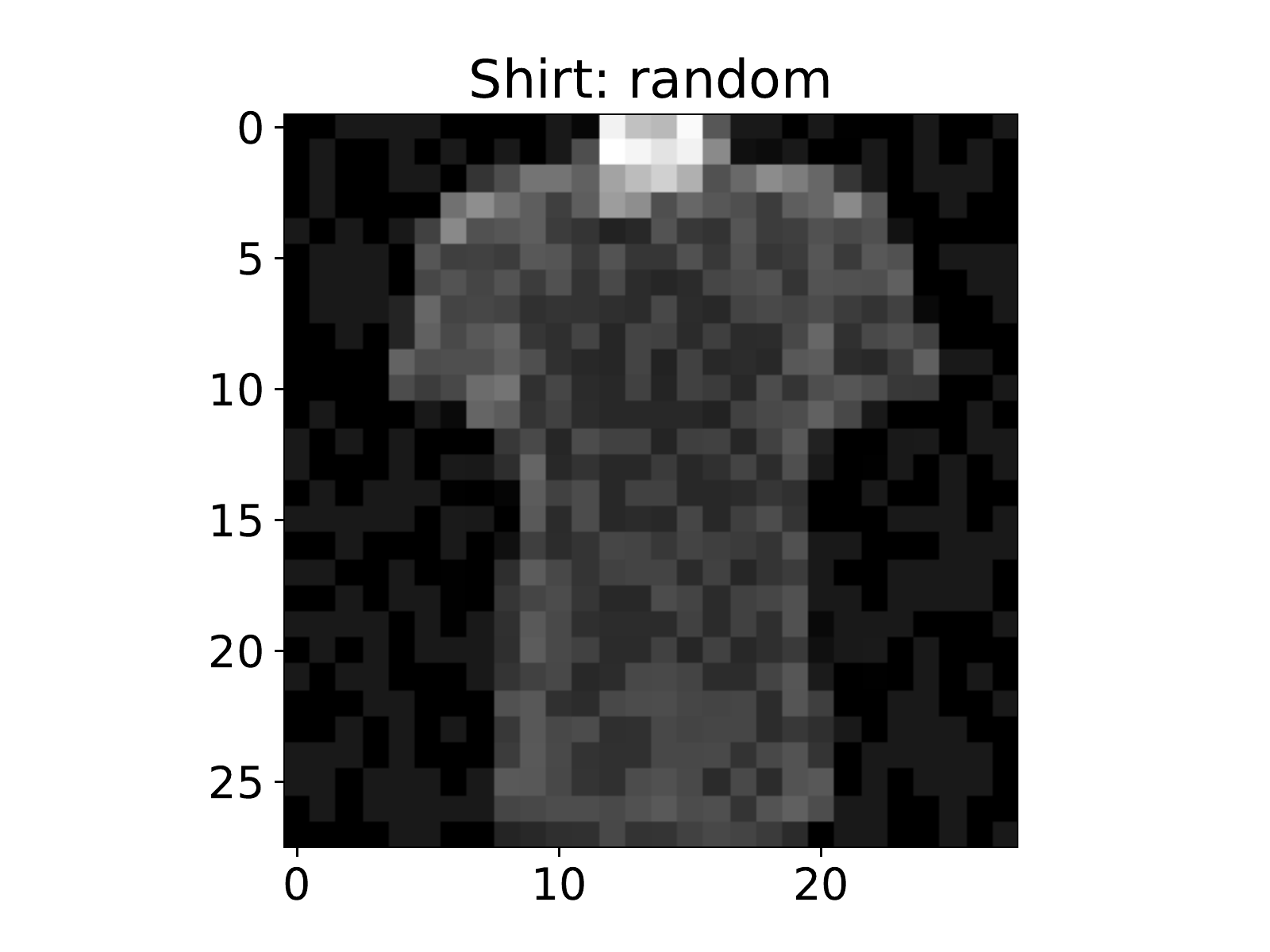}
\includegraphics[width=.3\textwidth]{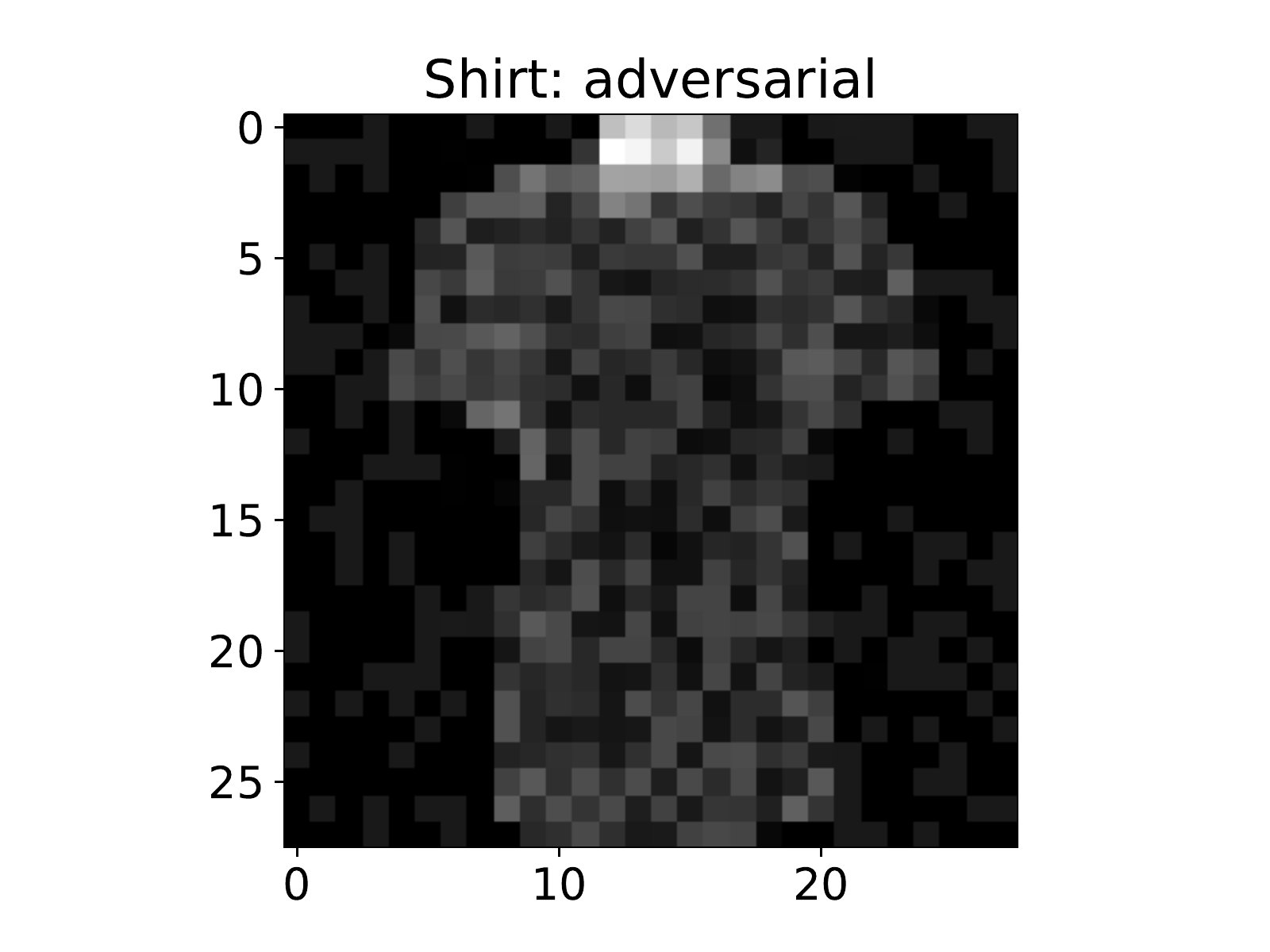}
\caption{``Shirt'' example}
\label{fig:exp:real:nn:example:shirt}
\end{figure}

\begin{figure}[!h]
\centering
\includegraphics[width=.3\textwidth]{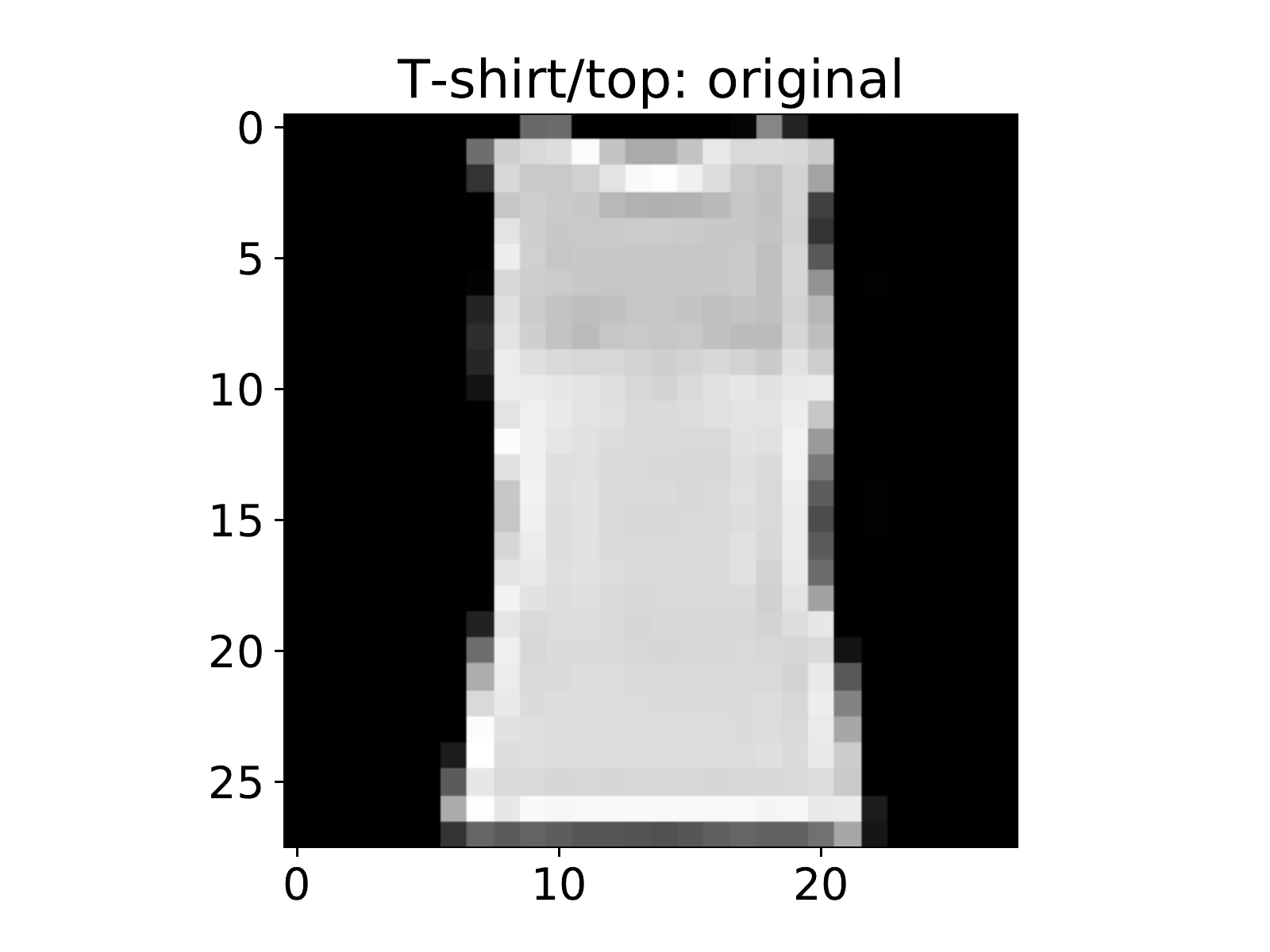}
\includegraphics[width=.3\textwidth]{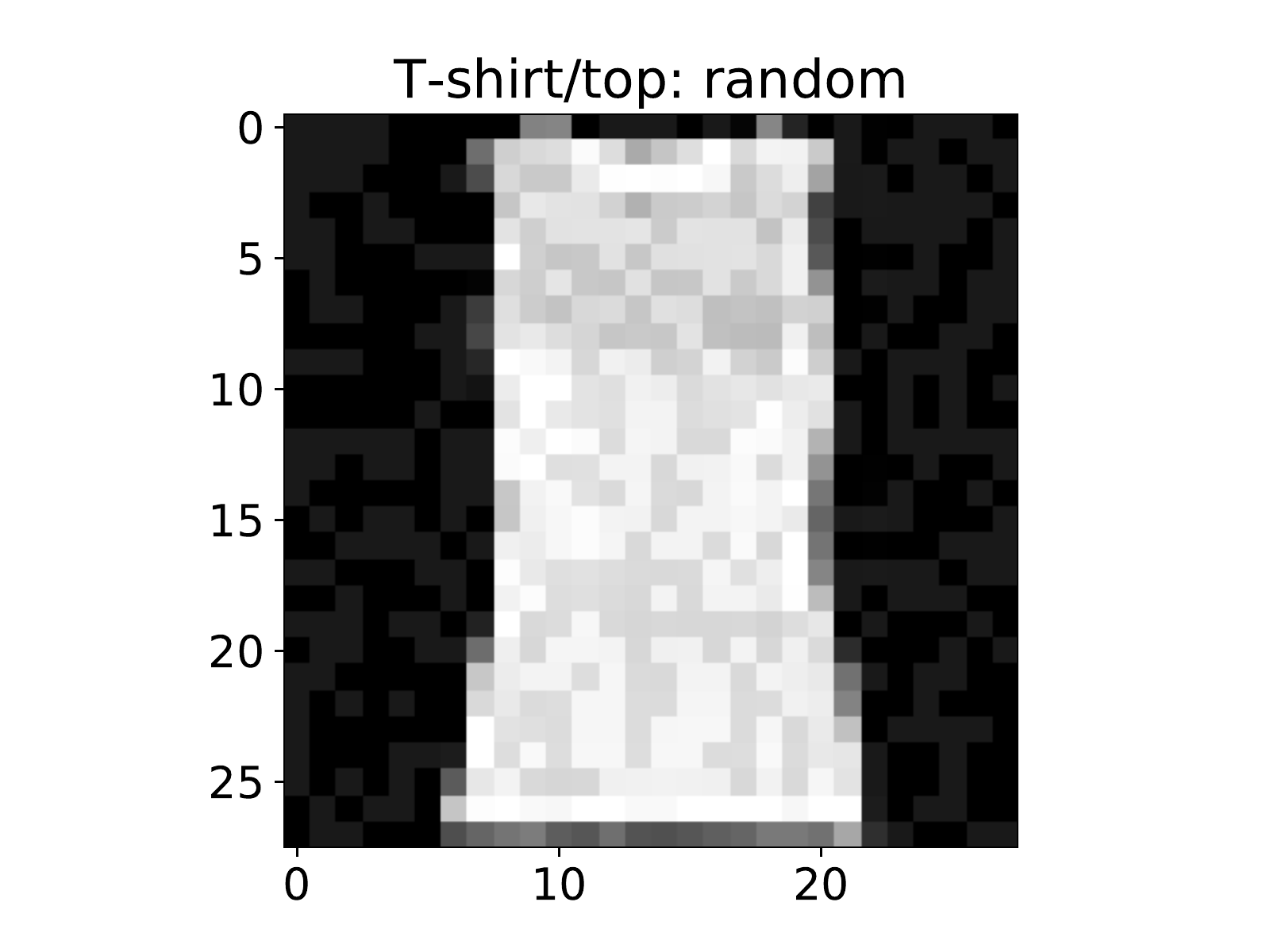}
\includegraphics[width=.3\textwidth]{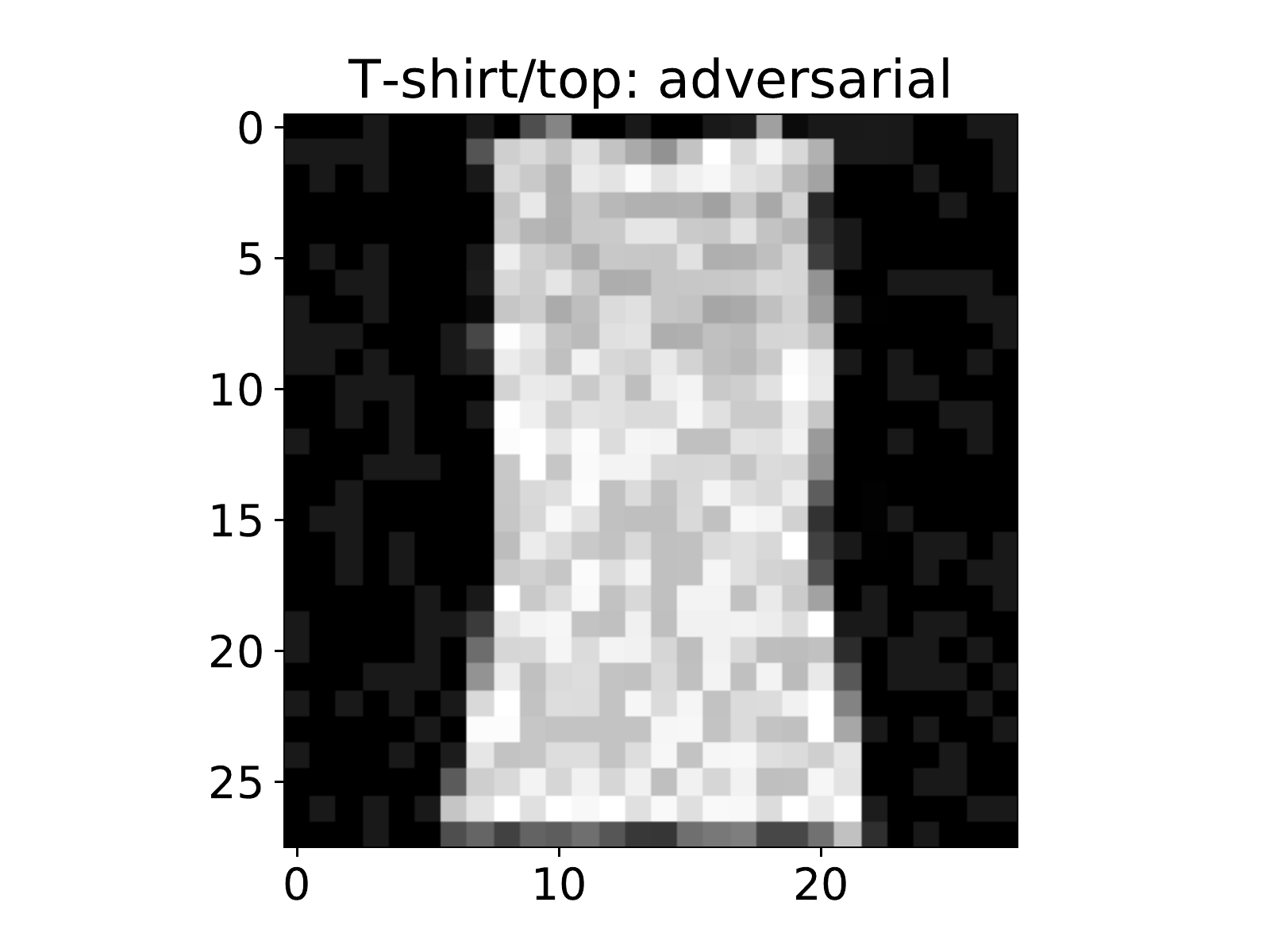}
\caption{``T-shirt/top'' example}
\label{fig:exp:real:nn:example:tshirt-top}
\end{figure}

\begin{figure}[!h]
\centering
\includegraphics[width=.3\textwidth]{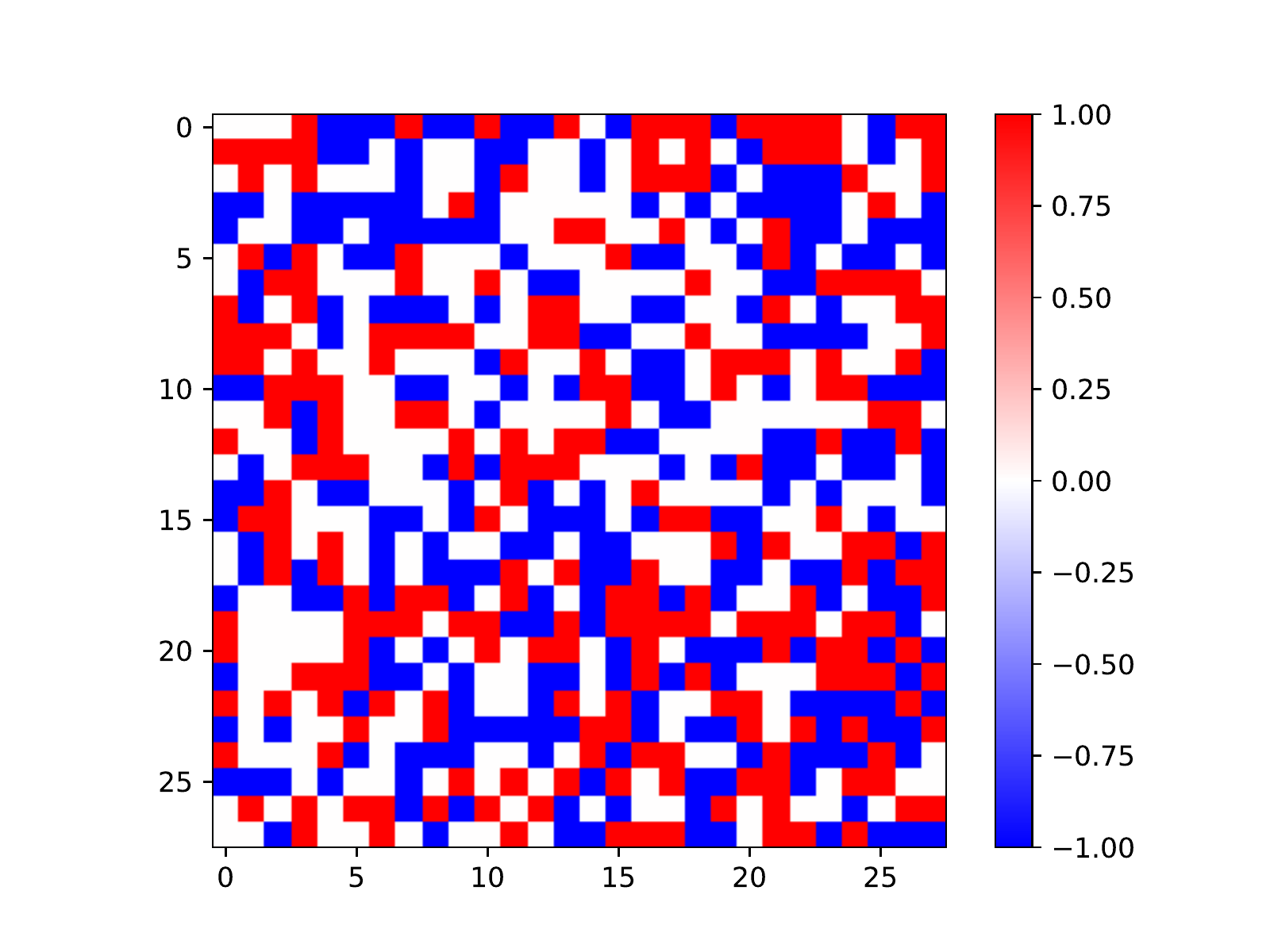}
\caption{Adversarial perturbation generated using the fast gradient sign method \cite{goodfellow2014explaining}}
\label{fig:exp:real:nn:advsarial-pertub}
\end{figure}

\subsubsection{High dimensional linear regression}
\label{append:subsubsec:exp:real:high-dim:riboflavin}

For both experiments, the hyper-parameters are chosen based on the results in \Cref{sec:high-dim:lasso:linear-regression}, 
where we estimate the true parameter's $\ell_1$ norm $\|\theta^\star\|_1$ and noise level $\sigma$ by vanilla LASSO with cross validation, 
using the LASSO solution's $\ell_1$ norm and LASSO residuals' 2\textsuperscript{nd} moment's square root. 
The covariance threshold is chosen so that it minimizes the thresholded covariance's condition number.

\paragraph{HIV drug resistance mutations dataset.} 
We apply our high dimensional inference procedure to the dataset in \cite{rhee2006genotypic}  to detect mutations related to HIV drug resistance. 
Our procedure is able to detect verified mutations in an expert dataset \cite{johnson2005update}, 
when we control the family-wise error rate (FWER) at $0.05$.

\begin{table}[!t]
\centering
\begin{tabular}{ | l | l | l  | }
\hline 
\multicolumn{2}{| c |}{Drug} & Mutations \\
\hline
\multirow{7}{*}{PI} & APV & 10F \\ \cline{2-3} 
 & ATV  & 33F, 43T, 84V \\  \cline{2-3} 
 & IDV  & 48V, 84A \\  \cline{2-3} 
 & LPV  & 46I \\  \cline{2-3} 
 & NFV  & 46L \\  \cline{2-3} 
 & RTV  & 10I, 54V \\  \cline{2-3}  
 & SQV   & 20R, 84V \\ 
 \hline
 \multirow{6}{*}{NRTI} &  3TC  & 184V \\ \cline{2-3} 
 & ABC  & 41L \\ \cline{2-3} 
 & AZT  & 41L, 210W \\ \cline{2-3} 
 & D4T  & 41L, 215Y \\ \cline{2-3} 
 & DDI  & 62V, 151M \\ \cline{2-3} 
 & TDF  & 41L, 75M  \\
 \hline
 \multirow{3}{*}{NNRTI} & DLV  & 228R \\ \cline{2-3} 
  & EFV  & 74V, 103N \\ \cline{2-3} 
  & NVP  & 103N, 181C \\ 
  \hline
\end{tabular}
\caption{HIV drug resistance related mutations detected by our high dimensional inference procedure}
\label{tab:exp:real:high-dim:hiv}
\end{table}

\paragraph{Riboflavin (vitamin B2) production rate data set.}
For the vanilla LASSO estimate on the high-throughput genomic data set concerning riboflavin (vitamin B2) production rate \cite{buhlmann2014high}, we set $\lambda = 0.021864$. 
\Cref{fig:exp:real:high-dim:riboflavin:CMP_LASSO}, and we see that our point estimate is similar to the vanilla LASSO point estimate. 

For statistical inference, 
in our method, we compute p-values  using two-sided Z-test. 
Adjusting FWER to 5\% signifi-cance level, 
our method does not find any significant gene. 
\cite{javanmard2014confidence,buhlmann2014high} report that \cite{buhlmann2013statistical} also does not find any significant gene, 
whereas \cite{meinshausen2009p} finds one significant gene (YXLD-at), and \cite{javanmard2014confidence} finds two significant genes (YXLD-at and YXLE-at). 
This indicates that our method is more conservative than  \cite{javanmard2014confidence, meinshausen2009p}.

\begin{figure}[!h]
\centering
\includegraphics[width=.4\textwidth]{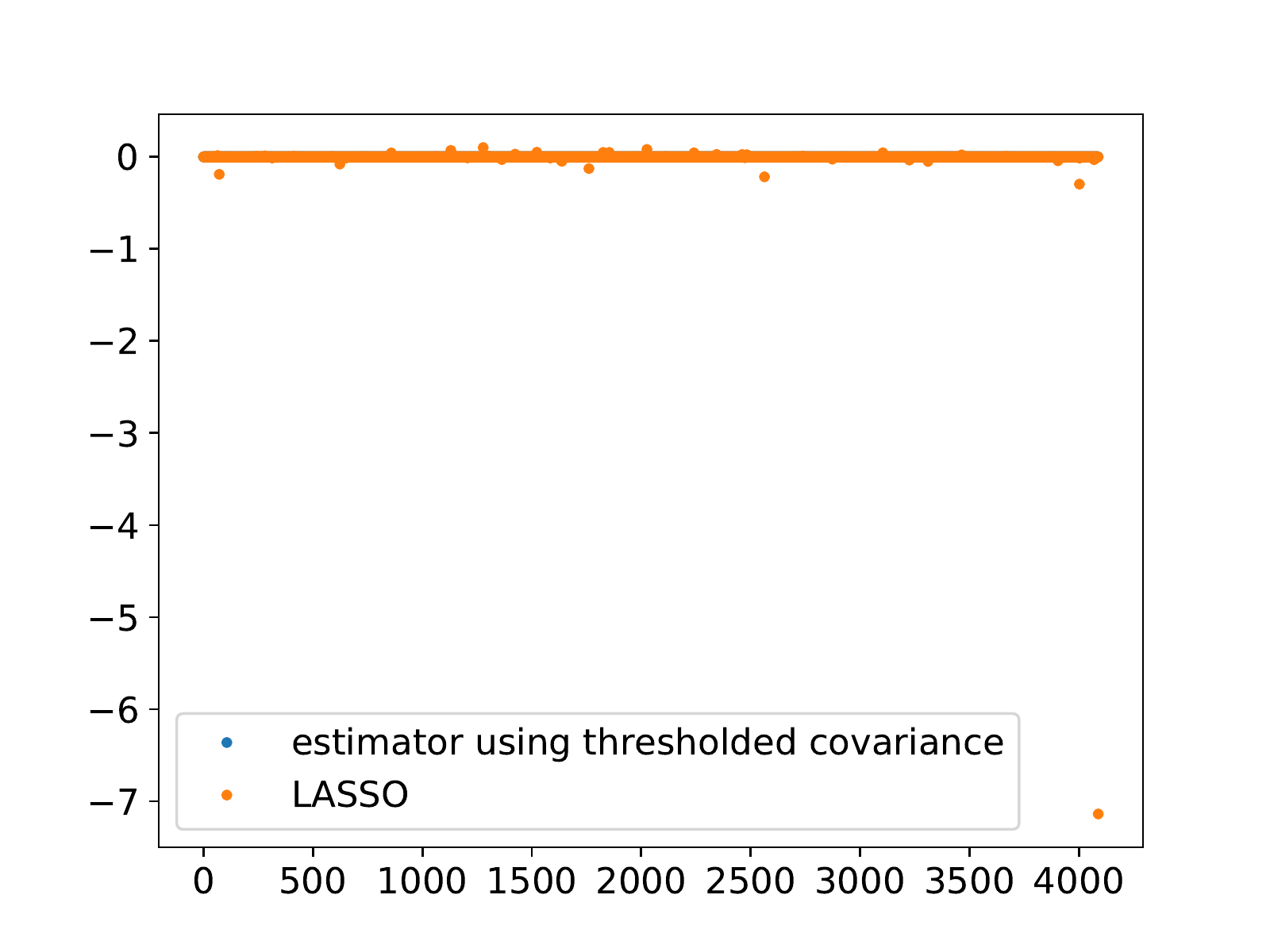}
\includegraphics[width=.4\textwidth]{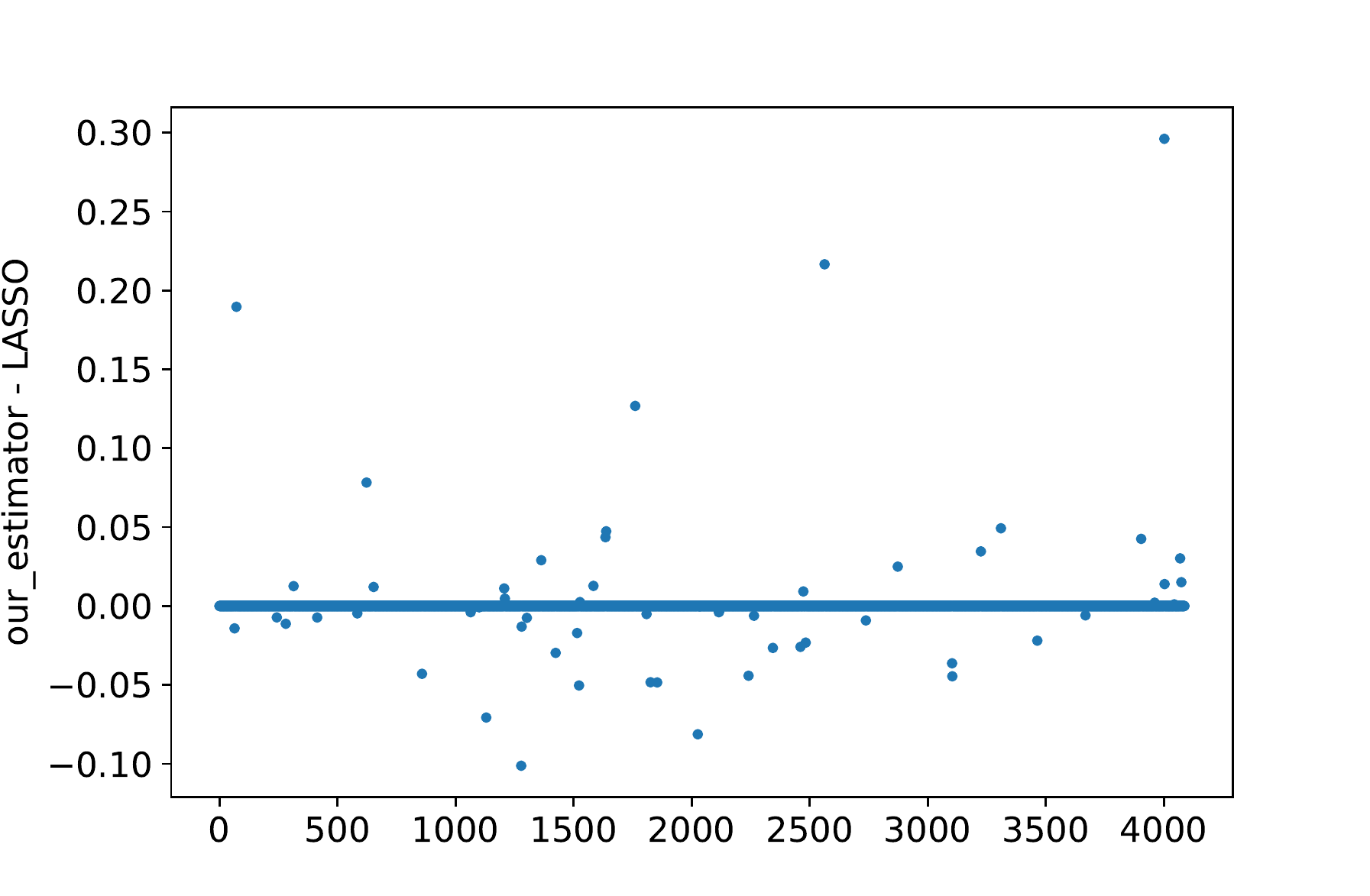}
\caption{Comparison of our high dimensional linear regression point estimate with the vanilla LASSO estimate}
\label{fig:exp:real:high-dim:riboflavin:CMP_LASSO}
\end{figure}

\subsubsection{Time series analysis}
\label{append:subsubsec:exp:real:time-series}

Using monthly equities returns data from \cite{frazzini2014betting}, 
we use our approximate Newton statistical inference procedure to show that the correlation between US equities market returns and non-US global equities market returns is statistically significant, 
which validates the capital asset pricing model (CAPM) \cite{sharpe1964capital, lintner1965valuation, fama2004capital}. 

We regress monthly US equities market returns from 1995 to 2018 against other countries' equities market returns, 
and each country's coefficient and its 95\% confidence interval is shown in \Cref{fig:exp:time-series:capm}. 
And we observe that the US market is highly positively correlated with Canada and other advanced economies such as Germany and UK.

\begin{figure}[!t]
\centering
\includegraphics[width=0.8\textwidth]{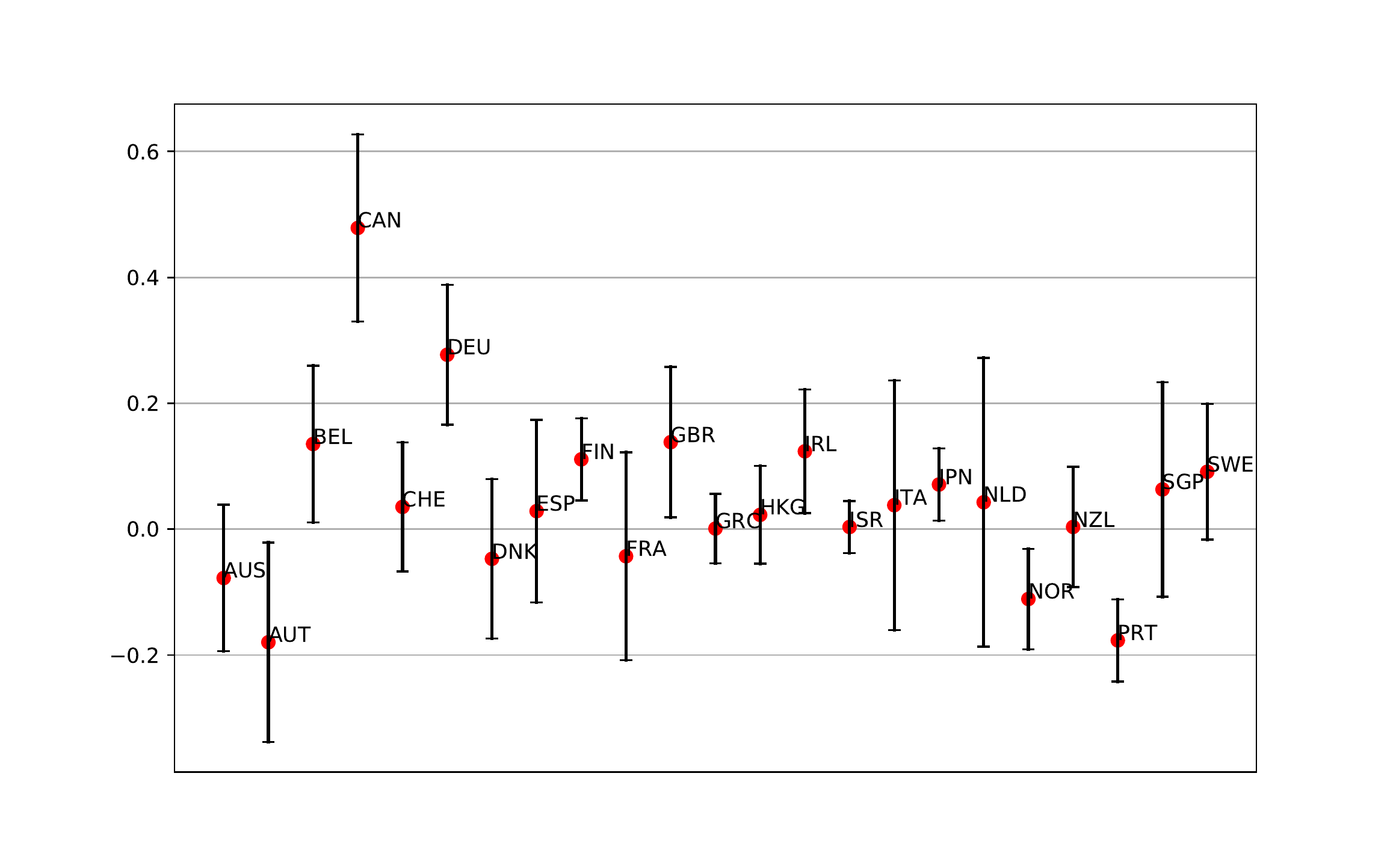}
\caption{Exposure of US equities market to equities markets of other countries}
\label{fig:exp:time-series:capm}
\end{figure}


\end{document}